\documentclass[twoside,11pt]{article}
\usepackage{jmlr2e}
\usepackage{xcolor}
\usepackage{graphicx}
\usepackage{enumitem}
\usepackage{subcaption}
\usepackage{wrapfig}
\usepackage{adjustbox}
\usepackage{booktabs}
\usepackage{multirow}
\usepackage{multicol}
\usepackage{amsmath}
\usepackage{textcomp}  
\usepackage{pifont}
\usepackage{setspace}
\usepackage{microtype}
\usepackage{lastpage}
\usepackage{thmtools}
\usepackage{overpic}
\usepackage[ruled,vlined]{algorithm2e}
\usepackage{array}
\usepackage{makecell}

\definecolor{ForestGreen}{RGB}{34,139,34}

\newcommand{\new}[1]{#1}

\newcommand{\fwd}{\mathsf{F}}
\newcommand{\rev}{\mathsf{B}}
\newcommand{\ODE}{{}}

\newcommand{\T}{\mathsf{T}}
\newcommand{\KL}[2]{\mathsf{KL}(#1\: \Vert \: #2)}
\newcommand{\cross}[2]{\mathsf{H}(#1\:\Vert\:#2)}
\newcommand{\fisher}[2]{\mathsf{FI}(#1\: \Vert \: #2)}
\newcommand{\norm}[1]{\left|#1\right|}

\newcommand{\Id}{\text{\it Id}}

\newcommand{\diff}{\mathsf{d}}

\def\CC{{\mathbb{C}}}
\def\EE{{\mathbb{E}}}
\def\NN{{\mathbb{N}}}

\def\RR{{\mathbb{R}}}
\def\eps{\epsilon}
\def\seps{\sqrt{2\epsilon}}

\newtheorem{assumption}[theorem]{Assumption}
\numberwithin{equation}{section}

\newcommand{\OS}{\mathsf{os}}
\newcommand{\MIR}{\mathsf{mir}}
\newcommand{\LIN}{\mathsf{lin}}
\newcommand{\OSLIN}{\mathsf{os,lin}}
\newcommand{\DEN}{\mathsf{den}}
\newcommand{\REC}{\mathsf{rec}}

\usepackage{lastpage}
\jmlrheading{26}{2025}{1-\pageref{LastPage}}{12/23; Revised
9/25}{9/25}{23-1605}{Michael S. Albergo, Nicholas M. Boffi, Eric Vanden-Eijnden}
\ShortHeadings{Stochastic Interpolants}{Albergo, Boffi, and Vanden-Eijnden}
\firstpageno{1}

\begin{document}

\title{Stochastic Interpolants:\\ A Unifying Framework for Flows and Diffusions}

\author{%
	\name Michael S.~Albergo$^*$ \email albergo@nyu.edu \\
	\addr Center for Cosmology and Particle Physics\\
	New York University\\
	New York, NY 10012, USA
	\AND
	\name Nicholas M.~Boffi\thanks{Author ordering alphabetical; authors contributed equally.} \email boffi@cims.nyu.edu \\
	\addr Courant Institute of Mathematical Sciences\\
	New York University\\
	New York, NY 10012, USA
	\AND
	\name Eric Vanden-Eijnden \email eve2@cims.nyu.edu \\
	\addr Courant Institute of Mathematical Sciences\\
	New York University\\
	New York, NY 10012, USA
}
\editor{Maxim Raginsky}
\maketitle

\begin{abstract}A class of generative models that unifies flow-based and diffusion-based methods is introduced. These models extend the framework proposed in~\cite{albergo2023building},  enabling the use of a broad class of continuous-time stochastic processes called `stochastic interpolants' to bridge any two  probability density functions exactly in finite time.
These interpolants are built by combining data from the two prescribed densities with an additional latent variable that shapes the bridge in a flexible way. The time-dependent probability density function of the stochastic interpolant is shown to satisfy a first-order transport equation as well as a family of forward and backward Fokker-Planck equations with tunable diffusion coefficient. 
Upon consideration of the time evolution of an individual sample, this viewpoint immediately leads to both deterministic and stochastic generative models based on probability flow equations or stochastic differential equations with an adjustable level of noise. 
The drift coefficients entering these models are time-dependent velocity fields characterized as the unique minimizers of simple quadratic objective functions, one of which is a new objective for the score of the interpolant density. 
We show that minimization of these quadratic objectives leads to control of the likelihood for generative models built upon stochastic dynamics, while likelihood control for deterministic dynamics is more stringent.
We also construct estimators for the likelihood and the cross-entropy of interpolant-based generative models, and we discuss connections with other methods such as score-based diffusion models, stochastic localization processes, probabilistic denoising techniques, and rectifying flows. 
In addition, we demonstrate that stochastic interpolants recover the Schr\"odinger bridge between the two target densities when explicitly optimizing over the interpolant. 
Finally, algorithmic aspects are discussed and the approach is illustrated on numerical examples.\end{abstract}

\tableofcontents

\section{Introduction}
\subsection{Background and motivation}
\label{sec:back:mot}

Dynamical approaches for deterministic and stochastic transport have become a central theme in contemporary generative modeling research. At the heart of progress is the idea to use ordinary or stochastic differential equations (ODEs/SDEs) to continuously transform samples from a base probability density function (PDF)~$\rho_0$ into samples from a target PDF~$\rho_1$ (or vice-versa), and the realization that inference over the velocity field in these equations can be formulated as an empirical risk minimization problem over a parametric class of functions \citep{grathwohl2018scalable,Song2019,ho2020,song2021scorebased,benhamu2022,albergo2023building,liu2022,lipman2022}.

A major milestone was the introduction of score-based diffusion methods (SBDM) \citep{song2021scorebased}, which map an arbitrary density into a standard Gaussian by passing samples through an Ornstein-Uhlenbeck (OU) process. The key insight of SBDM is that this process can be reversed by introducing a backwards SDE whose drift coefficient depends on the score of the time-dependent density of the process. By learning this score -- which can be done by minimization of a quadratic objective function known as the denoising loss~\citep{vincent_connection_2011} -- the backwards SDE can be used as a generative model that maps Gaussian noise into data from the target. Though theoretically exact, the mapping takes infinite time in both directions, and hence must be truncated in practice.

While diffusion-based methods have become state-of-the-art for tasks such as image generation, there remains considerable interest in developing methods that bridge two \textit{arbitrary} densities (rather than requiring one to be Gaussian), that accomplish the transport \textit{exactly}, and that do so on a \textit{finite} time interval.
Moreover, while the highest quality results from score-based diffusion were originally obtained using SDEs~\citep{song2021scorebased}, this has been challenged by recent works that find equivalent or better performance with ODE-based methods if the score is learned sufficiently well~\citep{Karras2022edm}.
If made to match the performance of their stochastic counterparts, ODE-based methods exhibit a number of desirable characteristics, such as an exact, computationally tractable formula for the likelihood and the easy application of well-developed adaptive integration schemes for sampling. 
It is an open question of significant practical importance to understand if there exists a separation in sample quality between generative models based on deterministic dynamics and those based on stochastic dynamics. 

In order to satisfy the desirable characteristics outlined in the previous paragraph, we develop a framework for generative modeling based on the method proposed in~\cite{albergo2023building}, which is built on the notion of a \textit{stochastic interpolant}~$x_t$ used to bridge  two arbitrary densities $\rho_0$ and $\rho_1$.
We will consider more general designs below, but as one example the reader can keep in mind:
\begin{equation}
    \label{eq:stoch:interp:lin}
    x_t = (1-t) x_0 + t x_1 + \sqrt{2t(1-t)} z, \quad t \in [0,1],
\end{equation}
where $x_0$, $x_1$, and $z$ are random variables drawn independently from $\rho_0$, $\rho_1$, and the standard Gaussian density $\mathsf{N}(0,\Id)$, respectively. The stochastic interpolant~$x_t$ defined in~\eqref{eq:stoch:interp:lin} is a continuous-time stochastic process that, by construction, satisfies $x_{t=0} = x_0\sim \rho_0$ and $x_{t=1} = x_1\sim \rho_1$. Its paths therefore \textit{exactly} bridge between samples  from $\rho_0$ at $t=0$ and from $\rho_1$ at $t=1$. A key observation is that:
\begin{quote}
    \textit{The law of the interpolant $x_t$ at any time $t\in[0,1]$ can be realized by many different processes, including an ODE and forward and backward SDEs whose drifts can be learned from data.}
\end{quote} 
To see why this is the case, one must consider the probability distribution of the interpolant~$x_t$. As shown below, for a large class of densities $\rho_0$ and $\rho_1$ supported on $\RR^d$, this distribution is absolutely continuous with respect to the Lebesgue measure. Moreover, its time-dependent density~$\rho(t)$ satisfies a first-order transport equation and a family of forward and backward Fokker-Planck equations in which the diffusion coefficient can be varied at will. Out of these equations, we can readily derive generative models that satisfy ODEs and SDEs, respectively, and whose densities at time~$t$ are given by $\rho(t)$.

\begin{figure}[t]
    \centering
    \includegraphics[width=\linewidth]{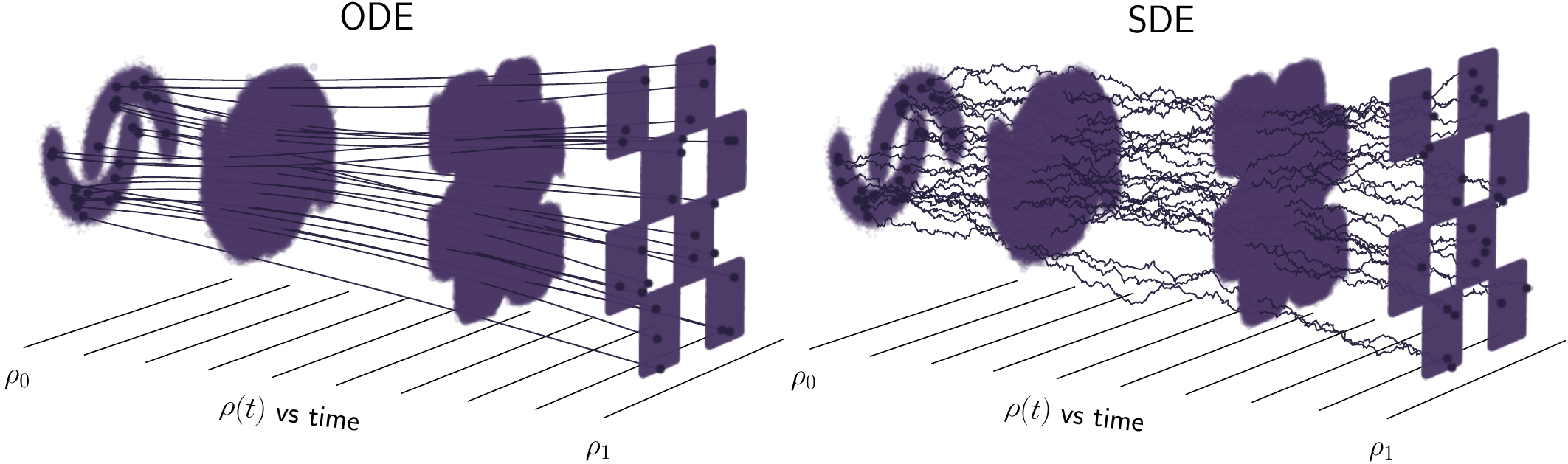}
    \caption{\textbf{The stochastic interpolant paradigm.} 
    Example generative models based on the proposed framework, which connects two densities $\rho_0$ and $\rho_1$ using samples from both.
    The design of the time-dependent probability density $\rho(t)$ that bridges between $\rho_0$ and $\rho_1$ is separated from the choice of how to sample it, which can be accomplished with deterministic or stochastic generative models. 
    \textit{Left panel:} Sampling with a deterministic (ODE) generative model known as the probability flow equation. 
    \textit{Right panel:} Sampling with a stochastic generative model given by an SDE with a tunable diffusion coefficient. 
    The probability flow equation and the SDE have different paths, but their time-dependent density $\rho(t)$ is the same.
    Moreover, the two equations rely on the same estimates for the velocity and the score. }
    \label{fig:my_label}
\end{figure}

Interestingly, the drift coefficients entering these ODEs/SDEs are the unique minimizers of quadratic objective functions that can be  estimated empirically using data from $\rho_0$, $\rho_1$, and $\mathsf{N}(0,\Id)$. The resulting least-squares regression problem allows us to estimate the drift coefficients of the ODE/SDEs, which can then be used to push samples from $\rho_0$ onto new samples from $\rho_1$ and vice-versa. 

\subsection{Main contributions and organization}
The approach introduced here is a versatile way to build generative models that unifies and extends many existing algorithms. In Sec.~\ref{sec:theo}, we develop the framework in full generality, where we emphasize the following key contributions:
\begin{itemize}[leftmargin=0.2in]
    \item We prove that the stochastic interpolant defined in Section~\ref{sec:si:gm} has a distribution that is absolutely continuous with respect to the Lebesgue measure on $\RR^d$, and that its density $\rho(t)$ satisfies a first-order transport equation (TE) as well as a family of forward and backward Fokker-Planck equations (FPEs) with tunable diffusion coefficients.
    \item We show how the stochastic interpolant can be used to learn the drift coefficients that enter the TE and the FPEs. 
    We characterize these coefficients as the minimizers of simple quadratic objective functions given in Section~\ref{sec:cont:eq}. 
    We introduce a new objective for the score $\nabla\log\rho(t)$ of the interpolant density, as well as an objective function for learning a denoiser $\eta_z$, which we relate to the score.
    \item In Section~\ref{sec:generative}, we derive ordinary and stochastic differential equations associated with the TE and FPEs that lead to deterministic and stochastic generative models.
    In Section~\ref{sec:likelihood_bounds}, we show that regressing the drift for SDE-based models controls the likelihood, but that regressing the drift alone is not sufficient for ODE-based models, which must also minimize a Fisher divergence. We show how to optimally tune the diffusion coefficient to maximize the likelihood for SDEs.
    \item In Section~\ref{sec:density}, we develop a general formula to evaluate the likelihood of SDE-based generative models that serves as a natural counterpart to the continuous change-of-variables formula commonly used to compute the likelihood of ODE-based models. In addition, we give formulas to estimate the cross-entropy.
\end{itemize}

In Section~\ref{sec:generalization}, we discuss  instantiations of the stochastic interpolant method. 
In Section~\ref{sec:sb} we first show that interpolants are equivalent to a class of stochastic bridges, but that they avoid the need for Doob's $h$-transform, which is generically unknown; we show that this simplifies the construction of a broad class of generative models.
In Section~\ref{sec:onesided}, we define the \textit{one-sided interpolant}, which corresponds to the conventional setting in which the base $\rho_0$ is taken to be a Gaussian.
With a Gaussian base, several aspects of the interpolant simplify, and we detail the corresponding objective functions.
In Section~\ref{sec:mirror}, we introduce a \textit{mirror interpolant} in which the base $\rho_0$ and the target $\rho_1$ are identical.
Finally, in Section~\ref{sec:sb}, we show how the interpolant framework leads to a natural formulation of the Schr\"odinger bridge problem between two densities.

In Section~\ref{sec:gen}, we discuss a special case in which the interpolant is spatially linear in $x_0$ and $x_1$. 
In this case, the velocity field can be factorized, which we show in Section~\ref{sec:factor} leads to a simpler learning problem. 
We detail specific choices of linear interpolants in Section~\ref{sec:specific:a:b:c}, and in Section~\ref{sec:impact:gam} we illustrate  how these choices influence the performance of the resulting generative model, with a particular focus on the role of the latent variable and the diffusion coefficient. 
For exposition, we focus on Gaussian mixture densities, for which the drift coefficients can be computed analytically. 
We provide the resulting formula in Appendix~\ref{app:Gauss:mixt}. 
Finally, in Section~\ref{sec:spatil:lin:os}, we discuss the case of spatially linear one-sided interpolants.

In Section~\ref{sec:connection}, we formalize the connection between stochastic interpolants and related classes of generative models.
In Section~\ref{sec:SBDM}, we show that score-based diffusion models can be re-written as one-sided interpolants after a reparameterization of time; we highlight how this approach eliminates singularities that appear when naively compressing score-based diffusion onto a finite-time interval.
In Section~\ref{sec:denoiser}, we show how interpolants can be used to derive the Bayes-optimal estimator for a denoiser, and we show how this approach can be iterated to create a generative model. 
In Section~\ref{sec:rect}, we consider the possibility of rectifying the flow map of a learned generative model.
We show that the rectification procedure does not change the underlying generative model, though it may change the time-dependent density of the interpolant.

In Section~\ref{sec:practical}, we provide the details of practical algorithms associated with the mathematical results presented above. 
In Section~\ref{sec:learning}, we describe how to numerically estimate the objectives given empirical datasets from the base and the target.
In Section~\ref{sec:sampling}, we complement this discussion on \textit{learning} with algorithms for \textit{sampling} with the ODE or an SDE. 

We provide numerical demonstrations in line with these recommendations in Section~\ref{sec:numerics}, and we conclude with some remarks in Section~\ref{sec:conc}.

\subsection{Related work}
\label{sec:related}

\paragraph{Deterministic Transport and Normalizing Flows.}

Transport-based sampling and density estimation has its contemporary roots in Gaussianizing data via maximum entropy methods \citep{Friedman1987,Chen2000, tabak2010, tabak2013}. The change of measure under such transformation is the backbone of normalizing flow models. The first neural network realizations of these methods arose through imposing clever structure on the transformation to make the change of measure tractable in discrete, sequential steps \citep{rezende2015, dinh2017density, papamakarios2017, huang2018, durkan2019}. A continuous time version of this procedure was made possible by viewing the map $T= X_t(x)$ as the solution of an ODE  \citep{chen2018, grathwohl2018scalable}, whose parametric drift defining the transport is learned via maximum likelihood estimation. Training this way is intractable at scale, as computing the gradient of the objective via the adjoint method requires simulating an ODE. Various methods have introduced regularization on the path taken between the two densities to make the ODE solves more efficient \citep{finlay2020,wu2021,tong_trajectorynet_2020}, but the fundamental difficulty remains. We also work in continuous time; however, our approach allows us to learn the drift without simulation of the dynamics, and can be formulated at sample generation time through either deterministic or stochastic transport.

\paragraph{Stochastic Transport and Score-Based Diffusion Models (SBDMs).}
Complementary to approaches based on deterministic maps, recent works have realized that connecting a data distribution to a Gaussian density can be viewed as the evolution of an Ornstein-Ulhenbeck (OU) process which gradually degrades samples from the distribution of interest to Gaussian noise~\citep{dickstein2015, ho2020, Song2019, song2021scorebased}. 
The OU process specifies a path in the space of probability densities; this path is simple to traverse in the forward direction by addition of noise, and can be reversed if access to the score of the time-dependent density $\nabla\log\rho(t)$ is available.
This score can be approximated through solution of a least-squares regression problem~\citep{hyvarinen05a, vincent_connection_2011}, and the target can be sampled by reversing the path once the score has been learned.
Interestingly, the resulting forward and backward stochastic processes have an equivalent formulation (at the distribution level) in terms of a deterministic probability flow equation, first noted by~\cite{Bakry1985, OTTO2000361, Kim2010AGO} and then applied in~\cite{maoutsa2020, song2021mle, kingma2021on, boffi2022}. 
The probability flow formulation is useful for density estimation and cross-entropy calculations, but it is worth noting that the probability flow and the reverse-time SDE will have densities that differ when using an approximate score.
The SBDM framework, as it has been originally presented, has a number of features which are not \textit{a~priori} well motivated, including the dependence on mapping to a normal density, the complicated tuning of the time parameterization and noise scheduling \citep{xiao2022tackling, hoogeboom2023simple}, and the choice of the underlying stochastic dynamics \citep{dockhorn2022score, Karras2022edm}.

\paragraph{Stochastic bridges.}
Starting with~\citep{peluchetti2022nondenoising} there has been some recent effort~\citep{liu2022let,liu2023I2SB,somnath2023aligned} to remove the dependence of SBDMs on the OU process via stochastic bridges, which can be used to connect two arbitrary densities in finite time. 
As another step in this direction, we observe here that the key idea behind SBDMs -- the bridging of two densities via a time-dependent density whose evolution equation is available -- can be generalized to a much wider class of processes in a straightforward and computationally accessible manner. 
This viewpoint highlights the key property that the construction of the bridge between the two densities is decoupled from the process used to sample it.

\paragraph{Stochastic Interpolants, Rectified Flows, and Flow matching.}
Variants of the stochastic interpolant method presented in \cite{albergo2023building} were also presented in \cite{liu2022, lipman2022}. In \cite{liu2022}, a linear interpolant was proposed with a focus on straight paths. This was employed as a step toward rectifying the transport paths \citep{liu2022-ot} through a procedure that improves sampling efficiency but introduces a bias.  In \cite{lipman2022}, the interpolant picture was assembled from the perspective of conditional probability paths connecting to a Gaussian, where a noise convolution was used to improve the learning at the cost of biasing the method. Extensions of \cite{lipman2022} were presented in \cite{tong2023conditional} that generalize the method beyond the Gaussian base density. In the method proposed here, we introduce an unbiased means to incorporate noise into the process, both via the introduction of a latent variable into the stochastic interpolant and the inclusion of a tunable diffusion coefficient in the associated stochastic generative models. We provide theoretical and practical motivation for the presence of these noise terms.

\paragraph{Optimal Transport and Schrödinger Bridges.}

There is both theoretical and practical interest in minimizing the transport cost of connecting $\rho_0$ and $\rho_1$. In the case of deterministic maps, this is characterized by the optimal transport problem, and in the case of diffusive maps, by the Schr\"odinger Bridge problem \citep{villani2009optimal, chen2021}. Formally, these two problems can be related by viewing the Schr\"odinger Bridge as an entropy-regularized optimal transport.
Optimal transport has primarily been employed as a means to regularize flow-based methods by imposing either a path length penalty \citep{zhang2018, wu2021, finlay2020, tong_trajectorynet_2020} or structure on the parameterization itself \citep{huang2021convex, Yang2022}. 
A variety of recent works have formulated the Schr\"odinger problem in the context of a learnable diffusion \citep{bortoli2021diffusion, su2023dual, chen2022likelihood}.  
For the case of Gaussians, recent work has also identified an analytical solution~\citep{bunne_schrodinger_2022}.
In the interpolant framework, \citep{albergo2023building, liu2022, lipman2022, tong2023conditional} all propose optimal transport extensions to the learning procedure. The method proposed in \cite{liu2022, liu2022-ot} allows one to sequentially lower the transport cost through rectification, at the cost of introducing a bias unless the velocity field is perfectly learned. The method proposed in \cite{albergo2023building} is an unbiased framework at the cost of solving an additional optimization problem over the interpolant function. The statement of optimal transport in \cite{lipman2022} only applies to Gaussians, but is shown to be practically useful in experimental demonstrations.

In the method proposed below, we provide two approaches for optimizing the transport under a stochastic dynamics. Our primary approach, based on the scheme introduced in \cite{albergo2023building}, is presented in Section~\ref{sec:si:schb}. It offers an alternative route to solve the Schr\"odinger bridge problem under the Benamou-Brenier hydrodynamic formulation of transport by maximizing over the interpolant \citep{benamou2000computational}. However, we stress that this additional optimization step is not necessary in practice, as our approach leads to bias-free generative models for any fixed interpolant. In addition, Section~\ref{sec:rect} discusses an unbiased variant of the rectification scheme proposed in \cite{liu2022}.

\paragraph{Convergence bounds.}
Inspired by the successes of score-based diffusion, significant recent research effort has been expended to understand the control that can be obtained on suitable distances between the distribution of the generative model and the target data distribution, such as $\mathsf{KL}$, $W_2$, or $\mathsf{TV}$. 
Perhaps the first line of work in this direction is~\cite{song2021mle}, which showed that standard score-based diffusion training techniques bound the likelihood of the resulting SDE model. 
Importantly, as we show here, the likelihood of the corresponding probability flow is not bounded in general by this technique, as first highlighted in the context of SBDM by~\cite{lu2022higherorder}.
Control for SBDM-based techniques was later quantified more rigorously under the assumption of functional inequalities in a discretized setting by~\cite{holden_score1}, which were removed by~\cite{holden_score2} and~\cite{sinho_score1} via Girsanov-based techniques.
Most relevant to the PDE-based methods considered here is~\cite{holden_score3}, which applies similar techniques to our own in the SBDM context to obtain sharp guarantees with minimal assumptions.

\subsection{Notation}
\label{sec:notations}

Throughout, we denote probability density functions as $\rho_0(x)$, $\rho_1(x)$, and $\rho(t,x)$, with $t\in[0,1]$ and $x\in\RR^d$, omitting the function arguments when clear from the context. We proceed similarly for other functions of time and space, such as $b(t,x)$ or $I(t,x_0,x_1)$. We use the subscript $t$ to denote the time-dependency of stochastic processes, such as the stochastic interpolant $x_t$ or the Wiener process $W_t$. To specify that the random variable $x_0$ is drawn from the probability distribution with density $\rho_0$, say, with a slight abuse of notations we use $x_0\sim\rho_0$. Similarly, we use ${\sf N}(0,\Id)$ to denote both the density and the distribution of the Gaussian random variable with mean zero and covariance identity. We denote expectation by $\EE$, and usually specify the random variables this expectation is taken over. With a slight abuse of terminology, we say that the law of the process $x_t$ is $\rho(t)$ if $\rho(t)$ is the density of the probability distribution of $x_t$ at time $t$.

We use standard notation for function spaces: for example, $C^1([0,1])$ is the space of continuously differentiable functions from $[0,1]$ to $\RR$, $(C^2(\RR^d))^d$ is the space of twice continuously differentiable functions from $\RR^d$ to $\RR^d$, and $C^p_0(\RR^d)$ is the space of compactly supported functions from $\RR^d$ to $\RR$ that are continuously differentiable $p$ times. Given a function $b:[0,1]\times \RR^d \to \RR^d$ with value $b(t,x)$ at $(t,x)$, we use $b \in C^1([0,1]; (C^2(\RR^d))^d)$ to indicate that $b$ is continuously differentiable in $t$ for all $(t,x)\in[0,1]\times \RR^d$ and that $b(t,\cdot)$ is an element of $(C^2(\RR^d))^d$ for all $t\in[0,1]$.

\section{Stochastic interpolant framework}
\label{sec:theo}
\subsection{Definitions and assumptions}
\label{sec:si:gm}

We begin by defining the stochastic processes that are central to our approach:
\begin{definition}[Stochastic interpolant] 
\label{def:interp}
 Given two probability density functions $\rho_0, \rho_1 : {\RR^d} \rightarrow \RR_{\geq 0}$, a \textit{stochastic interpolant} between $\rho_0$ and $\rho_1$ is a stochastic process $x_t$ defined as
\begin{equation}
    \label{eq:stochinterp}
    x_t = I(t,x_0,x_1) + \gamma(t) z,  \qquad t\in [0, 1],
\end{equation}
where: 
\begin{enumerate}[leftmargin=0.15in]
\item $I \in C^2([0,1],C^2(\RR^d\times\RR^d)^d)$ satisfies the boundary conditions $I(0,x_0,x_1) = x_0$ and $I(1,x_0,x_1) = x_1$, as well as
\begin{equation}
    \label{eq:bound:dI}
    \begin{aligned}
        &\exists C_1<\infty  \   : \ 
        &&|\partial_t I(t,x_0,x_1)|\le C_1|x_0-x_1|
        \quad  &&\forall (t,x_0,x_1) \in [0,1]\times \RR^d \times \RR^d.
        \end{aligned}
\end{equation}
\item $\gamma: [0,1] \to \RR $  satisfies $\gamma(0)=\gamma(1) = 0$, $\gamma(t) > 0$ for all $t \in (0, 1)$,  and $\gamma^2\in C^2([0,1])$.
\item The pair $(x_0,x_1)$ is drawn from a probability measure $\nu$ that marginalizes on $\rho_0$ and $\rho_1$, i.e. 
\begin{equation}
    \label{eq:margin:mu}
    \nu(dx_0,\RR^d) = \rho_0(x_0) dx_0, \qquad \nu(\RR^d,dx_1) = \rho_1(x_1) dx_1.
\end{equation}
\item $z$ is a Gaussian random variable independent of $(x_0,x_1)$, i.e. $z\sim {\sf N}(0,\text{\it Id})$ and $z\perp(x_0,x_1)$.
\end{enumerate}
\end{definition}
Eq.~\eqref{eq:bound:dI} states that $I(t, x_0, x_1)$ does not move too fast along the way from $x_0$ at $t=0$ to $x_1$ at $t=1$, and as a result does not wander too far from either endpoint -- this assumption is made for convenience but is not necessary for most arguments below.
Later, we will find it useful to consider choices for $I$ that are spatially nonlinear, which we show can recover the solution to the Schr\"odinger bridge problem. Nevertheless, a simple example that serves as a valid $I$ in the sense of Definition~\ref{def:interp} is given in~\eqref{eq:stoch:interp:lin}. 
The measure $\nu$ allows for a coupling between the two densities $\rho_0$ and $\rho_1$, which affects the properties of the stochastic interpolant, but a simple choice is to take the product measure $\nu(dx_0,dx_1) = \rho_0(x_0) \rho_1(x_1) dx_0dx_1$, in which case $x_0$ and $x_1$ are independent. 
In Section~\ref{sec:practical} we discuss how to design the  stochastic interpolant in~\eqref{eq:stochinterp} and state some properties of the corresponding process~$x_t$. Examples of stochastic interpolants are also shown in Figure~\ref{fig:example:xt} for various choices of $I$ and $\gamma$.

\begin{remark}[Comparison with~\cite{albergo2023building}]
    The main difference between the stochastic interpolant defined in~\eqref{eq:stochinterp} and the one originally introduced in~\cite{albergo2023building} is the inclusion of the latent variable $\gamma(t) z$. Many of the results below also hold when we set $\gamma(t)z=0$, but the objective of the present paper is to elucidate the advantages that this additional term provides when neither of the endpoints are Gaussian. We note that we could generalize the construction by making $\gamma(t)$ a tensor; here we focus on the scalar case for simplicity. Another difference is the possibility to couple $\rho_0$ and $\rho_1$ via $\nu$.
    \new{While the latent variable can be drawn from any noise distribution, as we will see, it will be convenient to choose it to be a Gaussian.}
\end{remark}

\begin{figure}[t!]
    \centering
    \includegraphics[width=.8\linewidth]{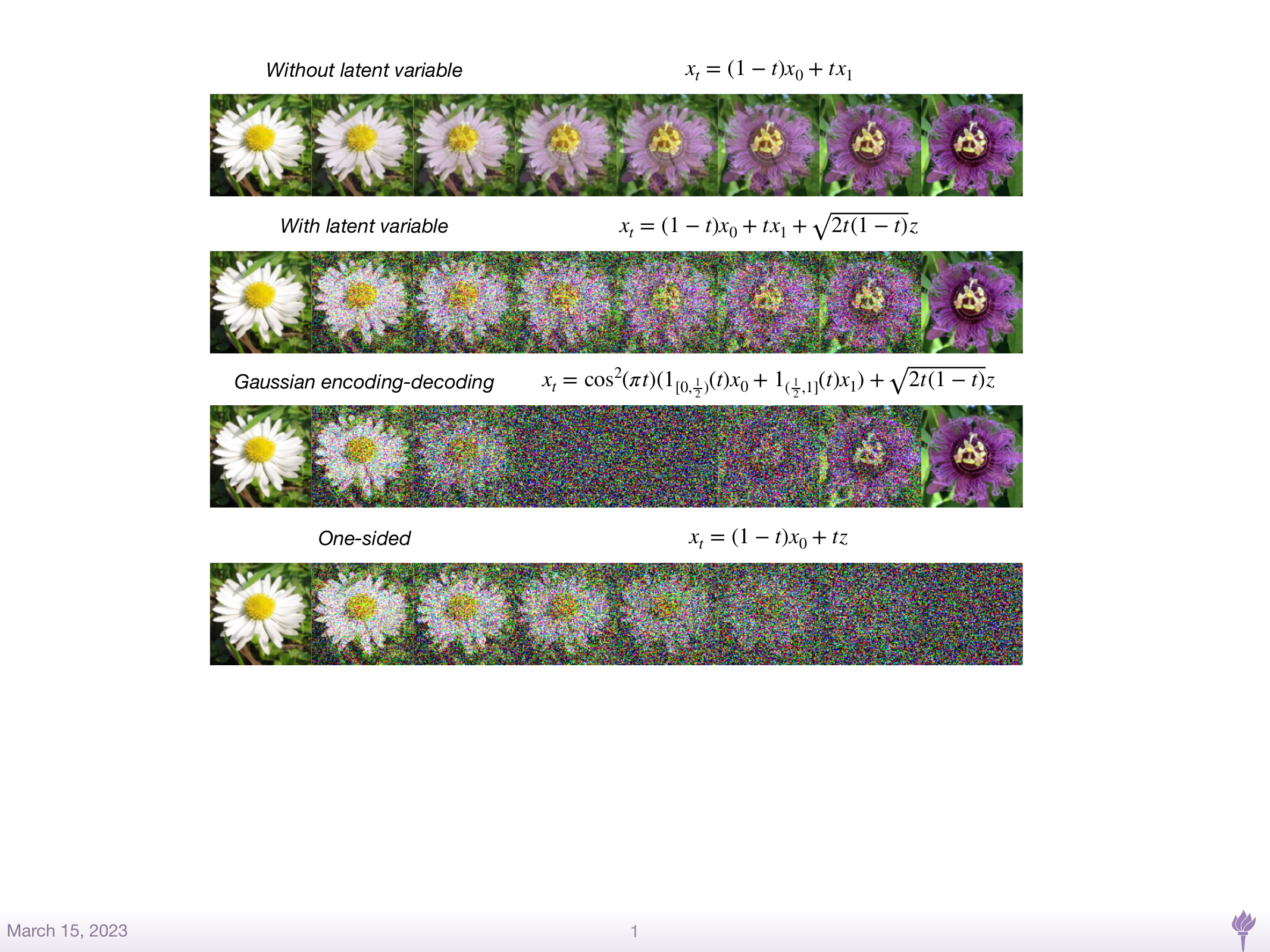}
    \caption{\textbf{Design flexibility.} An illustration of how stochastic interpolants can be tailored to specific aims. All examples show one realization of $x_t$ with one $x_0\sim\rho_0$, one $x_1\sim \rho_1$ (the flowers at the left and right of the figures), and one $z\sim {\sf N}(0,\Id )$. \textit{Top, Upper middle, and lower middle}: various interpolants, ranging from direct interpolation with no latent variable  (as in~\cite{albergo2023building}) to Gaussian encoding-decoding in which the data transitions to pure noise at the mid-point. \textit{Bottom}: one-sided interpolant, which connects with score-based diffusion methods. }
    \label{fig:example:xt}
\end{figure}

The stochastic interpolant $x_t$ in~\eqref{eq:stochinterp} is a continuous-time stochastic process whose realizations are samples from $\rho_0$ at time $t=0$ and from $\rho_1$ at time $t=1$ by construction. As a result, it offers a way to bridge $\rho_0$ and $\rho_1$ -- we are interested in characterizing the law of $x_t$ over the full interval $[0,1]$, as it will allow us to design generative models. Mathematically, we want to characterize the properties of the time-dependent probability distribution $\mu(t,dx)$  such that
\begin{equation}
    \label{eq:pdf:def:test:0}
    \forall t \in [0,1] \quad : \quad  \int_{\RR^d} \phi(x) \mu(t,dx)  = \EE [\phi(x_t)] \quad \text{for any test function} \quad \phi \in C^\infty_0(\RR^d),
\end{equation}
where $x_t$ is defined in~\eqref{eq:stochinterp} and the expectation  is taken independently over $(x_0,x_1)\sim\nu$, and $z\sim {\sf N}(0,\text{\it Id})$.
To this end, we will need to use conditional expectations over $x_t$\footnote{Formally, in terms of the Dirac delta distribution, we can write
$$
\EE\left[f(t,x_0,x_1,z)| x_t=x\right] = \frac{\EE\left[f(t,x_0,x_1,z)\delta(x-x_t)\right]}{\EE[\delta(x-x_t)]}
$$
and in this notation we also have $\mu(t,x) = \EE[\delta(x-x_t)$].}, as described in the following definition.
\new{
\begin{definition}
    \label{def:cond:expect}
    Given any $f:[0,1]\times \RR^d\times \RR^d\times\RR^d\to \RR$, its conditional expectation $\EE\left[f(t,x_0,x_1,z)| x_t=x\right]$ is the function of~$x$ such that,  for any test function $\phi \in C^\infty_0(\RR^d)$, we have
\begin{equation}
    \label{eq:cond:e}
    \forall t \in [0,1] \quad : \quad \int_{\RR^d} \phi(x) \EE\left[f(t,x_0,x_1,z)| x_t=x\right] \mu(t,dx) = \EE[ \phi(x_t) f(t,x_0,x_1,z)],
\end{equation}
where $\mu(t,dx)$ is the time-dependent distribution of $x_t$ defined by~\eqref{def:cond:expect}, and the expectation on the right-hand side is taken independently over $(x_0,x_1)\sim\nu$ and $z\sim {\sf N}(0,\text{\it Id})$ with $x_t$ given by~\eqref{eq:stochinterp}.
\end{definition}
}

Vector-valued functions have conditional expectations that are defined analogously. Note that, with our definition,  $\EE\left[f(t,x_0,x_1,z)| x_t=x\right]$ is a deterministic function of $(t,x)\in[0,1]\times \RR^d$, not to be confused with the random variable $\EE\left[f(t,x_0,x_1,z)| x_t\right]$ that can be defined analogously.

\begin{remark}
    Another seemingly more general way to define the stochastic interpolant is via
\begin{equation}
    \label{eq:stoch:interp:gen}
    x^\diff_t = I(t,x_0,x_1)+ N_t
\end{equation}
where $N: [0,1]\to \RR^d$ is a zero-mean Gaussian stochastic process constrained to satisfy $N_{t=0}=N_{t=1}=0$. As we will show below, our construction only depends on the single-time properties of~$N_t$, which are completely specified by $\EE [N_t N^\T_t]$. That is, if we take $\gamma(t)$ in~\eqref{eq:stochinterp} such that $\EE [N_t N^\T_t] = \gamma^2(t)\Id$, then the probability distribution of $x_t$ will coincide with that of $x'_t$ defined in~\eqref{eq:stoch:interp:gen}, $x_t \stackrel{\text{d}}{=} x^\diff_t$. For example, taking $\gamma(t) = \sqrt{t(1-t)}$ in~\eqref{eq:stochinterp} -- a choice we will consider below in Sec.~\ref{sec:sb} -- is equivalent to choosing $N_t$ to be a Brownian bridge in~\eqref{eq:stoch:interp:gen}, i.e. the stochastic process realizable in terms of the Wiener process $W_t$ as $N_t = W_t - t W_1$. This observation will also help us draw an analogy between our approach and the construction used in score-based diffusion models as well as methods based on stochastic bridges. As we will show in Sec.~\ref{sec:sb}, it is simpler for both analysis and practical implementation to work with the definition~\eqref{eq:stochinterp} for $x_t$.
\end{remark}

To proceed, we will make the following assumption on the densities $\rho_0$, $\rho_1$, and the interplay between the measure $\nu$ to the function~$I$:
\begin{assumption}
\label{as:rho:I}
The densities $\rho_0$ and $\rho_1$ are strictly positive elements of $C^2(\RR^d)$ and are such that
\begin{equation}
    \label{eq:rho0:1:sc}
     \int_{\RR^d} |\nabla \log \rho_0(x)|^2 \rho_0(x) dx < \infty \quad\text{and} \quad \int_{\RR^d} |\nabla \log \rho_1(x)|^2 \rho_1(x) dx < \infty.
\end{equation}
The measure $\nu$ and the function $I$ are such that
\begin{equation}
    \label{eq:It:L2}
    \exists M_1,M_2 < \infty  \ \ : \ \  \EE\big[ |\partial_t I(t,x_0,x_1)|^4\big] \le M_1; \quad \EE\big[ |\partial^2_t I(t,x_0,x_1)|^2\big] \le M_2, \quad  \forall t\in [0,1],
\end{equation}
where the expectation is taken  over $(x_0,x_1)\sim\nu$.
\end{assumption}
Note that for the interpolant~\eqref{eq:stoch:interp:lin}, Assumption~\ref{as:rho:I} holds if $\rho_0$ and $\rho_1$ both have finite fourth moments.

\begin{figure}[t!]
\centering
\includegraphics[width=\linewidth]{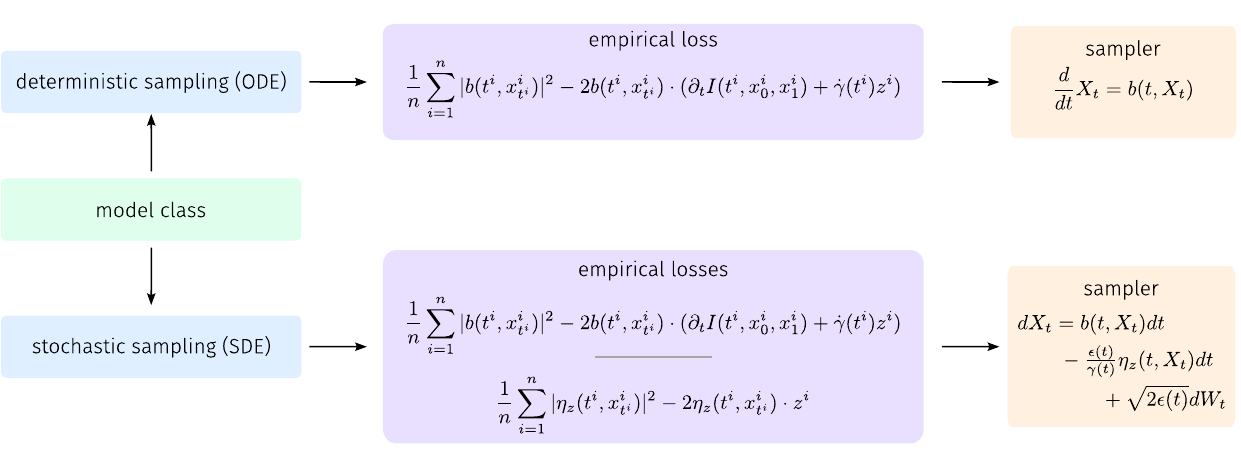}
\caption{\textbf{Algorithmic implementation.} 
A simple overview of suggested implementation strategies. 
For deterministic sampling, a single velocity field $b$ can be learned by minimizing the empirical loss in the top row.
For stochastic sampling, the velocity field $b$, along with the denoiser $\eta_z$, can be learned by minimizing the two empirical losses specified in the bottom row.
To sample deterministically, off-the-shelf ODE integrators can be used to integrate the probability flow equation.
To sample stochastically, the listed SDE can be integrated using standard techniques such as the Euler-Maruyama method or the Heun sampler introduced in~\cite{Karras2022edm}.
The time-dependent diffusion coefficient $\epsilon(t)$ can be specified \textit{after learning} to maximize sample quality.}
\label{fig:flow_chart}
\end{figure}

\subsection{Transport equations, score, and quadratic objectives}
\label{sec:cont:eq}
We now state a result that specifies some important properties of the probability distribution of the stochastic interpolant~$x_t$:

\begin{restatable}[Stochastic interpolant properties]{theorem}{interpolation}
\label{prop:interpolate}
The probability distribution of the stochastic interpolant $x_t$ defined in~\eqref{eq:stochinterp} is absolutely continuous with respect to the Lebesgue measure at all times $t\in[0,1]$ and its time-dependent density $\rho(t)$ satisfies $\rho(0) = \rho_0$, $\rho(1) = \rho_1$, $\rho \in C^1([0,1];C^p(\RR^d))$ for any $p\in\NN$, and $\rho(t, x) > 0$ for all $(t, x) \in [0, 1]\times \RR^d$. In addition, $\rho$ solves the transport equation
\begin{equation}
    \label{eq:transport}
    \partial_t \rho + \nabla \cdot \left(b_\ODE\rho\right) = 0, 
\end{equation}
where we defined the velocity
\begin{equation}
    \label{eq:b:ode:def}
    b_\ODE(t,x) = \EE [ \dot x_t| x_t = x] = \EE [ \partial_t I(t,x_0,x_1) + \dot \gamma(t) z| x_t = x].
\end{equation}
This velocity is in $C^0([0,1];(C^p(\RR^d))^d)$ for any $p\in \NN$, and  such that
\begin{equation}
    \label{eq:bt:bounded}
    \forall t\in [0,1] \quad : \quad \int_{\RR^d} |b_\ODE(t,x)|^2 \rho(t,x) dx < \infty.
\end{equation}
\end{restatable}

Note that this theorem means that we can write \eqref{eq:pdf:def:test:0} as
\begin{equation}
    \label{eq:pdf:def:test}
    \forall t \in [0,1] \quad : \quad  \int_{\RR^d} \phi(x) \rho(t,x)dx  = \EE \phi(x_t) \quad \text{for any test function} \quad \phi \in C^\infty_0(\RR^d).
\end{equation}
The transport equation~\eqref{eq:transport} can be solved either forward in time from the initial condition $\rho(0)=\rho_0$, in which case $\rho(1)=\rho_1$, or backward in time from the final condition $\rho(1)=\rho_1$, in which case $\rho(0)=\rho_0$. 

The proof of Theorem~\ref{prop:interpolate} is given in Appendix~\ref{app:proof:interpolate}; it mostly relies on manipulations involving the characteristic function of  the stochastic interpolant $x_t$. The transport equation~\eqref{eq:transport} for $\rho$ lead to methods for generative modeling and density estimation, as explained in Secs.~\ref{sec:generative} and \ref{sec:density}, provided that we can estimate the velocity~$b_\ODE$.  This velocity is explicitly available only in special cases, for example when $\rho_0$ and $\rho_1$ are both Gaussian mixture densities: this case is treated in Appendix~\ref{app:Gauss:mixt}.  In general  $b_\ODE$ must be calculated numerically, which can be performed via empirical risk minimization of a quadratic objective function, as characterized by our next result:

\begin{restatable}[Objective]{theorem}{interpolatelosses}
\label{prop:interpolate_losses}

The velocity $b_\ODE$ defined in~\eqref{eq:b:ode:def} is the unique minimizer in $C^0([0,1];(C^1(\RR^d))^d)$ of the quadratic objective
\begin{equation}
    \label{eq:obj:v}
    \mathcal{L}_b[\hat{b}] =\int_0^1   \EE \left( \tfrac12|\hat b(t,x_t)|^2 - \left(\partial_t I(t,x_0,x_1) + \dot \gamma(t) z \right) \cdot \hat b(t,x_t) \right) dt
\end{equation}
where $x_t$ is defined in~\eqref{eq:stochinterp} and the expectation is taken independently over $(x_0,x_1)\sim\nu$ and $z\sim {\sf N}(0,\text{\it Id}).$

\end{restatable}
The proof of Theorem~\ref{prop:interpolate_losses}  is given in Appendix~\ref{app:proof:interpolate}: it relies on the definitions of $b_\ODE$  in~\eqref{eq:b:ode:def}, as well as the definition of $\rho$ in~\eqref{eq:pdf:def:test} and some elementary properties of the conditional expectation. We discuss how to estimate the objective function~\eqref{eq:obj:v} in practice in Section~\ref{sec:practical}. Interestingly, we also have access to the score of the probability density, as shown by our next result:

\begin{restatable}[Score]{theorem}{score}
\label{thm:score}
    The score of the probability density $\rho$ specified in Theorem~\ref{prop:interpolate} is  in $C^1([0,1];(C^p(\RR^d))^d)$ for any $p\in\NN$ and given by
    \begin{equation}
        \label{eq:s:def} 
        s(t,x) = \nabla \log\rho(t,x) = - \gamma^{-1}(t) \EE( z |x_t=x) \quad \forall (t,x)\in(0,1)\times \RR^d
    \end{equation}
    In addition it satisfies 
    \begin{equation}
    \label{eq:st:bounded}
    \forall t\in [0,1] \quad : \quad \int_{\RR^d} |s(t,x)|^2 \rho(t,x) dx < \infty,
\end{equation}
and is the unique minimizer  in $C^1([0,1];(C^1(\RR^d))^d)$ of the quadratic objective 
\begin{equation}
    \label{eq:obj:s}
    \mathcal{L}_s[\hat{s}] = \int_0^1 \EE\left( \tfrac12|\hat s(t,x_t)|^2 +\gamma^{-1}(t) z\cdot \hat s(t,x_t) \right) dt
\end{equation}
where $x_t$ is defined in~\eqref{eq:stochinterp} and the expectation is taken independently over $(x_0,x_1)\sim\nu$ and  $z\sim {\sf N}(0,\text{\it Id})$
\end{restatable}

The proof of Theorem~\ref{thm:score} is given in Appendix~\ref{app:proof:interpolate}. We stress that the objective function is well defined despite the fact that $\gamma(0)=\gamma(1) =0$: see Section~\ref{sec:practical} for more details about how to evaluate this objective in practice.

\begin{remark}[Denoiser]
    The quantity 
    \begin{equation}
    \label{eq:denoiser}
        \eta_z(t,x) = \EE(z | x_t = x),
    \end{equation}
    will be referred as the \textit{denoiser}, for reasons that will be made clear in Section \ref{sec:denoiser}. By \eqref{eq:s:def}, this quantity gives access to the score on $t\in(0,1)$ (where $\gamma(t)>0$) since, from~\eqref{eq:s:def},
    \begin{equation}
        s(t,x) = - \gamma^{-1}(t) \eta_z(t,x).
    \end{equation}  
    This denoiser is the minimizer of an equivalent expression to \eqref{eq:obj:s},
\begin{equation}
    \label{eq:obj:eta:0}
    \mathcal{L}_{\eta_z}[\hat{\eta}_z] = \int_0^1 \EE\left( \tfrac12|\hat \eta_z(t,x_t)|^2 - z\cdot \hat \eta_z(t,x_t) \right) dt.
\end{equation}
The denoiser~$\eta_z$ is useful for numerical realizations.
In particular, the objective in~\eqref{eq:obj:eta:0} is easier to use than the one in~\eqref{eq:obj:s} because it does not contain the factor $\gamma^{-1}(t)$, which needs careful handling as $t$ approaches 0 and 1.
\end{remark}

Having access to the score immediately allows us to rewrite the TE~\eqref{eq:transport} as forward and backward Fokker-Planck equations, which we state as:

\begin{restatable}[Fokker-Planck equations]{corollary}{interpolationfpe}
\label{prop:interpolate_fpe}
For any $\eps\in C^0([0,1])$ with $\eps(t)\ge 0$ for all $t\in[0,1]$, the probability density $\rho$ specified in Theorem~\ref{prop:interpolate} satisfies:
\begin{enumerate}[leftmargin=0.15in]
\item The forward Fokker-Planck equation
\begin{equation}
    \label{eq:fpe}
    \partial_t \rho + \nabla \cdot \left(b_{\fwd}\rho\right) = \eps(t)  \Delta \rho, \qquad \rho(0) = \rho_0,
\end{equation}
where we defined the forward drift
\begin{equation}
    \label{eq:b:def}
    b_\fwd(t,x) = b_\ODE(t,x) + \eps(t) s(t,x).
\end{equation}
Equation~\eqref{eq:fpe} is well-posed when solved forward in time from $t=0$ to $t=1$, and its solution for the initial condition $\rho(t=0) = \rho_0$ satisfies $\rho(t=1) = \rho_1$. 
\item  The backward Fokker-Planck equation
\begin{equation}
    \label{eq:fpe:tr}
    \partial_t \rho + \nabla \cdot \left(b_\rev\rho\right) = -\eps(t)  \Delta \rho, \qquad \rho(1) = \rho_1,
\end{equation}
where we defined the backward drift
\begin{equation}
    \label{eq:b:r:def}
    b_\rev(t,x) = b_\ODE(t,x) - \eps(t) s(t,x).
\end{equation}
Equation~\eqref{eq:fpe:tr} is well-posed when solved backward in time from $t=1$ to $t=0$, and its solution for the final condition $\rho(1) = \rho_1$ satisfies $\rho(0) = \rho_0$. 
\end{enumerate}
\end{restatable}

In Section~\ref{sec:generative} we will use the results of this theorem to design generative models based on forward and backward stochastic differential equations. 
Note that we can replace the diffusion coefficient $\eps(t)$ by  a positive semi-definite tensor; also note that if we define $\rho_\rev(t_\rev,x) = \rho(1-t_\rev,x)$, the reversed FPE~\eqref{eq:fpe:tr} can be written as
\begin{equation}
    \label{eq:fpe:tr:r}
    \partial_{t_\rev} \rho_\rev({t_\rev},x) - \nabla \cdot \left(b_\rev(1-{t_\rev},x)\rho_\rev({t_\rev},x)\right) = \eps(1-t_\rev)  \Delta \rho_\rev({t_\rev},x), \qquad \rho_\rev({t_\rev}=0) = \rho_1,
\end{equation}
which is now well-posed forward in (reversed) time~$t_\rev$. 
So as to have only one definition of time~$t$, it is more convenient to work with~\eqref{eq:fpe:tr}.  

\medskip

Let us make a few remarks about the statements made so far:

\begin{remark}
    If we set $\gamma(t)=0$ in $x_t$ (i.e, if we remove the latent variable), the stochastic interpolant~\eqref{eq:stochinterp} reduces to the one originally considered in~\cite{albergo2023building}. In this setup, the results above formally stand except that we cannot guarantee the spatial regularity of $b_\ODE(t,x)$ and $s(t,x)$, since it relies on the presence of the latent variable (as shown in the proof of Theorem~\ref{prop:interpolate}). Hence, we expect the introduction of the latent variable $\gamma(t) z$ to help for generative modeling, where the solution to the corresponding ODEs/SDEs will be better behaved, and for statistical approximation, since the targets $b$ and $s$ will be more regular. We will see in Section~\ref{sec:practical} that it also gives us much greater flexibility in the way we can bridge $\rho_0$ and $\rho_1$, which will enable us to design generative models with appealing properties.
\end{remark}
\begin{remark} 
    We will see in Section~\ref{sec:likelihood_bounds} that the forward and backward FPE in~\eqref{eq:fpe} and \eqref{eq:fpe:tr} are more robust  than the TE in~\eqref{eq:transport} against approximation errors in the velocity $b_\ODE$ and the score $s$, which has practical implications for generative models based on these equations. 
\end{remark}

\begin{remark}
    We could also obtain $b_\ODE(t,\cdot)$ at any $t\in[0,1]$  by minimizing
\begin{equation}
    \label{eq:obj:vt}
    \EE \left( \tfrac12|\hat b(t,x_t)|^2 - \left(\partial_t I(t,x_0,x_1)+ \dot \gamma(t) z\right) \cdot \hat b(t,x_t) \right) \qquad t\in [0,1]
\end{equation}
and $s(t,\cdot)$ at any $t\in(0,1)$ by minimizing 
\begin{equation}
    \label{eq:obj:wt}
    \EE \left( \tfrac12|\hat s(t,x_t)|^2 +\gamma^{-1}(t) z\cdot \hat s(t,x_t) \right) \qquad t\in (0,1)
\end{equation}
Using the time-integrated versions of these objectives given in~\eqref{eq:obj:v} and~\eqref{eq:obj:s} is more convenient numerically as it allows one to parameterize  $\hat b_\ODE$ and $\hat s$ globally for $(t,x)\in [0,1]\times \RR^d$.
\end{remark}

\begin{remark}
From~\eqref{eq:b:ode:def} we can write
\begin{equation}
    \label{eq:b:decomp}
    b(t,x) = v(t,x) - \dot \gamma(t) \gamma(t) s(t,x),
\end{equation}
where $s$ is the score given in~\eqref{eq:s:def} and we defined the velocity field
\begin{equation}
    \label{eq:v:def}
    v(t,x) = \EE ( \partial_t I(t,x_0,x_1) | x_t = x).
\end{equation}
The velocity field $v \in C^0([0,1];(C^p(\RR^d))^d)$ for any $p\in \NN$ and can be characterized as the unique minimizer of
\begin{equation}
    \label{eq:obj:vv}
    \mathcal{L}_v[\hat{v}] = \int_0^1 \EE\left( \tfrac12|\hat v(t,x_t)|^2 -\partial_tI(t,x_0,x_1) \cdot \hat v(t,x_t) \right) dt
\end{equation}
Learning this  velocity and the score separately may be useful in practice. 
\end{remark}

\begin{remark}
    The objectives in~\eqref{eq:obj:v} and \eqref{eq:obj:s} (as well as the ones in\eqref{eq:obj:eta:0} and~\eqref{eq:obj:vv}) are amenable to empirical estimation if we have samples $(x_0,x_1)\sim\nu$, since in that case we can generate samples of $x_t= I(t,x_0,x_1) + \gamma(t) z$ at any time $t\in[0,1]$. We will use this feature in the numerical experiments presented below.
\end{remark}
\begin{remark}
    Since $s$ is the score of $\rho$, an alternative objective to estimate it is~\citep{hyvarinen05a} 
\begin{equation}
    \label{eq:ob:w:alt}
     \int_0^1 \EE\left(|\hat s(t,x_t)|^2 +2\nabla \cdot \hat s(t,x_t) \right)dt.
\end{equation}
The derivation of~\eqref{eq:ob:w:alt} is standard:
for the reader's convenience we recall it at the end of Appendix~\ref{app:proof:interpolate}. The advantage of using~\eqref{eq:obj:s} over~\eqref{eq:ob:w:alt} is that it does not require us to take the divergence of $\hat s$.
\end{remark}

\begin{remark}[Energy-based models] By definition, the score $s(t,x)=\nabla \log \rho(t,x)$ is a gradient field. As a result, if we model $\hat s(t,x) = -\nabla \hat E(t,x) $, we can turn \eqref{eq:obj:s} into an objective function for $\hat E(t,x)$
\begin{equation}
    \label{eq:obj:E}
    \mathcal{L}_E[\hat{E}] = \int_0^1 \EE\left( \tfrac12|\nabla \hat E(t,x_t)|^2 +\gamma^{-1}(t) z\cdot \nabla\hat E(t,x_t) \right) dt
\end{equation}
This objective is invariant to constant shifts in $\hat E$ and should therefore be minimized under some constraint, such as $\min_x \hat E(t,x) =0$ for all $t\in [0,1]$.
The minimizer of~\eqref{eq:obj:E} provides us with an energy-based model (EBM)~\citep{lecun2006tutorial,song2021train} that can in principle be used to sample the PDF of the stochastic interpolant, $\rho(t,x)$, at any fixed $t\in [0,1]$ using e.g. Langevin dynamics. 
We will not exploit this possibility here, and instead rely on generative models to sample $\rho(t,x)$, as discussed next in Sec.~\ref{sec:generative}.
\end{remark}

\subsection{Generative models}
\label{sec:generative}

Our next result is a direct consequence of Theorem~\ref{prop:interpolate}, and it shows how to design generative models using the stochastic processes associated with the TE~\eqref{eq:transport}, the forward FPE~\eqref{eq:fpe}, and the backward FPE~\eqref{eq:fpe:tr}:

\begin{restatable}[Generative models]{corollary}{generative}
\label{prop:generative}
At any time $t\in[0,1]$, the law of the stochastic interpolant~$x_t$ coincides with the law of the three processes $X_t$, $X^\fwd_t$, and $X^\rev_t$, respectively defined as:
\begin{enumerate}[leftmargin=0.15in]
\item The solutions of the probability flow associated with the transport equation~\eqref{eq:transport}
\begin{equation}
    \label{eq:ode:1}
    \frac{d}{dt}  X_t = b_\ODE(t, X_t),
\end{equation}
solved either forward in time from the initial data $X_{t=0} \sim\rho_0$ or backward in time from the final data $X_{t=1} = x_1\sim\rho_1$. 
\item The solutions of the forward SDE associated with the FPE~\eqref{eq:fpe}
\begin{equation}
    \label{eq:sde:1}
    dX^\fwd_t = b_\fwd(t,X^\fwd_t)dt  + \sqrt{2\eps(t)} \,  dW_t,
\end{equation}
solved forward in time from the initial data~$X^\fwd_{t=0}\sim\rho_0$ independent of $W$.

\item  The solutions of the backward SDE associated with the backward FPE~\eqref{eq:fpe:tr}
\begin{equation}
    \label{eq:sde:R}
    dX^\rev_t = b_\rev(t,X^\rev_t)dt  + \sqrt{2\eps(t)} \,    dW^\rev_t, \quad W_t^\rev = -W_{1-t},
\end{equation}
solved backward in time from the final data~$X^\rev_{t=1}\sim\rho_1$ independent of $W^\rev$; the solution of~\eqref{eq:sde:R} is by definition $X^\rev_{t}= Z^\fwd_{1-t}$ where $Z^\fwd_t$ satisfies
\begin{equation}
    \label{eq:sde:R:Y}
        dZ^\fwd_t = -b_\rev(1-t,Z^\fwd_t)dt+ \sqrt{2\eps(t)} \,    dW_t,
\end{equation}
solved forward in time from the initial data~$Z^\fwd_{t=0}\sim \rho_1$ independent of $W$.
\end{enumerate}
\end{restatable}

To avoid repeated applications of the transformation $t\mapsto1-t$, it is convenient to work with \eqref{eq:sde:R} directly using the reversed It\^o calculus rules stated in the following lemma, which follows from the results in~\cite{anderson1979reverse-time} and is proven in Appendix~\ref{app:proof:generative}:

\begin{restatable}[Reverse It\^o Calculus]{lemma}{reversed}
\label{lem:reversed} If $X^\rev_t $ solves the backward SDE \eqref{eq:sde:R}:
\begin{enumerate}
    \item For any $f\in C^1([0,1];C_0^2(\RR_d))$ and $t\in[0,1]$, the backward It\^o formula holds
\begin{equation}
    \label{eq:ito:formula}
    df(t,X^\rev_t) = \partial_t f(t,X^\rev_t) dt+ \nabla f(X^\rev_t) \cdot dX^\rev_t - \eps(t) \Delta f(t,X^\rev_t) dt.
\end{equation}
    \item For any $g\in C^0([0,1];(C_0(\RR_d))^d)$ and $t\in[0,1]$, the  backward It\^o isometries hold:
\begin{equation}
    \label{eq:ito:iso}
    \begin{aligned}
    \EE^x_\rev\int_t^1 g(t,X^\rev_t) \cdot dW^\rev_t = 0; \qquad \EE^x_\rev\left|\int_t^1 g(t,X^\rev_t)\cdot  dW^\rev_t\right|^2 &= \int_t^1 \EE^x_\rev\left|g(t,X^\rev_t)\right|^2 dt,
    \end{aligned}
\end{equation}
where $\EE^x_\rev$ denotes expectation conditioned on the event $X_{t=1}^\rev = x$.
\end{enumerate}
\end{restatable}
 
The relevance of Corollary~\ref{prop:generative} for generative modeling is clear. Assuming, for example, that $\rho_0$ is a simple density that can be sampled easily (e.g. a Gaussian or a Gaussian mixture density), we can use the ODE~\eqref{eq:ode:1} or the SDE~\eqref{eq:sde:1}  to push these samples forward in time and generate samples from a complex target density~$\rho_1$. In Section~\ref{sec:density}, we will show how to use the ODE~\eqref{eq:ode:1} or the reverse SDE~\eqref{eq:sde:R} to estimate $\rho_1$ at any $x\in\RR^d$ assuming that we can evaluate $\rho_0$ at any $x\in\RR^d$. We will also show how similar ideas can be used to estimate the cross entropy between $\rho_0$ and $\rho_1$. 

\begin{remark}
    We stress that the stochastic interpolant $x_t$, the solution $X_t$ to the ODE~\eqref{eq:ode:1}, and the solutions $X^\fwd_t$ and $X^\rev_t$ of the forward and backward SDEs~\eqref{eq:sde:1} and \eqref{eq:sde:R} are \textit{different} stochastic processes, but their laws all coincide with $\rho(t)$ at any time $t\in[0,1]$. This is all that matters when applying these processes as generative models. 
    However, the fact that these processes are different has implications for the accuracy of the numerical integration used to sample from them at any $t$ as well as for the propagation of statistical errors (see also the next remark).
\end{remark}

Generative models based on solutions $X_t$ to the ODE~\eqref{eq:ode:1}, solutions $X^\fwd_t$ to the forward SDE~\eqref{eq:sde:1}, and solutions $X^\rev_t$ to the backward SDE~\eqref{eq:sde:R} will typically involve drifts $b$, $b_\fwd$, and $b_\rev$ that are, in practice, imperfectly estimated via minimization of~\eqref{eq:obj:v} and~\eqref{eq:obj:s} over finite datasets. It is important to estimate how this statistical estimation error propagates to errors in sample quality, and how the propagation of error depends on the generative model used, which is the object of our next section.

\subsection{Likelihood control}
\label{sec:likelihood_bounds}
In this section, we demonstrate that jointly minimizing the objective functions~\eqref{eq:obj:vv} and~\eqref{eq:obj:s} (or the losses~\eqref{eq:obj:v} and~\eqref{eq:obj:s}) controls the $\mathsf{KL}$-divergence from the target density $\rho_1$ to the model density $\hat{\rho}_1$. We focus on bounds involving the score, but we note that analogous results hold for learning the denoiser $\eta_z(t,x)$ defined in~\eqref{eq:denoiser} by the relation $\eta_z(t, x) = -s(t, x) / \gamma(t)$. The derivation is based on a simple and exact characterization of the $\mathsf{KL}$-divergence between two transport equations or two Fokker-Planck equations with different drifts. Remarkably, we find that the presence of a diffusive term determines whether or not it is sufficient to learn the drift to control $\mathsf{KL}$. This can be seen as a generalization of the result for score-based diffusion models described in~\cite{song2021mle} to arbitrary generative models described by ODEs or SDEs. The proofs of the statements in this section are provided in Appendix~\ref{app:proof:kl}. \new{These proofs rely on manipulation of the time derivative of the $\mathsf{KL}$ divergence, which is a practice that has proven useful elsewhere in the literature \citep{vempala_wibisono_2019}.}

We first characterize the $\mathsf{KL}$ divergence between two densities transported by two different continuity equations but initialized from the same initial condition:

\begin{restatable}{lemma}{kltransport}
\label{lemma:kl_transport}
Let $\rho_0: \RR^d\rightarrow\RR_{\geq 0}$ denote a fixed base probability density function. Given two velocity fields $b_\ODE, \hat{b}_\ODE \in C^0([0,1], (C^1(\RR^d))^d)$, let the time-dependent densities $\rho: [0,1]\times \RR^d \to \RR_{\ge0}$ and $\hat \rho: [0,1]\times \RR^d \to \RR_{\ge0}$ denote the solutions to the transport equations
\begin{equation}
    \label{eq:2:te}
    \begin{aligned}
       &\partial_t\rho + \nabla \cdot(b_\ODE \rho) = 0,\qquad &&\rho(0)=\rho_0,\\ 
       &\partial_t\hat\rho + \nabla \cdot(\hat b_\ODE\hat \rho) = 0,\qquad &&\hat \rho(0)=\rho_0.
    \end{aligned} 
\end{equation}
 Then, the Kullback-Leibler divergence of $\rho(1)$ from $\hat\rho(1)$ is given by
\begin{equation}
    \label{eq:kl_te}
    \KL{\rho(1)}{\hat\rho(1)} = \int_0^1 \int_{\RR^d} \left(\nabla\log\hat\rho(t,x) - \nabla\log\rho(t,x)\right)\cdot\big(\hat b_\ODE(t,x) - b_\ODE(t,x)\big)\rho(t,x) dxdt.
\end{equation}
\end{restatable}

Lemma~\ref{lemma:kl_transport} shows that it is insufficient in general to match $\hat b$ with $b$ to obtain control on the $\mathsf{KL}$ divergence. The essence of the problem is that a small error in $\hat b - b$ does not ensure control on the Fisher divergence $\fisher{\rho(t)}{\hat\rho(t)} = \int_{\RR^d}\norm{\nabla\log\rho(t,x) - \nabla\log\hat \rho(t,x)}^2\rho(t,x)dx$, which is necessary due to the presence of $\left(\nabla\log\hat\rho - \nabla\log\rho\right)$ in~\eqref{eq:kl_te}.

In the next lemma, we study the case for two Fokker-Planck equations, and highlight that the situation becomes quite different.
\begin{restatable}{lemma}{klfpe}
\label{lemma:kl_fpe}
Let $\rho_0: \RR^d\rightarrow\RR_{\geq 0}$ denote a fixed base probability density function. Given two velocity fields $b_\fwd, \hat{b}_\fwd \in C^0([0,1], (C^1(\RR^d))^d)$,  let the time-dependent densities $\rho: [0,1]\times \RR^d \to \RR_{\ge0}$ and  $\hat \rho: [0,1]\times \RR^d \to \RR_{\ge0}$ denote the solutions to the Fokker-Planck equations
\begin{equation}
\label{eq:2:fpe}
    \begin{aligned}
       &\partial_t\rho + \nabla \cdot(b_\fwd \rho) = \eps\Delta \rho,\qquad &&\rho(0)=\rho_0,\\ 
       &\partial_t\hat\rho + \nabla \cdot(\hat b_\fwd\hat \rho) = \eps \Delta \hat \rho,\qquad &&\hat \rho(0)=\rho_0.
    \end{aligned} 
\end{equation}
where $\eps>0$.
Then, the Kullback-Leibler divergence from $\rho(1)$ to $\hat\rho(1)$ is given by
\begin{equation}
\begin{aligned}
    \KL{\rho(1)}{\hat\rho(1)} &= \int_0^1\int_{\RR^d} \left(\nabla\log\hat\rho(t,x) - \nabla\log\rho(t,x)\right)\cdot \left(\hat b_\fwd(t,x) - b_\fwd(t,x)\right)\rho(t,x) dxdt\\
    &\qquad -\eps\int_0^1 \int_{\RR^d} \norm{\nabla\log\rho(t,x) - \nabla\log\hat \rho(t,x)}^2\rho(t,x)dx dt,
\end{aligned}
\end{equation}
and as a result
\begin{equation}
\begin{aligned}
    \KL{\rho(1)}{\hat\rho(1)} &\leq \frac{1}{4\eps}\int_0^1 \int_{\RR^d} \norm{\hat b_\fwd(t,x) - b_\fwd(t,x)}^2\rho(t,x) dxdt.
\end{aligned}
\end{equation}
\end{restatable}
Lemma~\ref{lemma:kl_fpe} shows that, unlike for transport equations, the $\mathsf{KL}$-divergence between the solutions of two Fokker-Planck equations is controlled by the error in their drifts. The diffusive term in each Fokker-Planck equation provides an additional negative term in the $\mathsf{KL}$-divergence, which eliminates the need for explicit control on the Fisher divergence. 

Putting the above results together, we can state the following result, which demonstrates that the losses~\eqref{eq:obj:v} and~\eqref{eq:obj:s} control the likelihood for learned approximations to the FPE~\eqref{eq:fpe}.

\begin{restatable}{theorem}{likelihoodbound}
\label{thm:kl:bound}
Let $\rho$ denote the solution of the Fokker-Planck equation~\eqref{eq:fpe} with $\eps(t)=\eps>0$. Given two velocity fields $\hat{b}, \hat{s} \in C^0([0,1], (C^1(\RR^d))^d)$, define 
\begin{equation}
    \label{eq:bhat:vhat}
    \hat b_\fwd(t,x) = \hat b(t,x) + \eps \hat s(t,x), \qquad \hat v(t,x) = \hat b(t,x) + \gamma(t) \dot \gamma(t) \hat s(t,x)
\end{equation} 
where the function $\gamma$ satisfies the properties listed in Definition~\ref{def:interp}. Let $\hat \rho$ denote the solution to the Fokker-Planck equation
\begin{equation}
     \partial_t \hat \rho + \nabla \cdot (\hat b_\fwd\hat \rho) = \eps  \Delta \hat \rho, \qquad \hat \rho(0) = \rho_0.
\end{equation}
Then,
\begin{equation}
\label{eq:bound:kl}
    \KL{\rho_1}{\hat{\rho}(1)} \leq \frac{1}{2\eps}\left(\mathcal{L}_b[\hat{b}] - \min_{\hat{b}}\mathcal{L}_b[\hat{b}]\right)  + \frac{\eps}{2}\left(\mathcal{L}_s[\hat{s}] - \min_{\hat{s}}\mathcal{L}_s[\hat{s}]\right),
\end{equation}
where $\mathcal{L}_b[\hat{b}]$ and $\mathcal{L}_s[\hat{s}]$ are the objective functions defined in~\eqref{eq:obj:v} and~\eqref{eq:obj:s}, and 
\begin{equation}
    \KL{\rho_1}{\hat\rho(1)} \leq \frac{1}{2\eps}\left(\mathcal{L}_{v}[\hat{v}] - \min_{\hat{v}}\mathcal{L}_{v}[\hat{v}]\right) + \frac{\sup_{t\in [0 ,1]}(\gamma(t)\dot\gamma(t) - \epsilon)^2}{2\eps}\left(\mathcal{L}_{s}[\hat{s}] - \min_{\hat{v}}\mathcal{L}_{s}[\hat{s}]\right).
\end{equation}
where $\mathcal{L}_v[\hat{v}]$ is the objective function defined in~\eqref{eq:obj:vv}.
\end{restatable}

\begin{remark}[Generative modeling]
The above results have practical ramifications for generative modeling. In particular, they show that minimizing either the losses~\eqref{eq:obj:v} and~\eqref{eq:obj:s} or~\eqref{eq:obj:vv} and~\eqref{eq:obj:s} maximize the likelihood of the stochastic generative model
\begin{equation}
   d\hat{X}_t^\fwd = \left(\hat{b}(t, \hat{X}_t^\fwd) + \epsilon\hat{s}(t, \hat{X}_t^\fwd)\right)dt + \sqrt{2\eps}dW_t,
\end{equation}
but that minimizing the objective~\eqref{eq:obj:v} is insufficient in general to maximize the likelihood of the deterministic generative model
\begin{equation}
   \dot{\hat{X}}_t = \hat{b}(t, \hat{X}_t).
\end{equation}
Moreover, they show that, when learning $\hat{b}$ and $\hat{s}$, the choice of $\eps$ that minimizes the upper bound is given by
\begin{equation}
\label{eq:optimal:eps}
    \eps^* = \left(\frac{\mathcal{L}_b[\hat{b}] - \min_{\hat{b}}\mathcal{L}_b[\hat{b}]}{\mathcal{L}_s[\hat{s}] - \min_{\hat{s}}\mathcal{L}_s[\hat{s}]}\right)^{1/2},
\end{equation}
so that $\eps^* > 1$ if the score is learned to higher accuracy than $\hat{b}$ and  $\eps^*<1$ in the opposite situation. Note that \eqref{eq:optimal:eps} suggests to take $\eps= 0$ if $\hat{b}$ is learned perfectly but $\hat s$ is not, and send $\eps \to \infty$ in the opposite situation. While taking $\eps=0$ is achievable in practice and leads to the ODE~\eqref{eq:ode:1}, taking $\eps\to\infty$ is not, as increasing $\eps$ increases the expense of the numerical integration in~\eqref{eq:sde:1} and~\eqref{eq:sde:R}.
\end{remark}

\subsection{Density estimation and cross-entropy calculation}
\label{sec:density}

It is well-known that the solution of the TE~\eqref{eq:transport} can be expressed in terms of the solution to the probability flow ODE~\eqref{eq:ode:1}; for completeness, we now recall this fact:
\begin{restatable}{lemma}{TEs}
    \label{lem:tesol}
    Given the velocity field $\hat b_\ODE  \in C^0([0,1], (C^1(\RR^d))^d)$, let $\hat\rho$ satisfy the transport equation
\begin{equation}
    \label{eq:TE:hat}
    \partial_t \hat\rho + \nabla \cdot(\hat b \hat \rho) = 0,
\end{equation}
and let $X_{s,t}(x)$ solve the ODE
\begin{equation}
    \label{eq:ode:st}
    \frac{d}{dt} X_{s,t}(x) = b(t,X_{s,t}(x)), \qquad X{_{s,s}(x) = x, \qquad t,s \in [0,1]}
\end{equation}
Then, given the PDFs $\rho_0$ and $\rho_1$:
\begin{enumerate}[leftmargin=0.15in]
\item The solution to~\eqref{eq:TE:hat} for the initial condition $\hat \rho(0) = \rho_0$ is given at any time $t\in[0,1]$ by 
\begin{equation}
    \label{eq:TE:hat:s:f}
    \hat \rho(t,x) = \exp\left( - \int_0^t \nabla \cdot b_{\ODE}(\tau, X_{t,\tau}(x)) d\tau \right) \rho_0( X_{t,0}(x))
\end{equation}
\item The solution to~\eqref{eq:TE:hat} for the final condition $\hat \rho(1) = \rho_1$ is given at any time $t\in[0,1]$ by 
\begin{equation}
    \label{eq:TE:hat:s:b}
    \hat \rho(t,x) = \exp\left(  \int_t^1 \nabla \cdot b_{\ODE}(\tau, X_{t,\tau}(x)) d\tau \right) \rho_1( X_{t,1}(x))
\end{equation}
\end{enumerate}

\end{restatable}
The proof of Lemma~\ref{lem:tesol} can be found in Appendix~\ref{app:prrof:de}.  Interestingly, we can obtain a similar result for the solution of the forward and backward FPEs in~\eqref{eq:fpe} and~\eqref{eq:fpe:tr}. These results make use of auxiliary forward and backward SDEs in which the roles of the forward and backward drifts are switched:

\begin{restatable}{theorem}{FK}
    \label{prop:sde:rho}
    Given $\eps>0$ and two velocity fields $\hat b_\ODE, \hat s \in C^0([0,1], (C^1(\RR^d))^d)$, define
\begin{equation}
    \label{eq:hatb:hats}
    \hat b_\fwd(t,x) = \hat b(t,x) + \eps \hat s(t,x), \qquad \hat b_\rev(t,x) = \hat b(t,x) - \eps \hat s(t,x),
\end{equation}
and let $Y^\fwd_t$ and $Y^\rev_t$ denote solutions of the following forward and backward SDEs:
\begin{equation}
    \label{eq:sde:y:1}
    dY^\fwd_t = b_\rev(t,Y^\fwd_t)dt  + \seps  dW_t,
\end{equation}
to be solved forward in time from the initial condition~$Y^\fwd_{t=0}=x$ independent of $W$; and 
\begin{equation}
    \label{eq:sde:y:R}
    dY^\rev_t = b_\fwd(t,Y^\rev_t)dt  + \seps  dW^\rev_t, \quad W_t^\rev = -W_{1-t},
\end{equation}
to be solved backwards in time from the final condition~$Y^\rev_{t=1}=x$ independent of $W^\rev$. Then, given the densities $\rho_0$ and $\rho_1$:
    \begin{enumerate}[leftmargin=0.15in]
\item The solution to the forward FPE
\begin{equation}
    \label{eq:fpe:f:hat}
    \partial_t \hat \rho_\fwd + \nabla \cdot (\hat b_\fwd \hat \rho_\fwd ) = \eps \Delta \hat \rho_\fwd , \qquad \hat \rho_\fwd (0) = \rho_0,
\end{equation}
can be expressed at $t=1$ as
\begin{equation}
    \label{eq:fk}
    \hat \rho_\fwd (1,x) = \EE_\rev^x \left(\exp\left(- \int_0^1\nabla \cdot \hat b_\fwd(t, Y^\rev_t) dt\right)  \rho_0(Y_{t=0}^\rev)\right),
\end{equation}
where $\EE_\rev^x$ denotes expectation on the path of $Y_t^\rev$ conditional on the event $Y_{t=1}^\rev = x$.

\item The solution to the backward FPE
\begin{equation}
    \label{eq:fpe:r:hat}
    \partial_t \hat \rho_\rev+ \nabla \cdot (\hat b_\rev \hat \rho_\rev) = -\eps \Delta \hat \rho_\rev, \qquad \hat \rho_\rev(1) = \rho_1,
\end{equation}
can be expressed at any $t=0$ as 
\begin{equation}
    \label{eq:fk:rev}
    \hat \rho_\rev(0, x) = \EE_\fwd^x \left(\exp\left( \int_0^1\nabla \cdot \hat b_\rev(t, Y^\fwd_t)  dt\right)\rho_1(Y^\fwd_{t=1})\right),
\end{equation}
where $\EE_\fwd^x$ denotes expectation on the path of $Y^\fwd_t$ conditional on $Y^\fwd_{t=0} = x$.
\end{enumerate}
\end{restatable}

The proof of Theorem~\ref{prop:sde:rho} can be found in Appendix~\ref{app:prrof:de}.  Note that to generate  data from either $\hat \rho_\fwd(1)$ or $\hat \rho_\rev(0)$ assuming that we can sample exactly the PDF  at the other end, i.e. $\rho_0$ and $\rho_1$ respectively,  we would still rely on the equivalent of the forward and backward  SDE in~\eqref{eq:sde:1} and~\eqref{eq:sde:R}, now used with the approximate drifts in~\eqref{eq:hatb:hats}, i.e.

\begin{equation}
    \label{eq:sde:h:1}
    d\hat X^\fwd_t = \hat b_\fwd(t,\hat X^\fwd_t)dt  + \seps  dW_t,
\end{equation}
and 
\begin{equation}
    \label{eq:sde:h:R}
    d\hat X^\rev_t = \hat b_\rev(t,\hat X^\rev_t)dt  + \seps  dW^\rev_t, \quad W_t^\rev = -W_{1-t},
\end{equation}
If we solve \eqref{eq:sde:h:1} forward in time from initial data $\hat X^\fwd_{t=0} \sim \rho_0$, we then have $\hat X^\fwd_{t=1} \sim \hat \rho_\fwd(1)$ where $\hat \rho_\fwd$ is the solution to the forward FPE~\eqref{eq:fpe:f:hat}. Similarly If we solve \eqref{eq:sde:h:R} backward in time from final data $\hat X^\rev_{t=1} \sim \rho_1$, we then have $\hat X^\rev_{t=0} \sim \hat \rho_\rev(0)$ where $\hat \rho_\rev$ is the solution to the backward FPE~\eqref{eq:fpe:r:hat}.

The results of Lemma~\ref{lem:tesol} and Theorem~\ref{prop:sde:rho} can be used to test the quality of  samples generated by either the ODE~\eqref{eq:ode:1} or the forward and backward SDEs~\eqref{eq:sde:1} and~\eqref{eq:sde:R}. In particular, the following two results are direct consequences of Lemma~\ref{lem:tesol} and Theorem~\ref{prop:sde:rho}, respectively:

\begin{restatable}{corollary}{crossentode}
    \label{prop:ce:ode}
    Under the same conditions as Lemma~\ref{lem:tesol}, if $\hat \rho(0) = \rho_0$, the cross-entropy of $\hat \rho(1)$ relative to $\rho_1$ is given by
\begin{equation}
    \label{eq:ce:ode}
    \begin{aligned}
    \cross{\rho_1}{\hat\rho(1)} &= - \int_{\RR^d} \log \hat\rho(1,x) \rho_1(x) dx\\
   & = \EE_1 \int_0^1 \nabla \cdot b_{\ODE}(\tau, X_{1,\tau}(x_1)) d\tau - \EE_1 \log \rho_0( X_{1,0}(x_1))
   \end{aligned}
\end{equation}
where $\EE_1$ denotes an expectation over $x_1\sim\rho_1$. Similarly, if $\hat \rho(1) = \rho_1$, the cross-entropy of $\hat \rho(0)$ relative to $\rho_0$ is given by
\begin{equation}
    \label{eq:ce:ode:r}
    \begin{aligned}
    \cross{\rho_0}{\hat\rho(0)} &= - \int_{\RR^d} \log \hat\rho(0,x) \rho_0(x) dx\\
   & = -\EE_0 \int_0^1 \nabla \cdot b_{\ODE}(\tau, X_{0,\tau}(x_0)) d\tau - \EE_0 \log \rho_1( X_{0,1}(x_0))
   \end{aligned}
\end{equation}
 where $\EE_0$ denotes an expectation over $x_0\sim\rho_0$.
\end{restatable} 

\begin{restatable}{corollary}{crossentsde}
    \label{prop:ce:sde}
    Under the same conditions as Theorem~\ref{prop:sde:rho}, the cross-entropy of $\hat \rho_\fwd(1)$ relative to $\rho_1$ is given by
\begin{equation}
    \label{eq:ce:sde}
    \begin{aligned}
    \cross{\rho_1}{\hat \rho_\fwd(1)} &= - \int_{\RR^d} \log \hat \rho_\fwd(1,x) \rho_1(x) dx\\
   & = -\EE_1 \log \EE_\rev^{x_1} \left(\exp\left(- \int_0^1\nabla \cdot b_\fwd(t, Y^\rev_t) dt\right)  \rho_0(Y_{t=0}^\rev)\right),
   \end{aligned}
\end{equation}
where  $\EE_\rev^{x_1}$ denotes an expectation over $Y^\rev_t$ conditioned on the event $Y^\rev_{t=1}=x_1$, and $\EE_1$ denotes an expectation over $x_1\sim\rho_1$.
Similarly, the cross-entropy of $\hat \rho_\rev(0)$ relative to $\rho_0$ is given by
\begin{equation}
    \label{eq:ce:sde:r}
    \begin{aligned}
    \cross{\rho_0}{\hat \rho_\rev(0)} &= - \int_{\RR^d} \log \hat \rho_\rev(0,x) \rho_0(x) dx\\
   & = -\EE_0 \log \EE_\fwd^{x_0} \left(\exp\left( \int_0^1\nabla \cdot b_\rev(t, Y^\fwd_t)  dt\right)\rho_1(Y^\fwd_{t=1})\right),
   \end{aligned}
\end{equation}
where $\EE_\rev^{x_0}$ denotes an expectation over $Y^\fwd_t$ conditioned on the event $Y^\fwd_{t=0}=x_0$, and $\EE_0$ denotes an expectation over $x_0\sim\rho_0$.
  
\end{restatable}

If in~\eqref{eq:ce:ode}, \eqref{eq:ce:ode:r}, \eqref{eq:ce:sde}, and \eqref{eq:ce:sde:r} we approximate the expectations $\EE_0$ and $\EE_1$ over $\rho_0$ and $\rho_1$ by empirical expectations over  the available data, these equations allow us to cross-validate different approximations of $\hat b$ and $\hat s$, as well as to compare the cross-entropies of densities evolved by the TE~\eqref{eq:TE:hat} with those of the forward and backward FPEs~\eqref{eq:fpe:f:hat} and~\eqref{eq:fpe:r:hat}.

\begin{remark}
    When using  \eqref{eq:ce:sde} and \eqref{eq:ce:sde:r} in practice, taking the $\log$ of the expectations $\EE_\rev^{x_1}$ and $\EE_\fwd^{x_0}$ may create difficulties, such as when using Hutchinson's trace estimator to compute the divergence of $b_\fwd$ or $b_\rev$, which will introduce a bias. One way to remove this bias is to use Jensen's inequality, which leads to the upper bounds
\begin{equation}
    \label{eq:ce:sde:bound}
    \begin{aligned}
    \cross{\rho_1}{\hat \rho_\fwd(1)} \le  \int_0^1 \EE_1 \EE_\rev^{x_1}\nabla \cdot b_\fwd(t, Y^\rev_t) dt - \EE_1 \EE_\rev^{x_1} \log  \rho_0(Y_{t=0}^\rev),
   \end{aligned}
\end{equation}
and
\begin{equation}
    \label{eq:ce:sde:r:bound}
    \begin{aligned}
    \cross{\rho_0}{\hat \rho_\rev(0)} \le -\EE_0 \EE_\fwd^{x_0} \int_0^1\nabla \cdot b_\rev(t, Y^\fwd_t)  dt- \EE_0 \EE_\fwd^{x_0} \log \rho_1(Y^\fwd_{t=1}).
   \end{aligned}
\end{equation}
However, these bounds are not sharp in general -- in fact, using calculations similar to the one presented in the proof of Theorem~\ref{prop:ce:sde}, we can derive exact expressions that capture precisely what is lost when applying Jensen's inequality:
\begin{equation}
    \label{eq:ce:sde:exact}
    \begin{aligned}
    \cross{\rho_1}{\hat \rho_\fwd(1)}=   \int_0^1 \EE_1 \EE_\rev^{x_1}\left(\nabla \cdot b_\fwd(t, Y^\rev_t) -\eps |\nabla \log \hat \rho_\fwd(t,Y_t^\rev)|^2\right)  dt - \EE_1 \EE_\rev^{x_1} \log  \rho_0(Y_{t=0}^\rev),
   \end{aligned}
\end{equation}
and
\begin{equation}
    \label{eq:ce:sde:r:exact}
    \begin{aligned}
    \cross{\rho_0}{\hat \rho_\rev(0)} = -\EE_0 \EE_\fwd^{x_0} \int_0^1\left(\nabla \cdot b_\rev(t, Y^\fwd_t) +\eps |\nabla \log \hat \rho_\rev(t,Y_t^\fwd)|^2\right) dt- \EE_0 \EE_\fwd^{x_0} \log \rho_1(Y^\fwd_{t=1}).
   \end{aligned}
\end{equation}
Unfortunately, since $\nabla \log \hat\rho_\fwd \not= \hat s$ and $\nabla \log \hat\rho_\rev \not= \hat s$ in general due to approximation errors, we do not know how to estimate the extra terms on the right-hand side of~\eqref{eq:ce:sde:exact} and \eqref{eq:ce:sde:r:exact}. One possibility is to use~$\hat s$ as a proxy for $\nabla \log \hat{\rho}_\fwd$ and $\nabla \log \hat\rho_\rev$, which may be useful in practice, but this approximation is uncontrolled in general.
\end{remark}

\section{Instantiations and extensions}
\label{sec:generalization}
In this section, we instantiate the stochastic interpolant framework discussed in Section~\ref{sec:theo}. 

\subsection{Diffusive interpolants}
\label{sec:sisb}
Recently, there has been a surge of interest in the construction of generative models through diffusive bridge processes~\citep{peluchetti2022nondenoising,liu2022let, somnath2023aligned}. In this section, we connect these approaches with our own, highlighting that stochastic interpolants allow us to manipulate certain bridge processes in a simpler and more direct manner. We also show that this perspective leads to a generative process that samples any target density $\rho_1$ by pushing a point mass at any $x_0\in \RR^d$ through an SDE. We begin by introducing a new kind of interpolant:

\begin{definition}[Diffusive interpolant]
    \label{def:diff:interp} 
     Given two probability density functions $\rho_0, \rho_1 : {\RR^d} \rightarrow \RR_{\geq 0}$, a \textit{diffusive interpolant} between $\rho_0$ and $\rho_1$ is a stochastic process $x^\diff_t$ defined as
\begin{equation}
    \label{eq:diffinterp}
    x^\diff_t = I(t,x_0,x_1) + \sqrt{2a(t)} B_t,  \qquad t\in [0, 1],
\end{equation}
where: 
\begin{enumerate}[leftmargin=0.15in]
\item $I(t,x_0,x_1)$ is as in Definition~\ref{def:interp};
\item $(x_0,x_1)\sim \nu$ with $\nu$ satisfying~\eqref{eq:margin:mu} in Definition~\ref{def:interp};
\item $a(t)\in C^2([0,1])$ with $a(0) >0$ and $a(t)\ge 0 $ for all $t\in(0,1]$, and;
\item $B_t$ is a standard Brownian bridge process, independent of $x_0$ and $x_1$. 
\end{enumerate}
\end{definition}

Pathwise, \eqref{eq:diffinterp} is different from the stochastic interpolant introduced in Definition~\ref{def:interp}: in particular, $x^\diff_t$ is continuous but not differentiable  in time. At the same time, since $B_t$ is a Gaussian process with mean zero and variance $\EE B^2_t = t(1-t)$, \eqref{eq:diffinterp} has the same single-time statistics and time-dependent density~$\rho(t,x)$ as the stochastic interpolant~\eqref{eq:stochinterp} if we set $\gamma(t) = \sqrt{2a(t)t(1-t)}$, i.e.
\begin{equation}
\label{eq:bb}
x_t = I(t,x_0,x_1)+ \sqrt{2a(t)t(1-t)} z \quad \text{with} \quad (x_0,x_1)\sim \nu, \  z \sim {\sf N}(0,\Id),\  (x_0,x_1) \perp z.
\end{equation}
As a result,~\eqref{eq:diffinterp} and~\eqref{eq:bb} lead to the \textit{same} generative models.  Technically, it is easier to work with~\eqref{eq:bb} than with~\eqref{eq:diffinterp}, because it avoids the use of It\^o calculus, and enables direct sampling of $x_t$ using samples from $\rho_0$, $\rho_1$, and $\mathsf{N}(0,\Id)$. 
However, \eqref{eq:diffinterp} sheds light on some interesting properties of the generative models based on~\eqref{eq:bb}, i.e. stochastic interpolants with $\gamma(t) = \sqrt{2a(t)t(1-t)}$.
To see why, we now re-derive the transport equation for the density $\rho(t, x)$ shared by~\eqref{eq:diffinterp} and~\eqref{eq:bb} using the relation~\eqref{eq:diffinterp} \new{using Fourier analysis}.
For simplicity, we focus on the case where $a(t)$ is constant in time, i.e. we set $a(t)=a>0$ in~\eqref{eq:diffinterp}. 

To begin, recall that the Brownian Bridge $B_t$  can be expressed in terms of the Wiener process $W_t$ as $B_t = W_t - tW_{t=1}$.
Moreover, it satisfies the SDE obtained by conditioning on $B_{t=1}=0$ via Doob's $h$-transform~\citep{doob1984potential}:
\begin{equation}
    \label{eq:dobb}
    dB_t = - \frac{B_t} {1-t} dt + dW_t, \qquad B_{t=0}=0.
\end{equation}
A direct application of It\^o's formula implies that
\begin{equation}
    \label{eq:dobb:2}
    de^{ik\cdot x_t^\diff} = ik \cdot \Big( \partial_t I(t,x_0,x_1) - \frac{\sqrt{2a}B_t} {1-t} \Big) e^{ik\cdot x_t^\diff}  dt - a|k|^2 e^{ik\cdot x_t^\diff}  dt + \sqrt{2a}ik\cdot dW_t e^{ik\cdot x_t^\diff}.
\end{equation}
Taking the expectation of this expression and using the independence between $(x_0,x_1)$ and $B_t$, we deduce that
\begin{equation}
    \label{eq:dobb:3}
    \partial_t \EE e^{ik\cdot x_t^\diff} = ik \cdot \EE \Big(\Big( \partial_t I(t,x_0,x_1) - \frac{\sqrt{2a}B_t} {1-t} \Big) e^{ik\cdot x_t^\diff} \Big)  - a|k|^2 \EE e^{ik\cdot x_t^\diff}.
\end{equation}
Since for all fixed $t\in [0,1]$ we have $B_t \stackrel{d}{=} \sqrt{t(1-t)}z$ and $x^\diff_t \stackrel{d}{=} x_t $ with $x_t$ defined in~\eqref{eq:bb}, the time derivative \eqref{eq:dobb:3} can also be written as
\begin{equation}
    \label{eq:dobb:4}
    \partial_t \EE e^{ik\cdot x_t} = ik \cdot \EE \Big(\Big(\partial_t I(t,x_0,x_1) - \frac{\sqrt{2at}\, z} {\sqrt{1-t} }\Big) e^{ik\cdot x_t} \Big)  - a|k|^2 \EE e^{ik\cdot x_t}.
\end{equation}
Moreover, since by definition of their probability density we have $\EE e^{ik\cdot x_t^\diff}=\EE e^{ik\cdot x_t} = \int_{\RR^d} e^{ik\cdot x}\rho(t,x)dx$, we can deduce from~\eqref{eq:dobb:4} that $\rho(t)$ satisfies
\begin{equation}
    \label{eq:dobb:5}
    \partial_t \rho+\nabla \cdot(u \rho) = a \Delta \rho,
\end{equation}
where we defined
\begin{equation}
    \label{eq:u:def}
    u(t,x) = \EE\Big( \partial_t I(t,x_0,x_1) - \frac{\sqrt{2at} \, z} {\sqrt{1-t}}\Big| x_t = x\Big).
\end{equation}
For the interpolant $x_t$ in~\eqref{eq:bb},  we have from the definitions of $b$ and $s$ in~\eqref{eq:b:ode:def} and \eqref{eq:s:def} that
\begin{equation}
    \label{eq:sbdoob}
    \begin{aligned} 
    b(t,x) &= \EE\Big( \partial_t I(t,x_0,x_1) + \frac{a(1-2t) z} {\sqrt{2t(1-t)}}\Big| x_t = x\Big),\\
    s(t,x) &= \nabla \log\rho(t,x) = - \frac1{\sqrt{2at(1-t)}}\EE(z| x_t = x),
    \end{aligned}
\end{equation}
As a result, $u-s=b$ and~\eqref{eq:dobb:5} can also be written as the TE~\eqref{eq:transport} using $\Delta \rho = \nabla \cdot(s \rho)$.

\paragraph{Conditional sampling.}
Remarkably, the drift $u$ defined in~\eqref{eq:u:def} remains non-singular for all $t\in [0, 1]$ (including $t=0$) even if $\rho_0$ is replaced by a point mass at $x_0$; by contrast, both $b$ and $s$ are singular at $t=0$ in this case.
Hence, the  SDE associated with the FPE~\eqref{eq:dobb:5} provides us with a generative model that samples $\rho_1$ from a base measure concentrated at a single $x_0$ (i.e. such that the density $\rho_0$ is replaced by a point mass measure at $x=x_0$). 
We formalize this result in the following theorem:

\begin{restatable}{theorem}{diffgen}
    \label{thm:diff}
    Assume that $I(t,x_0,x_1) = x_0$ for $t\in [0,\delta]$ with some $\delta \in (0,1]$.
    Given any  $a>0$, let
\begin{equation}
    \label{eq:u:def:x0}
    u^\diff(t,x,x_0) = \EE_{x_1,z}\Big( \partial_t I(t,x_0,x_1) - \frac{\sqrt{2a t} \, z} {\sqrt{1-t}}\Big| x_t = x\Big),
\end{equation}
where $x_t$ is given by~\eqref{eq:bb} and where $\EE_{x_1,z}(\cdot|x_t=x)$ denotes an expectation over $x_1\sim \rho_1 \perp z\sim \mathsf{N}(0,\Id)$ conditioned on $x_t=x$ with $x_0\in \RR^d$ fixed.
Then $u^\diff(\cdot, \cdot, x_0) \in C^0([0,1];(C^p(\RR^d))^d)$ for any $p\in \NN$ and $x_0 \in \RR^d$.
Moreover, the solutions to the forward SDE
\begin{equation}
    \label{eq:diff:sde}
    dX_t^\diff = u^\diff(t,X_t^\diff, x_0) dt + \sqrt{2a} \, dW_t, \qquad X_{t=0}^\diff = x_0,
\end{equation}
are such that $X^\diff_{t=1} \sim \rho_1$.
\end{restatable}

Note that the additional assumption we make on $I(t,x_0,x_1)$ is consistent with the requirements in Definition~\ref{def:interp} and Assumption~\ref{as:rho:I}: this additional assumption is made for simplicity and can probably be relaxed to $\partial_tI(t=0,x_0,x_1) = 0$.

The proof of Theorem~\ref{thm:diff} is given in Appendix~\ref{app:diff}.
It relies on the calculations that led to \eqref{eq:u:def}, along with the observation that at $t=0$ and $x=x_0$,
\begin{equation}
    \label{eq:u:def:x0:t0}
    u^\diff(t=0,x_0,x_0) = \EE_{x_1}\left(\partial_t I(t=0,x_0,x_1) \right),
\end{equation}
whereas at $t=1$ and any $x\in\RR^d$, we have
\begin{equation}
    \label{eq:u:def:x0:t1}
    u^\diff(t=1,x,x_0) = \partial_t I(t=1,x_0,x) + 2a \nabla \log \rho_1(x),
\end{equation}
which are both well-defined. 
To put the result in Theorem~\eqref{thm:diff} in perspective, observe that no probability flow ODE with $b\in C^0([0,1];(C^p(\RR^d))^d)$ can achieve the same feat as the diffusion in~\eqref{eq:diff:sde}.
This is because the solutions of such an ODE are unique, and therefore can only map $x_0$ onto a single point at time $t=1$. 
%
%
Moreover, $u^\diff(t,x,x_0)$ is the unique minimizer of the objective function
\begin{equation*}
    \label{eq:udiff:obj}
    \mathcal{L}_{u^\diff} [\hat u^{\diff}] = \int_0^1 \EE_{x_1,z} \left( |\hat u^d(t,x_t,x_0)|^2 - 2\Big( \partial_t I(t,x_0,x_1) - \frac{\sqrt{2a t} z} {\sqrt{(1-t)}}\Big)\cdot \hat u^d(t,x_t,x_0)\right) dt.
\end{equation*}


\subsection{One-sided interpolants for Gaussian $\rho_0$}
\label{sec:onesided}
A common choice of base density for generative modeling in the absence of prior information is to choose $\rho_0 = \mathsf{N}(0, \Id)$. 
In this setting, we can group the effect of the latent variable $z$ with $x_0$.
This leads to a simpler type of stochastic interpolant that, in particular, will enables us to instantiate score-based diffusion within our general framework (see Section~\ref{sec:sb}).

\begin{definition}[One-sided stochastic interpolant]
\label{def:interp:os}
 Given a probability density function $\rho_1: {\RR^d} \rightarrow \RR_{\geq 0}$, a \textit{one-sided stochastic interpolant} between  ${\sf N}(0,\Id)$ and $\rho_1$ is a stochastic process $x^\OS_t$
\begin{equation}
    \label{eq:stochinterp:os}
    x^\OS_t = \alpha(t) z+ J(t,x_1),  \qquad t\in [0, 1]
\end{equation}
that fulfills the requirements:
\begin{enumerate}[leftmargin=0.15in]
\item $J\in C^2([0,1],C^2(\RR^d)^d)$ satisfies the boundary conditions $J(0,x_1) = 0$ and $J(1,x_1) = x_1$.
\item $x_1$ and $z$ are  independent random variables drawn  from $\rho_1$ and ${\sf N}(0,\Id)$, respectively.
\item $\alpha: [0,1] \to \RR $ satisfies $\alpha(0)=1$, $\alpha(1) = 0$, $\alpha(t)>0$ for all $t \in [0, 1)$, and $\alpha^2 \in C^2([0,1])$.
\end{enumerate}
\end{definition}

By construction, $x^\OS_{t=0} = z \sim {\sf N}(0,\Id)$ and $x^\OS_{t=1} = x_1 \sim \rho_1$, so that the distribution of the stochastic process~$x^\OS_t$ bridges ${\sf N}(0,\Id)$ and $\rho_1$. It is easy to see that the one-sided stochastic interpolant defined in~\eqref{eq:stochinterp:os} will have the same density as the stochastic interpolant defined in~\eqref{eq:stochinterp} if we set $I(t,x_0,x_1) = J_t(x_1) + \delta(t) x_0$ and take $\delta^2(t)+ \gamma^2(t) = \alpha^2(t)$.  Restricting to this case, our earlier theoretical results apply where the velocity field $b$ defined in~\eqref{eq:b:ode:def} becomes
\begin{equation}
    \label{eq:b:ode:def:os}
    b_\ODE(t,x) = \EE ( \dot \alpha(t) z + \partial_t J(t,x_1)| x^\OS_t = x),
\end{equation}
and the quadratic objective in~\eqref{eq:obj:v} becomes
\begin{equation}
    \label{eq:obj:v:os}
    \mathcal{L}_b[\hat{b}] =\int_0^1   \EE \left( \tfrac12|\hat b(t,x^\OS_t)|^2 - \left(\dot \alpha(t) z + \partial_t J(t,x_1) \right) \cdot \hat b(t,x^\OS_t) \right) dt.
\end{equation}
In the expression above, $x^\OS_t$ is given by~\eqref{eq:stochinterp:os} and the expectation $\EE$ is taken independently over $x_1\sim \rho_1$ and $z\sim {\sf N}(0,\Id)$. Similarly, the score is given by
\begin{equation}
    \label{eq:obj:s:os}
    s(t,x) = - \alpha^{-1}(t) \eta_z(t,x), \qquad \eta_z(t,x) = \EE(z|x^\OS_t=x),
\end{equation}
where $\eta_z(t,x)$ is the equivalent of the denoiser defined in~\eqref{eq:denoiser}. These functions are the unique minimizers of the objectives
\begin{equation}
    \label{eq:os:obj}
    \mathcal{L}_s[\hat{s}] =\int_0^1   \EE \left( \tfrac12|\hat s(t,x^\OS_t)|^2  +\gamma^{-1}(t) z \cdot \hat s(t,x^\OS_t) \right) dt,
\end{equation}
\begin{equation}
    \label{eq:os:obj:etaz}
    \mathcal{L}_{\eta_z}[\hat{\eta}_z] =\int_0^1   \EE \left( \tfrac12|\hat \eta_z(t,x^\OS_t)|^2  - z \cdot \hat \eta_z(t,x^\OS_t) \right) dt.
\end{equation}
Moreover, we can weaken Assumption~\ref{as:rho:I} to the following requirement:

\begin{assumption}
\label{as:rho:J}
The density $\rho_1 \in C^2(\RR^d)$, satisfies $\rho_1(x) > 0$ for all $x \in \RR^d$, and:
\begin{equation}
    \label{eq:rho0:sc}
     \int_{\RR^d} |\nabla \log \rho_1(x)|^2 \rho_1(x) dx < \infty.
\end{equation}
The function $J$ satisfies
\begin{equation}
    \label{eq:bound:dJ}
    \begin{aligned}
        &\exists C_1<\infty  \   : \ 
        &&|\partial_t J(t,x_1)|\le C_1|x_1|
        \quad  &&\text{for all}\quad (t,x_1) \in [0,1]\times \RR^d,
        \end{aligned}
\end{equation}
and
\begin{equation}
    \label{eq:Jt:L2}
    \exists M_1,M_2 < \infty  \ \ : \ \  \EE\big[ |\partial_t J(t,x_1)|^4\big] \le M_1; \quad \EE\big[ |\partial^2_t J(t,x_1)|^2\big] \le M_2, \quad  \text{for all}\quad t\in [0,1],
\end{equation}
where the expectation is taken over $x_1\sim \rho_1$.
\end{assumption}

\begin{remark}
    The construction above can easily be generalized to the case where $\rho_0 = \mathsf{N}(0, C_0)$ with $C_0$ a positive-definite matrix. Without loss of generality, we can then assume that $C_0$ can be represented as $C_0 = \sigma_0 \sigma_0^\T$ where $\sigma_0$ is a lower-triangular matrix and replace \eqref{eq:stochinterp:os}
    \begin{equation}
    \label{eq:stochinterp:os:C}
    x^\OS_t = \alpha(t) \sigma_0 z + J(t,x_1) ,  \qquad t\in [0, 1],
\end{equation}
with $J$ and $\alpha$ satisfying the conditions listed in Definition~\ref{def:interp:os} and where $z\sim {\sf N}(0,\Id)$.
\end{remark}

\subsection{Mirror interpolants}
\label{sec:mirror}
Another practically relevant setting is when the base and the target are the same density $\rho_1$. In this setting we can define a stochastic interpolant as: 

\begin{definition}[Mirror stochastic interpolant]
\label{def:interp:mirr}
 Given a probability density function $\rho_1: {\RR^d} \rightarrow \RR_{\geq 0}$, a \textit{mirror stochastic interpolant} between $\rho_0$ and itself is a stochastic process $x^\MIR_t$
\begin{equation}
    \label{eq:stochinterp:mirror}
    x^\MIR_t = K(t,x_1) + \gamma(t) z,  \qquad t\in [0, 1]
\end{equation}
that fulfills the requirements:
\begin{enumerate}[leftmargin=0.15in]
\item $K\in C^2([0,1],C^2(\RR^d)^d)$ satisfies the boundary conditions $K(0,x_1) = x_1$ and $K(1,x_1) = x_1$.
\item $x_1$ and $z$ are  random variables drawn independently from $\rho_1$ and ${\sf N}(0,\Id)$, respectively.
\item $\gamma: [0,1] \to \RR $ satisfies $\gamma(0)=\gamma(1)=0$, $\gamma(t)>0$ for all $t \in (0, 1)$, and $\gamma^2 \in C^1([0,1])$.
\end{enumerate}
\end{definition}
By construction, $x^\MIR_{t=0} = x^\MIR_{t=1}= x_1 \sim \rho_1$, so that the distribution of the stochastic process~$x^\MIR_t$ bridges $\rho_1$ to itself. 
Note that a valid choice is $K(t,x_1) = \alpha(t) x_1$ with $\alpha(0)=\alpha(1) =1$ (e.g. $\alpha(t) =1$): in this case, mirror interpolants are related to denoisers, as will be discussed in Section~\ref{sec:denoiser}.

It is easy to see that our earlier theoretical results apply where the velocity field $b$ defined in~\eqref{eq:b:ode:def} becomes
\begin{equation}
    \label{eq:b:ode:def:mirro}
    b_\ODE(t,x) = \EE ( \partial_t K(t,x_1) + \dot \gamma(t) z| x^\MIR_t = x),
\end{equation}
and the quadratic objective in~\eqref{eq:obj:v} becomes
\begin{equation}
    \label{eq:obj:v:mirror}
    \mathcal{L}_b[\hat{b}] =\int_0^1   \EE \left( \tfrac12|\hat b(t,x^\MIR_t)|^2 - \left(\partial_t K(t,x_1) + \dot \gamma(t) z\right) \cdot \hat b(t,x^\MIR_t) \right) dt.
\end{equation}
In the expression above, $x^\MIR_t$ is given by~\eqref{eq:stochinterp:mirror} and the expectation $\EE$ is taken independently over $x_1\sim \rho_1$ and $z\sim {\sf N}(0,\Id)$. Similarly, the score is given by
\begin{equation}
    \label{eq:obj:s:mirro}
    s(t,x) = - \gamma^{-1}(t) \eta_z(t,x), \qquad \eta_z(t,x)=\EE(z|x^\MIR_t=x),
\end{equation}
which are the unique minimizers of the objective functions
\begin{equation}
    \label{eq:mirror:obj}
    \mathcal{L}_s[\hat{s}] =\int_0^1   \EE \left( \tfrac12|\hat s(t,x^\MIR_t)|^2  +\gamma^{-1}(t) z \cdot \hat s(t,x^\MIR_t) \right) dt.
\end{equation}
\begin{equation}
    \label{eq:mirror:obj:denoise}
    \mathcal{L}_{\eta_z}[\hat{\eta}_z] =\int_0^1   \EE \left( \tfrac12|\hat \eta_z(t,x^\MIR_t)|^2  - z \cdot \hat \eta_z(t,x^\MIR_t) \right) dt.
\end{equation}
Moreover, we can weaken Assumption~\ref{as:rho:I} to the following requirement:

\begin{assumption}
\label{as:rho:K}
The density $\rho_1 \in C^2(\RR^d)$ satisfies $\rho_1(x) > 0$ for all $x \in \RR^d$ and
\begin{equation}
    \label{eq:rho0:sc:2}
     \int_{\RR^d} |\nabla \log \rho_1(x)|^2 \rho_1(x) dx < \infty.
\end{equation}
The function $K$ satisfies
\begin{equation}
    \label{eq:bound:dK}
    \begin{aligned}
        &\exists C_1<\infty  \   : \ 
        &&|\partial_t K(t,x_1)|\le C_1|x_1|
        \quad  &&\text{for all}\quad (t,x_1) \in [0,1]\times \RR^d,
        \end{aligned}
\end{equation}
and
\begin{equation}
    \label{eq:Kt:L2}
    \exists M_1,M_2 < \infty  \ \ : \ \  \EE\big[ |\partial_t K(t,x_1)|^4\big] \le M_1; \quad \EE\big[ |\partial^2_t K(t,x_1)|^2\big] \le M_2, \quad  \text{for all}\quad t\in [0,1],
\end{equation}
where the expectation is taken over $x_1\sim \rho_1$.
\end{assumption}

\begin{remark}
Interestingly, if we take $K(t,x_1)=x_1$, then $\partial_t K(t,x_1) = 0$, and the velocity field defined in \eqref{eq:b:ode:def:mirro}  is completely defined by the denoiser $\eta_z$
\begin{equation}
    \label{eq:mirror:b:2}
    b_\ODE(t,x) = \dot \gamma(t) \eta_z(t,x)
\end{equation}
Since the score  $s$ also depends on $\eta_z$, this denoiser is the only quantity that needs to be learned. 
\end{remark}

\begin{remark}
If $\rho_1$ is only accessible via empirical samples, mirror interpolants do not enable calculation of the functional form of $\rho_1$.
A notable exception is if we set $K(t,x_1) =0$ for $t\in [t_1,t_2]$ with $0< t_1\le t_2 <1$: in that case, $x^\MIR_t = \gamma(t) z \sim \gamma(t) {\sf N}(0,\Id)$ for $t\in [t_1,t_2]$, which gives us a reference density for comparison. 
In this setup, mirror interpolants essentially reduce to two one-sided interpolants glued together (with the second one time-reversed), or in fact a regular stochastic interpolant when $\rho_0=\rho_1$ and we set $I(t,x_0,x_1) = 0 $ for $t\in [t_1,t_2]$.
\end{remark}

\subsection{Stochastic interpolants and Schr\"odinger bridges}
\label{sec:si:schb}
\label{sec:sb}
The stochastic interpolant framework can also be used to solve the Schr\"odinger bridge problem. For background material on this problem, we refer the reader to~\cite{leonard2014survey} and the references therein. Consistent with the overall viewpoint of this paper, we consider the hydrodynamic formulation of the Schr\"odinger bridge problem, in which the goal is to obtain a pair $(\rho,u)$, that solves the following optimization problem for a fixed $\epsilon > 0$ 
 \begin{equation}
     \label{eq:min:sb}
     \begin{aligned}
         &\min_{\hat{u}, \hat{\rho}} \int_0^1 \int_{\RR^d} |\hat u(t,x)|^2 \hat \rho(t,x) dx dt\\
         \text{subject to:} \quad & \partial_t \hat \rho + \nabla \cdot \left(\hat u \hat\rho\right) = \eps  \Delta \hat \rho, \quad \hat \rho(0) = \rho_0\quad \hat \rho(1) = \rho_1
     \end{aligned}
 \end{equation}
 Under our assumptions on $\rho_0$ and $\rho_1$ listed in Assumption~\ref{as:rho:I}, it is known (see e.g. Proposition 4.1 in~\cite{leonard2014survey}) that~\eqref{eq:min:sb} has a unique minimizer $(\rho,u=\nabla\lambda)$, with $(\rho,\lambda)$ classical solutions of the Euler-Lagrange equations:
 \begin{equation}
    \label{eq:max:min:rho:j:el:2}
    \begin{aligned}
        & \partial_t \rho + \nabla \cdot \left(\nabla \lambda \rho\right) = \eps  \Delta \rho, \quad \rho(0) = \rho_0,\quad \rho(1) = \rho_1,\\
        & \partial_t \lambda +\tfrac12 |\nabla \lambda|^2=- \eps  \Delta \lambda.
    \end{aligned}
\end{equation}
To proceed we will make the additional assumption that the solution $\rho$ to \eqref{eq:max:min:rho:j:el:2} can be reversibly mapped to a standard Gaussian: 
\begin{assumption}
\label{as:sb:2}
There exists a reversible map $T:[0,1]\times \RR^d \to \RR^d$ with $T,T^{-1}\in C^1([0,1], (C^d(\RR^d))^d)$ such that:
\begin{equation}
    \label{eq:interpol}
    \forall t\in [0,1] \quad : \quad z \sim {\sf N}(0,\Id) \ \Rightarrow \ T(t,z) \sim \rho(t); \quad x_t \sim \rho(t) \ \Rightarrow \ T^{-1}(t,x_t) \sim {\sf N}(0,\Id),
\end{equation}
where $\rho$ is the solution to~\eqref{eq:max:min:rho:j:el:2}.
\end{assumption}
We stress that the actual form of the map $T$ is not important for the arguments below.
Assumption~\ref{as:sb:2}  can be used to show the existence of a stochastic interpolant whose density solves~\eqref{eq:max:min:rho:j:el:2}:
\begin{restatable}{lemma}{interppdf}
\label{lem:interp}
If Assumption~\eqref{as:sb:2}  holds, then the solution $\rho(t)$ to~\eqref{eq:max:min:rho:j:el:2} is the density of the stochastic interpolant
\begin{equation}
    \label{eq:stoch:interpolable}
    x_t = T(t, \alpha(t) T^{-1}(0,x_0) + \beta(t) T^{-1}(1,x_1)) + \gamma(t) z,
\end{equation}
as long as $\alpha^2(t)+ \beta^2(t) + \gamma^2(t) =1$.
\end{restatable}
The proof is given in Appendix~\ref{app:sb}: \eqref{eq:stoch:interpolable} corresponds to choosing $I(t,x_0,x_1)= T(t, \alpha(t) T^{-1}(t,x_0) + \beta(t) T^{-1}(t,x_1))$ in~\eqref{eq:stochinterp}. With the help of Lemma~\ref{lem:interp}, we can establish the following result, which shows how to optimize over the function $I$ to solve the problem~\eqref{eq:min:sb}

\begin{restatable}{theorem}{schrob}
    \label{prop:sb}
    Pick some $\gamma:[0,1]\to [0,1)$ such that $\gamma(0)=\gamma(1) = 0$, $\gamma(t)>0$ for $t\in(0,1)$, $\gamma \in C^2((0,1))$ and $\gamma^2\in C^1([0,1])$, and let $\hat x_t = \hat I(t,x_0,x_1)+\gamma(t) z$, with $x_0\sim\rho_0$, $x_1\sim\rho_1$, and $z\sim {\sf N}(0,\Id)$ all independent. Consider the max-min problem over $\hat I\in C^1([0,1],(C^1(\RR^d\times\RR^d))^d)$ and $\hat u\in C^0([0,1],(C^1(\RR^d))^d)$: 
\begin{equation}
    \label{eq:max:min}
    \max_{\hat I} \min_{\hat u} \int_0^1 \EE \left( \tfrac12|\hat u(t,\hat x_t)|^2 - \left(\partial_t \hat I(t,x_0,x_1) + (\dot \gamma(t) -\eps \gamma^{-1}(t)) z \right) 
    \cdot \hat u(t,\hat x_t) \right) dt.
\end{equation}
If Assumption~\ref{as:sb:2} holds, then all the optimizers $(I,u)$ of~\eqref{eq:max:min} are such that the density of the associated $x_t = I(t,x_0,x_1)+\gamma(t) z$ is the solution $\rho$ to~\eqref{eq:max:min:rho:j:el:2}. Moreover, $u = \nabla \lambda$, with $\lambda$ the solution to~\eqref{eq:max:min:rho:j:el:2}.
\end{restatable} 

The proof is also given in Appendix~\ref{app:sb}. Note that if we fix $\hat I$, the velocity $u$ minimizing this objective is the forward drift $b_\fwd$ defined in~\eqref{eq:b:def}. Note also that if we set $\eps\to0$, the minimizing velocity field is~$b$ as defined in~\eqref{eq:b:ode:def}, and the max-min problem formally reduces to solving the optimal transport problem. In this case, Assumption~\ref{as:sb:2} becomes more stringent, as we need to assume that that system~\eqref{eq:max:min:rho:j:el:2} with $\eps=0$ (i.e. in the absence of the diffusive terms) has a classical solution. Theorem~\eqref{prop:sb} gives a practical route towards solving the Schr\"odinger bridge problem with stochastic interpolants, and we leave the numerical investigation of this formulation to future work. 

\section{Spatially linear interpolants}
\label{sec:gen}
In this section, we study the stochastic interpolants that are obtained when we specialize the function $I$ used in~\eqref{eq:stochinterp} to be linear in both $x_0$ and $x_1$, i.e. we consider
\begin{equation}
    \label{eq:lin:interp}
    x^\LIN_t = \alpha(t) x_0+ \beta(t) x_1 + \gamma(t) z,
\end{equation} 
where $(x_0,x_1)\sim \nu$ and $z \sim {\sf N}(0,\Id)$ with $(x_0,x_1)\perp z$, and $\alpha, \beta, \gamma^2\in C^2([0,1])$ satisfy the conditions
\begin{equation}
    \label{eq:lin:interp:a:b}
    \begin{aligned}
         & \alpha(0)=\beta(1)=1; \quad \alpha(1)=\beta(0)=\gamma(0) = \gamma(1) = 0;  \quad \forall t\in(0,1) \ : \ \gamma(t) > 0.
    \end{aligned}
\end{equation}
Despite its simplicity, this setup offers significant design flexibility. The discussion highlights how the presence of the latent variable $\gamma(t)z$ can simplify the structure of the intermediate density $\rho(t)$. 
Since our ultimate aim is to investigate the properties of practical generative models built upon ODEs or SDEs, we will also study the effect of time-dependent diffusion coefficient $\eps(t)$, which controls the amplitude of the noise in a generative SDE. 
Throughout, to build intuition, we choose $\rho_0$ and $\rho_1$ to be Gaussian mixture densities, for which the drift coefficients can be computed analytically (see Appendix~\ref{app:Gauss:mixt}). This enables us to visualize the effect of each choice on the resulting generative models.

\subsection{Factorization of the velocity field}
\label{sec:factor} 

When the stochastic interpolant is of the form~\eqref{eq:lin:interp}, the velocity~$b$ and the score~$s$ defined in~\eqref{eq:b:ode:def} and \eqref{eq:s:def} can both be expressed in terms of the following three conditional expectations (the third is the denoiser already defined in~\eqref{eq:denoiser}):
\begin{equation}
    \label{eq:g:c}
    \eta_0(t,x) = \EE(x_0|x^\LIN_t = x), \quad \eta_1(t,x) = \EE(x_1|x^\LIN_t = x), \quad \eta_z(t,x) = \EE(z|x^\LIN_t = x).
\end{equation}
Specifically, we have
\begin{equation}
    \label{eq:b:s:g}
    b(t,x) = \dot\alpha(t) \eta_0(t,x) +  \dot\beta(t) \eta_1(t,x)+  \dot\gamma(t) \eta_z(t,x), \qquad s(t,x) = - \gamma^{-1}(t) \eta_z(t,x).
\end{equation}
The second relation above holds for $t\in(0,1)$ (i.e. when $\gamma(t) \not=0$).
Moreover, because $\EE (x_t|x_t=x) = x$ by definition, the functions $\eta_0$, $\eta_1$, and $\eta_z$ satisfy
\begin{equation}
    \label{eq:constraint}
    \forall (t,x) \in [0,1]\times \RR^d \quad : \quad \alpha(t) \eta_0(t,x) + \beta(t) \eta_1(t,x) + \gamma(t) \eta_z(t,x) = x.
\end{equation}
This enables us to reduce computational expense: given two of the $\eta$'s, the third can always be calculated via~\eqref{eq:constraint}.
 Finally, it is easy to see that  the functions $\eta_0$, $\eta_1$, and $\eta_z$ are the unique minimizers of the objectives
\begin{equation}
    \label{eq:obj}
    \begin{aligned}
        \mathcal L_{\eta_0}(\hat \eta_0) &= \int_0^1 \EE \left[\tfrac12|\hat \eta_0(t,x^\LIN_t)|^2 - x_0\cdot \hat \eta_0(t,x^\LIN_t)\right] dt,\\
        \mathcal L_{\eta_1}(\hat \eta_1) &= \int_0^1 \EE \left[\tfrac12|\hat \eta_1(t,x^\LIN_t)|^2 - x_1\cdot \hat \eta_1(t,x^\LIN_t)\right] dt,\\
        \mathcal L_{\eta_z}(\hat \eta_z) &= \int_0^1 \EE \left[\tfrac12|\hat \eta_z(t,x^\LIN_t)|^2 - z\cdot \hat \eta_z(t,x^\LIN_t)\right] dt,
    \end{aligned}
\end{equation}
where $x^\LIN_t$ is defined in~\eqref{eq:lin:interp} and the expectation is taken independently over $(x_0,x_1)\sim\nu$ and $z\sim {\sf N}(0,\text{\it Id}).$

\begin{table}[t!]
\centering
\renewcommand{\arraystretch}{1.5}  
\begin{tabular}{@{}ccccc@{}}
\toprule
\multicolumn{1}{c}{Stochastic Interpolant}                    &                & $\alpha(t)$              & $\beta(t)$               & $\gamma(t)$      \\ \midrule

\multicolumn{1}{l}{\multirow{3}{*}{Arbitrary $\rho_0$ (two-sided)}} & linear  & $1-t$                    & $t$                    & $\sqrt{at(1-t)}$ \\ \cmidrule(l){2-5}
                                                                     & trig    & $\cos{\tfrac{\pi}{2} t}$ & $\sin{\tfrac{\pi}{2} t}$ & $\sqrt{at(1-t)}$ \\ \cmidrule(l){2-5}
                                                                     & enc-dec &     $\cos^2(\pi t) 1_{[0,\frac12)}(t) $                    &    $\cos^2(\pi t) 1_{(\frac12,1]}(t)   $                   &   $\sin^2(\pi t)$               \\ 
\midrule
\multicolumn{1}{l}{\multirow{3}{*}{Gaussian $\rho_0$ (one-sided)}}  & linear  & $1-t$                    & $t$                      & 0                \\ \cmidrule(l){2-5}
                                                                     & trig    & $\cos{\tfrac{\pi}{2} t}$ & $\sin{\tfrac{\pi}{2} t}$ & 0\\ \cmidrule(l){2-5}
                                                                     & {\small SBDM (VP)}    & $\sqrt{1-t^2}$ & $t$ & 0                \\ 
\midrule
\multicolumn{1}{c}{Mirror}                    &                      & 0                        & 1                        & $\sqrt{at(1-t)}$ \\ \bottomrule
\end{tabular}
\caption{\textbf{\new{Spatially linear interpolants.}} A table suggesting various linear interpolants. 
In general, this paper describes methods for arbitrary $\rho_0$ and $\rho_1$.
In Section \ref{sec:spatil:lin:os}, we detail linear interpolants for one-sided generation, where $\rho_0$ is a Gaussian and the latent variable $z$ can be absorbed into $x_0$. 
Later, in Section~\ref{sec:SBDM}, we discuss how to recast score-based diffusion models (SBDM) as linear one-sided interpolants, which leads to the expressions given in the table when using a variance-preserving parameterization.
Linear mirror interpolants, where $\rho_0 = \rho_1$ are equal, are defined by~\eqref{eq:stochinterp:mirror} with the choice $K(t,x_1) = \alpha(t)x_1$.}
\end{table}

\subsection{Some specific design choices}
\label{sec:specific:a:b:c}
It is useful to assume that both $\rho_0$ and $\rho_1$ have been scaled to have zero mean and identity covariance (which can be achieved in practice, for example, by an affine transformation of the data). In this case, the time-dependent mean and covariance of~\eqref{eq:lin:interp} are given by
\begin{equation}
\label{eq:mean:var}
\EE [x^\LIN_t] = 0, \qquad \EE [ x^\LIN_t (x^\LIN_t)^\T] = (\alpha^2(t)+\beta^2(t)+\gamma^2(t)) \Id.
\end{equation}
Preserving the identity covariance at all times therefore leads to the constraint
\begin{equation}
\label{eq:sumto1}
\forall t \in [0,1] \quad : \quad \alpha^2(t)+\beta^2(t)+\gamma^2(t) = 1.
\end{equation}
This choice is also sensible if $\rho_0$ and $\rho_1$ have covariances that are not the identity but are on a similar scale. In this case we no longer need to enforce~\eqref{eq:sumto1} exactly, and could, for example, take three functions whose sum of squares is of order one. For definiteness, in the sequel we discuss possible choices that satisfy~\eqref{eq:sumto1} exactly, with the understanding that the corresponding functions $\alpha$, $\beta$, and $\gamma$ could all be slightly modified without significantly affecting the conclusions.

\paragraph{Linear and trigonometric $\alpha$ and $\beta$.} 
One way to ensure that~\eqref{eq:sumto1} holds while maintaining the influence of $\rho_0$ and $\rho_1$ everywhere on $[0,1]$ except at the endpoints is to choose 
\begin{equation}
\label{eq:lin:a:b:c}
\alpha(t) = t, \qquad \beta(t) = 1-t, \qquad \gamma(t) = \sqrt{2t(1-t)}.
\end{equation}
This choice was advocated in~\cite{liu2022}, without the inclusion of the latent variable ($\gamma =0$). Another possibility that gives more leeway is to pick any $\gamma: [0,1]\to[0,1]$ and set
\begin{equation}
\label{eq:trig}
\alpha(t) = \sqrt{1-\gamma^2(t)} \cos(\tfrac 12 \pi t), \qquad \beta(t) = \sqrt{1-\gamma^2(t)} \sin(\tfrac 12 \pi t).
\end{equation}
With $\gamma=0$, this was the choice preferred in~\cite{albergo2023building}. The PDF $\rho(t)$ obtained with the choices~\eqref{eq:lin:a:b:c} and~\eqref{eq:trig} when $\rho_0$ and $\rho_1$ are both Gaussian mixture densities are shown in Figure~\ref{fig:gmm_rhot}. As this example shows, when $\rho_0$ and $\rho_1$ have distinct complex features, these would be duplicated in $\rho(t)$ at intermediate times if not for the smoothing effect of the latent variable; this behavior is seen in Figure~\ref{fig:gmm_rhot}, where it is most prominent in the first row with $\gamma(t) = 0$. From a statistical learning perspective, eliminating the formation of spurious features will simplify the estimation of the velocity field $b$, which becomes smoother as the formation of such features is suppressed.

\label{sec:reg:noise}
\begin{figure}[t!]
    \centering
    \includegraphics[width=\textwidth]{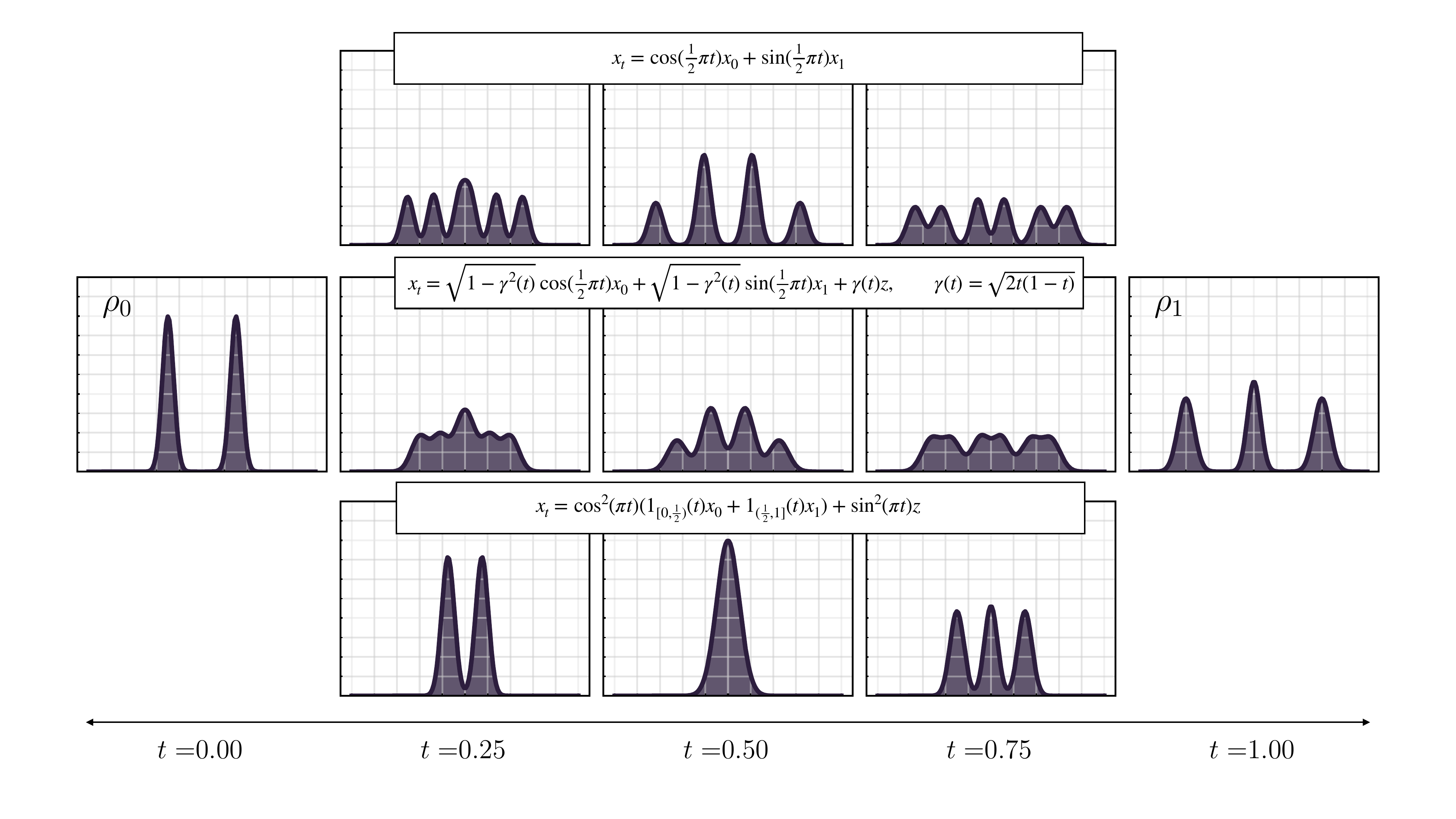}   
    \caption{\textbf{The effect of $\gamma(t)z$ on $\rho(t)$.} A visualization of how the choice of $\gamma(t)$ changes the density $\rho(t)$ of $x^\LIN_t=\alpha(t) x_0 + \beta(t) x_1+ \gamma(t)z $ when $\rho_0$ and $\rho_1$ are Gaussian mixture densities with two modes and three modes, respectively.
    The first row depicts $\gamma(t)=0$, which reduces to the stochastic interpolant developed in~\cite{albergo2023building}. This case forms a valid transport between $\rho_0$ and $\rho_1$, but produces spurious intermediate modes in $\rho(t)$.
    The second row depicts the choice of $\gamma(t) = \sqrt{2t(1-t)}$. In this case, the spurious modes are partially damped by the addition of the latent variable, leading to a simpler $\rho(t)$.
    The final row shows the Gaussian encoding-decoding, which smoothly encodes $\rho_0$ into a standard normal distribution on the time interval $[0, 1/2)$, which is then decoded into $\rho_1$ on the interval $(1/2, 1]$. In this case, no intermediate modes form in $\rho(t)$: the two modes in $\rho_0$ collide to form $\mathsf{N}(0, 1)$ at $t=\tfrac12$, which then spreads into the three modes of $\rho_1$. 
    A visualization of individual sample trajectories from deterministic and stochastic generative models based on ODEs and SDEs whose solutions have density $\rho(t)$ can be seen in Figure~\ref{fig:gmm_trajs}.
    }
    \label{fig:gmm_rhot}
\end{figure}

\begin{figure}[t!]
    \centering
    \includegraphics[width=\textwidth]{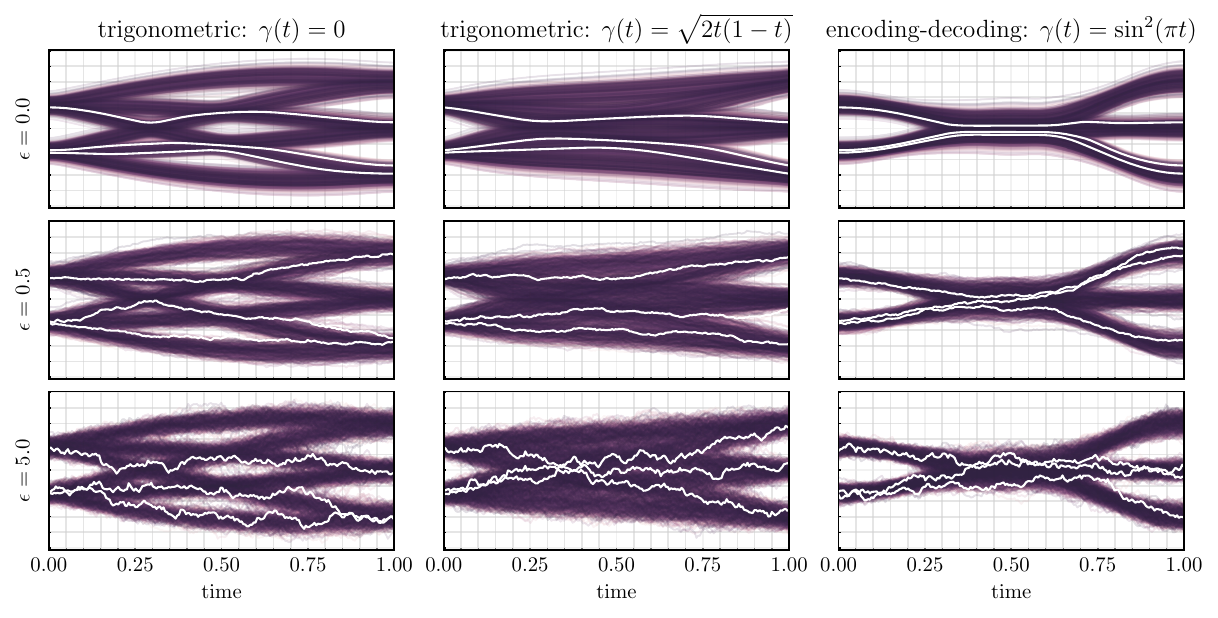}   
    \caption{\textbf{The effect of $\epsilon$ on sample trajectories.} A visualization of how the choice of $\epsilon$ affects the sample trajectories obtained by solving the ODE~\eqref{eq:ode:1} or the forward SDE~\eqref{eq:sde:1}. The set-up is the same as in Figure~\ref{fig:gmm_rhot}:
    $\rho_0$ and $\rho_1$ are taken to be the same Gaussian mixture densities, and the analytical expressions for $b$ and $s$ are used.
    In the three panels in each column the value of $\gamma$ is the same, and each panel shows trajectories with different $\eps$. 
    Three specific trajectories from the same three initial conditions drawn from $\rho_0$ are also highlighted in white in every panel.
    As $\epsilon$ increases but $\gamma$ stays the same, the density $\rho(t)$ is unchanged, but the individual trajectories become increasingly stochastic.
    While all choices are equivalent with exact $b$ and $s$, Theorem~\ref{thm:kl:bound} shows that nonzero values of $\epsilon$ provide control on the likelihood in terms of the error in $b$ and $s$ when they are approximate.}
    \label{fig:gmm_trajs}
\end{figure}

\paragraph{Gaussian encoding-decoding.}
A useful limiting case is to devolve the data from $\rho_0$ completely into noise by the halfway point $t=\tfrac12$ and to reconstruct $\rho_1$ completely from noise starting from $t=\frac12$. One choice that allows us to do so while satisfying~\eqref{eq:sumto1} is
\begin{equation}
    \label{eq:alpha:beta:2}
    \alpha(t) = \cos^2(\pi t) 1_{[0,\frac12)}(t), \qquad \beta(t) = \cos^2(\pi t)1_{(\frac12,1]}(t), \qquad \gamma(t) = \sin^2(\pi t),
\end{equation}
where $1_{A}(t)$ is the indicator function of $A$, i.e. $1_A(t) =1$ if $t\in A$ and $1_A(t)=0$ otherwise.
With this choice, it is easy to see that $x_{t=\frac12} = \gamma(\tfrac{1}{2}) z \sim {\sf N}(0,\gamma^2(\tfrac{1}{2}))$, which seamlessly glues together two interpolants: one between $\rho_0$ and a standard Gaussian, and one between a standard Gaussian and $\rho_1$.

Even though the choice~\eqref{eq:alpha:beta:2} encodes $\rho_0$ into pure noise on the interval $[0,\tfrac12]$, which is then decoded into $\rho_1$ on the interval $[\tfrac12,1]$ (and vice-versa when proceeding backwards in time), the resulting velocity $b$ still defines a single continuity equation that maps $\rho_0$ to $\rho_1$ on $[0, 1]$. 
This is most clearly seen at the level of the probability flow~\eqref{eq:ode:1}, since its solution $X_t$ is a bijection between the initial and final conditions $X_{t=0}$ and $X_{t=1}$, but a similar pairing can also be observed in the solutions to the forward and backward SDEs~\eqref{eq:sde:1} and~\eqref{eq:sde:R}, whose solutions at time $t=1$ or $t=0$ remain correlated with the initial or final condition used. 
\begin{wrapfigure}[10]{r}{0.6\textwidth}
\centering
\vspace{-0.0cm}
  \includegraphics[width=1.0\linewidth]{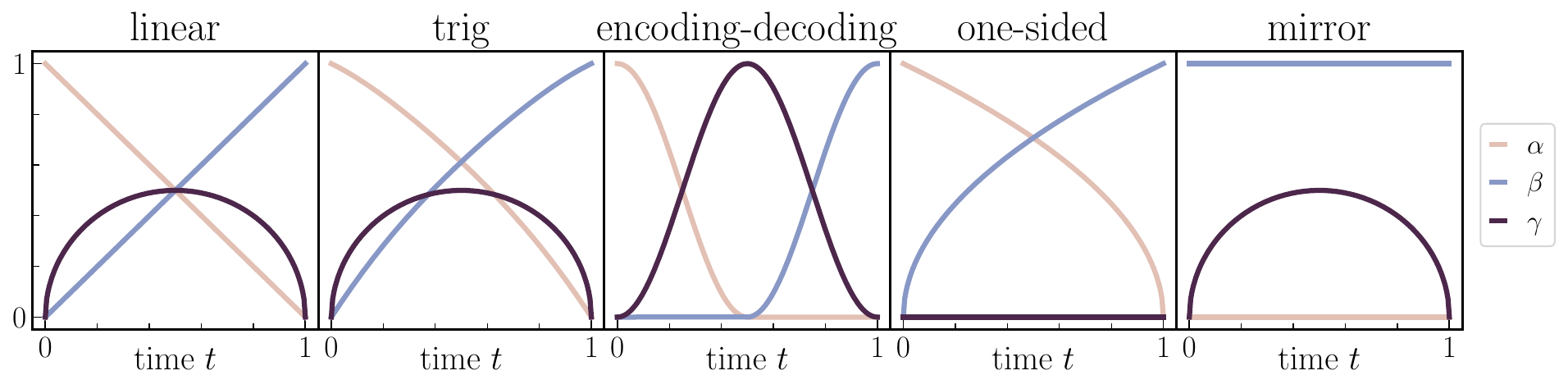}
  \caption{\textbf{The functions $\alpha(t)$, $\beta(t)$, and $\gamma(t)$} for the linear \eqref{eq:lin:a:b:c}, trigonometric~\eqref{eq:trig} with $\gamma(t)=\sqrt{2t(1-t)}$, Gaussian encoding-decoding~\eqref{eq:alpha:beta:2}, one-sided~\eqref{eq:stochinterp:os}, and mirror~\eqref{eq:stochinterp:mirror} interpolants.}
  \label{fig:interps}
\end{wrapfigure}
This allows for a more direct means of image-to-image translation with diffusions when compared to the recent approach described in \cite{su2023dual}. 
The choice~\eqref{eq:alpha:beta:2} is depicted in the final row of Figure~\ref{fig:gmm_rhot}, where no spurious modes form at all; individual sample trajectories of the deterministic and stochastic generative models based on ODEs and SDEs whose solutions have this $\rho(t)$ as density can be seen in the panels forming the third column in Figure~\ref{fig:gmm_trajs}.
We note that the elimination of spurious intermediate modes can also be implemented by use of a data-adapted coupling $\nu(dx_0, dx_1)$, as considered in~\cite{albergo_stochastic_2023}.

Unsurprisingly, it is necessary to have $\gamma(t) > 0$ for the choice~\eqref{eq:alpha:beta:2}: for $\gamma(t) = 0$, the density $\rho(t)$ collapses to a Dirac measure at $t=\frac12$. This consideration highlights that the inclusion of the latent variable $\gamma(t) z$ matters even for the deterministic dynamics~\eqref{eq:ode:1}, and its presence is distinct from the stochasticity inherent to the SDEs~\eqref{eq:sde:1} and \eqref{eq:sde:R}.

\subsection{Impact of the latent variable $\gamma(t) z$ and the diffusion coefficient~$\epsilon(t)$}
\label{sec:impact:gam}

The stochastic interpolant framework enables us to discern the independent roles of the latent variable $\gamma(t) z$ and the diffusion coefficient $\epsilon(t)$ we use in a generative model.
As shown in Theorem~\ref{prop:interpolate}, the presence of the latent variable $\gamma(t) z$ for $\gamma \neq 0$ smooths both the density $\rho(t)$ and the velocity $b$ defined in~\eqref{eq:b:ode:def} spatially.
This provides a computational advantage at sample generation time because it simplifies the required numerical integration of~\eqref{eq:ode:1},~\eqref{eq:sde:1}, and~\eqref{eq:sde:R}. 
Intuitively, this is because the density $\rho(t)$ of $x_t$ can be represented exactly as the density that would be obtained with $\gamma(t) = 0$ convolved with $\mathsf{N}(0, \gamma^2(t)Id)$ at each $t\in(0,1)$. A comparison between the density $\rho(t)$ obtained with trigonometric interpolants with $\gamma(t) = 0$ and $\gamma(t) = \sqrt{2t(1-t)}$ can be seen in the first and second row of Figure~\ref{fig:gmm_rhot}.

By contrast, the diffusion coefficient $\eps(t)$ leaves the density  $\rho(t)$ unchanged, and only affects the way we sample it. 
In particular, the probability flow ODE~\eqref{eq:ode:1} results in a map that pushes every $X_{t=0}=x_0$ onto a single $X_{t=1}=x_1$ and vice-versa. 
The forward SDE~\eqref{eq:sde:1} maps each $X^\fwd_{t=0}=x_0$ onto an ensemble $X^\fwd_{t=1} $ whose spread is controlled by the amplitude of $\epsilon(t)$ (and similarly for the reversed SDE ~\eqref{eq:sde:R} that maps each $X^\rev_{t=1}=x_1$ onto an ensemble $X^\rev_{t=0}$). 
This ensemble is not distributed according to $\rho_1$ for finite $\epsilon(t)$ -- like with the ODE, we need to sample initial conditions from $\rho_0$ to get solutions at time $t=1$ that sample $\rho_1$ -- but its density converges towards $\rho_1$ as $\epsilon(t) \to \infty$. 
These features are illustrated in Figure~\ref{fig:gmm_trajs}. 

\begin{remark}
\label{rem:endpoints}
    Another potential advantage of including the latent variable $\gamma(t) z$ is its impact on the velocity $b$ at the end points. Since $x_{t=0}=x_0$ and $x_{t=1}=x_1$, it is easy to see that the velocity $b$ of the linear interpolant $x_t$ defined in~\eqref{eq:lin:interp} satisfies
\begin{equation}
    \label{eq:vel:0:1}
    \begin{aligned}
    b(0,x) &= \dot\alpha(0) x + \dot \beta(0) \EE [x_1|x_0=x] - \lim_{t\to0} \gamma(t) \dot\gamma(t) s_0(x),\\
    b(1,x) &= \dot\alpha(1) \EE [x_0|x_1=x] + \dot \beta(1) x - \lim_{t\to1} \gamma(t) \dot\gamma(t) s_1(x),
    \end{aligned}
\end{equation}
where  $s_0 = \nabla \log \rho_0$ and $s_1 = \nabla \log \rho_1$. If $\gamma\in C^2([0,1])$, because $\gamma(0)=\gamma(1)=0$, the terms involving the scores $s_0$ and $s_1$ in these expressions vanish. Choosing $\gamma^2 \in C^1([0,1])$ but $\gamma$ not differentiable at $t=0$ or $t=1$ leaves open the possibility that the limits remain nonzero. For example, if we take one of the choices discussed in Section~\ref{sec:sisb}, i.e. 
\begin{equation}
    \label{eq:gam:Bt}
    \gamma(t) = \sqrt{a t(1-t)}, \qquad a>0,
\end{equation}
we obtain
\begin{equation}
    \label{eq:gam:Bt:lim}
    \lim_{t\to0}\gamma(t)\dot{\gamma}(t) = - \lim_{t\to1}\gamma(t)\dot{\gamma}(t) = \frac{a}{2}.
\end{equation}
As a result, the choice~\eqref{eq:gam:Bt} ensures that the velocity $b$ encodes information about the score of the densities $\rho_0$ and $\rho_1$ at the end points. 
We stress however that, while the choice of $\gamma(t)$ given in \eqref{eq:gam:Bt} is appealing because of its nontrivial influence on the velocity $b$ at the endpoints, the user is free to explore a variety of alternatives. 
We present some examples in Table~\ref{tab:gammas}, specifying the differentiability of $\gamma$ at $t=0$ and $t=1$. The function $\gamma(t)$ specified in~\eqref{eq:gam:Bt} is the only featured case for which the contribution from the score is non-vanishing in the velocity $b$ at the endpoints. In Section~\ref{sec:numerics}, we illustrate on numerical examples that there are tradeoffs between different choices of $\gamma$, which might be directly related to this fact. 
When using the ODE as a generative model, the score is only felt through $b$, whereas it is explicit when using the SDE as a generative model.
\end{remark} 

\begin{table}[t!]
    \centering
    \captionsetup{justification=raggedright} 
    \newcolumntype{C}{>{\centering\arraybackslash}p{1.5cm}} 
    \begin{tabular}{ccccc}
    \toprule
       $\gamma(t):$ & $\sqrt{a t(1-t)}$ & $ t(1-t)$ & $\hat \sigma(t)$ & $\sin^2(\pi t)$ \\
    \midrule
    $C^1$ at $t=0,1$ & \ding{55}   & \ding{51}   & \ding{51} &  \ding{51} \\
    \bottomrule
    \end{tabular}
    \caption{\textbf{Differentiability of $\gamma(t)z$.} 
    A characterization of the possible choices of $\gamma(t)$ with respect to their differentiability. The column specified by $\hat \sigma(t)$ is sum of sigmoid functions, made compact by the notation $\hat\sigma(t) = \sigma(f (t - \tfrac{1}{2}) + 1) - \sigma(f(t - \tfrac{1}{2}) - 1) - \sigma(-\tfrac{f}{2}+1) + \sigma(-\tfrac{f}{2}-1) $, where $\sigma(t) = e^t/(1+e^t)$ and $f$ is a scaling factor.}
    \label{tab:gammas}
\end{table}

\subsection{Spatially linear one-sided interpolants}
\label{sec:spatil:lin:os}

Much of the discussion above generalizes to one-sided interpolants if we take the function $J(t,x_1)$ in \eqref{eq:stochinterp:os} to be linear in $x_1$ and define
\begin{equation}
    \label{eq:interp:os:lin}
    x^\OSLIN_t = \alpha(t) z + \beta(t) x_1, \qquad t\in[0,1]
\end{equation}
where $\alpha^2,\beta\in C^2([0,1])$ and $\alpha(0)=\beta(1) = 1$, $\alpha(1) = \beta(0) = 0$, and $\alpha(t) >0$ for all $t\in[0,1)$. 
The velocity~$b$ and the score~$s$ defined in~\eqref{eq:b:ode:def} and \eqref{eq:s:def} can now be expressed as
\begin{equation}
    \label{eq:b:ode:os:lin}
    b_\ODE(t,x) = \dot\alpha(t) \eta^\OS_z(t,x) + \dot \beta(t) \eta^\OS_1(t,x), \qquad s(t,x) = -\alpha^{-1}(t) \eta^\OS_z(t,x),
\end{equation}
where the second expression holds for all $t\in [0,1)$ and we defined:
\begin{equation}
    \label{eq:eta:os}
    \eta^\OS_z(t,x) = \EE(z|x^\OSLIN_t = x), \qquad \eta^\OS_1(t,x) = \EE(x_1|x^\OSLIN_t = x).
\end{equation}
Note that, by definition of the conditional expectation, $\eta^\OS_z$ and $\eta^\OS_1$ satisfy
\begin{equation}
    \label{eq:eta:os:c}
    \forall (t,x) \in [0,1]\times \RR^d \quad : \quad \alpha(t) \eta^\OS_z(t,x) + \beta(t) \eta^\OS_1(t,x)  = x. 
\end{equation}
As a result, only one of them needs to be estimated. 
For example, we can express $\eta^\OS_1$ as a function of $\eta^\OS_z$ for all $t$ such that $\beta(t)\not=0$, and use the result to express the velocity~\eqref{eq:b:ode:os:lin} as
\begin{equation}
    \label{eq:b:os:solved}
    b_\ODE(t,x) = \dot \beta(t) \beta^{-1}(t) x + \big( \dot \alpha(t) - \alpha(t) \dot\beta(t) \beta^{-1}(t)\big) \eta^\OS_z(t,x) \qquad \forall t \ : \ \beta(t)\not=0.
\end{equation}
Assuming that $\beta(t)\not=0$ for all $t\in(0,1]$, this formula only needs to be supplemented at $t=0$ with
\begin{equation}
    \label{eq:b:lin:t1}
    b_\ODE(0,x) = \dot\alpha(0) x + \dot\beta(0) \EE[x_1]
\end{equation}
which follows from~\eqref{eq:b:ode:os:lin} since $x^\OSLIN_{t=0} = z$. 
Later in Section \ref{sec:denoiser} we will show that using the velocity $b$ in \eqref{eq:b:os:solved} to solve the probability flow ODE \eqref{eq:ode:1} can be seen as using a denoiser to construct a generative model. 

Finally note that $\eta_z$ and/or $\eta_1$ can be estimated using the following two objective functions, respectively:
\begin{equation}
    \label{eq:obj:sbdm:eta}
    \begin{aligned}
    \mathcal L_{\eta_z}(\hat \eta^\OS_z) &= \int_0^1 \EE \left[\tfrac12|\hat \eta^\OS_z(t,x^\OSLIN_t)|^2 - z\cdot \hat \eta^\OS_z(t,x^\OSLIN_t)\right] dt,\\
    \mathcal L_{\eta_1}(\hat \eta^\OS_1) &= \int_0^1 \EE \left[\tfrac12|\hat \eta^\OS_1(t,x^\OSLIN_t)|^2 - x_1\cdot \hat \eta^\OS_1(t,x^\OSLIN_t)\right] dt.
    \end{aligned}
\end{equation}

\section{Connections with other methods}
\label{sec:connection}
In this section, we discuss connections between the stochastic interpolant framework and the score-based diffusion method~\citep{song2021scorebased}, the stochastic localization framework~\citep{eldan2013, alaoui2022,montanari2023sampling}), denoising methods~\citep{simoncelli1996, hyvarinen1999, kadkhodaie2021solving, ho2020}, and the rectified flow method~\citep{liu2022}. 

\subsection{Score-based diffusion models and stochastic localization}
\label{sec:SBDM}
Score-based diffusion models (SBDM) are based on variants of the Ornstein-Uhlenbeck process
\begin{equation}
    \label{eq:sbdm:sde}
    d Z_\tau = - Z_\tau dt + \sqrt{2} dW_\tau, \qquad Z_{\tau=0} \sim \rho_1,
\end{equation}
which has the property that the marginal density of its solution at time $\tau$ converges to a standard normal as $\tau$ tends towards infinity. By learning the score of the density of $Z_\tau$, we can write the associated backward SDE for~\eqref{eq:sbdm:sde}, which can then be used as a generative model -- this backwards SDE is also the one that is used in the stochastic localization process, see~\cite{montanari2023sampling}.

To see the connection with stochastic interpolants, notice that the solution of \eqref{eq:sbdm:sde} from the initial condition $Z_{\tau=0} = x_1\sim \rho_1$ can be written exactly as
\begin{equation}
    \label{eq:sbdm:sde:sol}
    Z_\tau = x_1 e^{-\tau}  + \sqrt{2} \int_0^\tau e^{-\tau+s} dW_s.
\end{equation}
As a result, the law of $Z_\tau$ conditioned on $Z_{\tau=0} = x_1$ is given by
\begin{equation}
    \label{eq:sbdm:sde:sol:2}
    Z_\tau \sim {\sf N}(x_1 e^{-\tau}, (1-e^{-2\tau})\Id),
\end{equation}
for any time $\tau\in[0,\infty)$. This is also the law of the process
\begin{equation}
    \label{eq:sbdm:sde:sol:3}
    y_\tau = x_1 e^{-\tau} + \sqrt{1-e^{-2\tau}}\, z, \qquad z\sim {\sf N}(0,\Id), \qquad \tau \in [0, \infty).
\end{equation}
If we let $x_1\sim \rho_1$ with $x_1\perp z$, the process $y_\tau$ is similar to a one-sided stochastic interpolant, except 
the density of $y_\tau$ only converges to ${\sf N}(0,\Id)$ as $\tau\to\infty$; by contrast, the one-sided interpolants we introduced in Section~\ref{sec:onesided} converge on the finite interval $[0, 1]$.
In SBDM, this is handled by capping the evolution of $Z_\tau$ to a finite time interval $[0, T]$ with $T < \infty$, and then by using the backward SDE associated with~\eqref{eq:sbdm:sde} restricted to $[0,T]$. 
However, this introduces a bias that is not present with one-sided stochastic interpolants, because the final condition used for the backwards SDE in SBDM is drawn from ${\sf N}(0,\Id)$ even though the density of the process~\eqref{eq:sbdm:sde} is not Gaussian at time~$T$.

We can, however, turn~\eqref{eq:sbdm:sde:sol:3} into a one-sided linear stochastic interpolant by defining $t=e^{-\tau}$ and by choosing $\alpha(t)$ and $\beta(t)$ in~\eqref{eq:interp:os:lin} to have a specific form. More precisely, evaluating~\eqref{eq:sbdm:sde:sol:3} at $\tau = -\log t$,
\begin{equation}
    \label{eq:link:os:sbdm}
    y_{\tau=-\log t} = \sqrt{1-t^2} z + t x_1 \equiv  x^\OSLIN_t   \quad \text{for} \quad  \alpha(t) = \sqrt{1-t^2}, \quad \beta(t) =t.
\end{equation}
With this choice of $\alpha(t)$ and $\beta(t)$, from~\eqref{eq:b:ode:os:lin} we get the velocity field
\begin{equation}
    \label{eq:b:sbdm}
    b_\ODE(t,x) = -\frac{t}{\sqrt{1-t^2}} \eta^\OS_z(t,x) + \eta^\OS_1(t,x) \equiv t s(t,x) + \eta^\OS_1(t,x)
\end{equation}
where $\eta_z^\OS$ and $\eta_1^\OS$ are defined in~\eqref{eq:eta:os}. 
This expression shows that the velocity $b_\ODE$ used in the probability flow ODE~\eqref{eq:ode:1} is well-behaved at all times, including at $t=1$ where $\dot\alpha(t)$ is singular. 
The same is true for the drift $b_\fwd(t,x) = b_\ODE(t,x) + \eps(t) s(t,x)$ used in the forward SDE~\eqref{eq:sde:1}, regardless of the choice of $\eps \in C^0([0,1])$ with $\eps(t)\ge 0$. 
This shows that casting SBDM into a one-sided linear stochastic interpolant~\eqref{eq:interp:os:lin} allows the construction of \textit{unbiased} generative models that operate on $t\in[0,1]$.
This comes at no extra computational cost, since only one of the two functions defined in~\eqref{eq:eta:os} needs to be estimated, which is akin to estimating the score in SBDM. 

It is worth comparing the above procedure to an equivalent change of time at the level of the diffusion process~\eqref{eq:sbdm:sde}, which we now show leads to singular terms that pose numerical and analytical difficulties.
Indeed, if we define $Z^\rev_t = Z_{\tau=-\log t}$, from~\eqref{eq:sbdm:sde} we obtain
\begin{equation}
\label{eq:sde:sbdm:tchange:r}
d Z^\rev_t = t^{-1} Z^\rev_tdt + \sqrt{2 t^{-1}} dW^\rev_t, \quad Z^\rev_{t=1} \sim \rho_1,
\end{equation}
to be solved backwards in time. 
Because of the factor $t^{-1}$, this SDE cannot easily be solved until $t=0$, which corresponds to $\tau=\infty$ in the original~\eqref{eq:sbdm:sde}. 
For the same reason, the forward SDE associated with~\eqref{eq:sde:sbdm:tchange:r}
\begin{equation}
    \label{eq:sde:sbdm:tchange:f}
    d Z^\fwd_t = t^{-1} Z^\fwd_tdt + 2 t^{-1}s(t,Z^\fwd_t) dt + \sqrt{2 t^{-1}} dW_t, 
\end{equation}
cannot be solved from $t=0$, where formally $Z^\fwd_{t=0} \stackrel{d}{=} Z^\rev_{t=0} = Z_{\tau=\infty} \sim \mathsf N(0,\Id)$.
This means it cannot be used as a generative model unless we start from some $t>0$, which introduces a bias. 
Importantly, this problem does not arise with the stochastic interpolant framework, because the construction of the density $\rho(t)$ connecting $\rho_0$ and $\rho_1$ is handled separately from the construction of the process that generates samples from $\rho(t)$.
By contrast, SBDM combines these two operations into one, leading to the singularity at $t=0$ in the coefficients in~\eqref{eq:sde:sbdm:tchange:r} and \eqref{eq:sde:sbdm:tchange:f}.

\begin{remark}
    To emphasize the last point made above, we stress that there is no contradiction between having  a singular drift and diffusion coefficient in~\eqref{eq:sde:sbdm:tchange:f}, and being able to write a nonsingular SDE with stochastic interpolants. To see why, notice that the stochastic interpolant tells us that we can change the diffusion coefficient in~\eqref{eq:sde:sbdm:tchange:f} to any nonsingular $\eps\in C^0([0,1])$ with $\eps(t)\ge 0$ and replace this SDE with
\begin{equation}
    \label{eq:sde:sbdm:tchange:f:2}
    d X^\fwd_t = t^{-1} Z^\fwd_tdt + ( t^{-1} +\eps(t)) s(t,X^\fwd_t) dt + \sqrt{2 \eps(t)} dW_t, 
\end{equation}
This SDE has the property that $X^\fwd_{t=1}\sim \rho_1$ if $X^\fwd_{t=0}\sim \rho_0$, and its drift is also nonsingular at $t=0$ and given precisely by~\eqref{eq:b:sbdm}. Indeed, using the constraint~\eqref{eq:eta:os:c}, which here reads $x= \sqrt{1-t^2} \eta^\OS_z(t,x) + t \eta^\OS_1(t,x)\equiv -(1-t^2) s(t,x) + t \eta^\OS_1(t,x)$, it is easy to see that
\begin{equation}
    \label{eq:sde:sbdm:tchange:f:3}
    t^{-1} x +  t^{-1} s(t,x) = t s(t,x) + \eta_1^\OS(t,x), 
\end{equation}
which is nonsingular at $t=0$.
\end{remark}

\subsection{Denoising methods}
\label{sec:denoiser}
Consider the spatially-linear one-sided stochastic interpolant defined in~\eqref{eq:interp:os:lin}. By solving this equation for $x_1$, we obtain
\begin{equation}
    \label{eq:x1:xt}
    x_1 = \beta^{-1}(t) \left(x^\OSLIN_t - \alpha(t) z \right )\qquad t \in (0,1].
\end{equation}
Taking a conditional expectation at fixed $x^\OSLIN_t$ and using~\eqref{eq:eta:os} implies that 
\begin{equation}
    \label{eq:identity:0}
    \EE (x_1| x^\OSLIN_t ) = \eta^\OS_1(t,x^\OSLIN_t) = \beta^{-1}(t) \left(x^\OSLIN_t - \alpha(t) \eta^\OS_z(t,x^\OSLIN_t) \right )\qquad t \in (0,1]
\end{equation}
while trivially $\EE (x_1| x^\OSLIN_{t=0} ) = \EE[x_1]$ since $x^\OSLIN_{t=0} = z$. This expression is commonly used in denoising methods~\citep{simoncelli1996, kadkhodaie2021solving}, and it is Stein’s unbiased risk estimator (SURE) for $x_1$ given the noisy information in $x^\OSLIN_t$~\cite{stein1981estimation}. 
Rather than considering the conditional expectation of $x_1$, we can consider an analogous quantity for $x_s^{\OSLIN}$ for any $s \in [0, 1]$; this leads to the following result.
\begin{lemma}[SURE]
    \label{lem:two:t:denoise}
 For  $s\in [0,1]$, we have
 \begin{equation}
    \label{eq:identity:1}
    \EE (x^\OSLIN_s| x^\OSLIN_t) = \frac{\beta(s)}{\beta(t)} x^\OSLIN_t + \left( \alpha(s) - \frac{\alpha(t)\beta(s)}{\beta(t)}\right)  \eta^\OS_z(t,x^\OSLIN_t) \qquad t \in (0,1]
\end{equation}
and $\EE (x^\OSLIN_s| x^\OSLIN_{t=0}) = \alpha(s) x^\OSLIN_{t=0} + \beta(s) \EE[x_1]$.
\end{lemma}

\begin{proof}
    \eqref{eq:identity:1} follows from inserting~\eqref{eq:x1:xt} in the expression for $x^\OSLIN_s$ and taking the conditional expectation using the definition of $\eta_1^\OS$ and $\eta_z^\OS$ in~\eqref{eq:eta:os}. $\EE (x^\OSLIN_s| x^\OSLIN_{t=0}) = \alpha(s) x^\OSLIN_{t=0} + \beta(s) \EE[x_1]$ follows from $x^\OSLIN_{t=0} = z$ together with~\eqref{eq:x1:xt}. 
\end{proof}

At this stage, equations~\eqref{eq:x1:xt} and \eqref{eq:identity:1} cannot be used as generative models: the random variable $\EE (x_1| x^\OSLIN_t ) $ is not a sample of $\rho_1$, and the random variable  $\EE (x^\OSLIN_s| x^\OSLIN_t) $ is not a sample from $\rho(s)$, the density of $x^\OSLIN_s$. 
However, the following result shows that if we iterate upon formula~\eqref{eq:identity:1} by taking infinitesimal steps, we obtain a generative model consistent with the probability flow equation~\eqref{eq:ode:1} associated with~$x^\OSLIN_t$.
\begin{restatable}{theorem}{dn}
    \label{thm:denoise:iter}
    Let $t_j=j/N$ with $j\in\{1,\ldots, N\}$, set $X^\DEN_{1} = z$, and define for $j=1,\ldots, N-1$,
\begin{equation}
    \label{eq:iterate}
    X^\DEN_{{j+1}} = \frac{\beta(t_{j+1})}{\beta(t_j)} X^\DEN_{j}+ \left( \alpha(t_{j+1}) - \frac{\alpha(t_j)\beta(t_{j+1})}{\beta(t_j)}\right)  \eta^\OS_z(t_{j},X^\DEN_{j}) .
\end{equation}
Then,~\eqref{eq:iterate} is a consistent integration scheme for the probability flow equation~\eqref{eq:ode:1} associated with the velocity field~\eqref{eq:b:ode:os:lin} expressed as in~\eqref{eq:b:os:solved}. 
That is, if $N,j\to\infty$ with $j/N \to t\in[0,1]$, then $X^\DEN_j \to X_t$ where
\begin{equation}
    \label{eq:iterate:lim}
    \dot X_t  =b(t,X_t) =  \frac{\dot \beta(t)}{\beta(t)} X_t + \left( \dot \alpha(t) - \frac{\alpha(t) \dot\beta(t)}{\beta(t)}\right) \eta^\OS_z(t,X_t), \quad X_{t=0} = z.
\end{equation}
In particular, if $z\sim {\sf N}(0,\Id)$, then $X^\DEN_N \to x_1 \sim \rho_1$ in this limit.
\end{restatable}

The proof of this theorem is given in Appendix~\ref{app:denoise}, and proceeds by Taylor expansion of the right-hand side of \eqref{eq:iterate}.

\subsection{Rectified flows}
\label{sec:rect}
We now discuss how stochastic interpolants can be \textit{rectified} according to the procedure proposed in~\cite{liu2022}.
Suppose that we have perfectly learned the velocity field $b$ in the probability flow equation~\eqref{eq:ode:1} for a given stochastic interpolant. 
Denote by $X_t(x)$ the solution to this ODE with the initial condition $X_{t=0}(x)=x$, i.e.
\begin{equation}
    \label{eq:ode:1:os:x0}
    \frac{d}{dt}  X_t(x) = b_\ODE(t, X_t(x)), \qquad X_{t=0}(x)=x.
\end{equation}
We can use the map $X_{t=1}: \RR^d \to \RR^d$ to define a new stochastic interpolant 
\begin{equation}
    \label{eq:new:os}
    x^\REC_t = \alpha(t) \new{x_0} + \beta(t) X_{t=1}(\new{x_0}),
\end{equation}
where $\alpha^2,\beta\in C^2([0,1])$ satisfy $\alpha(0)=\beta(1) = 1$, $\alpha(1) = \beta(0)w= 0$, and $\alpha(t) >0$ for all $t\in[0,1)$.
Clearly, we then have $x^\REC_{t=0} = \new{x_0}\sim \rho_0$ since $X_{t=0}(\new{x_0})=\new{x_0}$ and $x^\REC_{t=1} = X_{t=1}(\new{x_0}) \sim \rho_1$ by definition of the probability flow equation. 
We can define a new probability flow equation associated with the velocity field
\begin{equation}
    \label{eq:b:rec}
    b^\REC(t,x) =  \EE[ \dot x^\REC_t | x^\REC_t =x] = \dot\alpha(t) \EE[\new{x_0} | x^\REC_t =x] + \dot\beta(t) \EE[X_{t=1}(\new{x_0})| x^\REC_t =x].
\end{equation}
It is easy to see that this velocity field is amenable to estimation, since it is the unique minimizer of 
\begin{equation}
    \label{eq:obj:brec}
    \mathcal L_{b^\REC}[\hat b^\REC] = \int_0^1 \EE \big[ \tfrac12 |\hat b^\REC(t,x^\REC_t) |^2 - (\dot\alpha(t) \new{x_0} + \dot\beta(t) X_{t=1}(\new{x_0})) \cdot \hat b^\REC(t,x^\REC_t)\big] dt,
\end{equation}
where $x^\REC_t$ is given in~\eqref{eq:new:os} and the expectation is now only on $\new{x_0 \sim \rho_0}$. Our next result show that the probability flow equation associated with the velocity field \eqref{eq:b:rec} has straight line solutions, but ultimately it leads to a generative model that is identical to the one based on~\eqref{eq:ode:1:os:x0}. 
To phrase this result, we first make an assumption on the invertibility of $x_t^{\REC}$.
\begin{assumption}
    \label{as:rec}
    The map $x\to M(t,x) $ where $M(t,x) = \alpha(t) x + \beta(t) X_{t=1}(x)$ with $X_t(x)$ solution to \eqref{eq:ode:1:os:x0} is invertible for all $(t,x)\in[0,1]\times\RR^d$, i.e. $\exists N(t,\cdot,): \RR^d \to \RR^d $ such that
    \begin{equation}
    \label{eq:q:def}
    \forall (t,x)\in[0,1]\times\RR^d \quad : \quad N(t,M(t,x))=  M(t,N(t,x)) = x.
    \end{equation}
\end{assumption}
This is equivalent to requiring that the determinant of the Jacobian of $M(t,x)$ is nonzero for all $(t,x)\in[0,1]\times\RR^d$; put differently, Assumption~\ref{as:rec} requires that the Jacobian of $X_{t=1}(x)$ never has eigenvalues precisely equal to $-\alpha(t)/\beta(t)$, which is generic.
Under this assumption, we state the following theorem.
\begin{restatable}{theorem}{rec}
    \label{thm:cond}
    Consider the probability flow equation associated with~\eqref{eq:b:rec}, 
    \begin{equation}
        \label{eq:prob:flow:rec}
        \frac{d}{dt}  X^\REC_t(x) = b^\REC_\ODE(t, X^\REC_t(x)), \qquad X^\REC_{t=0}(x)=x.
    \end{equation}
    Then, all solutions are such that $X^\REC_{t=1}(\new{x_0})\sim \rho_1$ if $\new{x_0 \sim \rho_0}$. In addition, 
    if Assumption~\eqref{as:rec} holds, the velocity field defined in~\eqref{eq:b:rec} reduces to
    \begin{equation}
        \label{eq:b:explicit}
        b^\REC(t,x) = \dot \alpha(t) N(t,x) + \dot\beta(t) X_{t=1}(N(t,x))
    \end{equation}
    and the solution to the probability flow ODE~\eqref{eq:prob:flow:rec} is simply
    \begin{equation}
        \label{eq:prod:flow:sol}
        X^\REC_t(x)= \alpha(t)x + \beta(t)X_{t=1}(x).
    \end{equation}
\end{restatable}

The proof is given in Appendix~\ref{app:rect}. 

Theorem~\ref{thm:cond} implies that  $X^\REC_t(x)$ is a simpler flow than $X_t(x)$, but we stress that they give the \textit{same} map, $X^\REC_{t=1}=X_{t=1}$.   In particular, $X^\REC_t(x)$ reduces to a straight line between $x$ and $X_{t=1}(x)$ for $\alpha(t)=1-t$ and $\beta(t)=t$.
We also note that  the approach can be used to learn a single-step map, since~\eqref{eq:b:explicit} and $N(t=0,x) = x$ give
\begin{equation}
        \label{eq:b:explicit:t0}
        b^\REC(t=0,x) = \dot \alpha(0) x + \dot\beta(0) X_{t=1}(x),
\end{equation}
which expresses $X_{t=1}(x)$ in terms of known quantities as long as $\dot \beta(0) \not = 0$.
For example, if $\dot\alpha(0)=0$ and $\dot \beta(0) = 1$, we obtain $b^\REC(t=0,x) = X_{t=1}(x)$.

\begin{remark}[Optimal transport]
    The discussion above highlights the fact that a probability flow equation can have straight line solutions and lead to a map that exactly pushes $\rho_0$ onto $\rho_1$ but is not the optimal transport map.
    That is,  straight line solutions is a necessary condition for optimal transport, but it is not sufficient.
\end{remark}

\begin{remark}[Gradient fields]
    The map is unaffected by the rectification procedure because we do not impose that the velocity $b^\REC(t,x)$ be a gradient field. 
    If we do impose this structure by setting $b^\REC(t,x) = \nabla \phi(t,x)$ for some $\phi: \RR^d \to \RR$, then $X_{t=1}^\REC \not = X_{t=1}$. 
    As shown in~\cite{liu2022-ot}, iterating over this procedure eventually gives the optimal transport map. 
    That is, implemented over gradient fields and iterated infinitely, rectification computes Brenier's polar decomposition of the map~\citep{brenier1987polar}. 
\end{remark}

\begin{remark}[Consistency models]
    Recent work has introduced the notion of \textit{consistency models}~\citep{song_consistency_2023}, which distill a velocity field learned via score-based diffusion into a single-step map.
    Section~\ref{sec:SBDM} and the previous discussion provide an alternative perspective on consistency models, and show how they may be computed in the framework of stochastic interpolants via rectification.
\end{remark}

\section{Algorithmic aspects}
\label{sec:practical}
The methods described in the previous sections have efficient numerical realizations.
Here, we detail algorithms and practical recommendations for an implementation.
These suggestions can be split into two complementary tasks: learning the drift coefficients, and sampling with an ODE or an SDE.

\subsection{Learning}
\label{sec:learning}

As described in Section \ref{sec:cont:eq}, there are a variety of algorithmic choices that can be made when learning the drift coefficients in \eqref{eq:transport}, \eqref{eq:fpe}, and  \eqref{eq:fpe:tr}. While all choices lead to exact generative models in the absence of numerical and statistical errors, in practice, the presence of these errors ensures that different choices lead to different generative models, some of which may perform better for specific applications. Here, we describe the various realizations explicitly.

\paragraph{Deterministic generative modeling: Learning $b$ versus learning $v$ and $s$.} Recall from Section~\ref{sec:cont:eq} that the drift $b$ of the transport equation~\eqref{eq:transport} can be written as $b(t, x) = v(t, x) - \gamma(t)\dot{\gamma}(t)s(t, x)$. This raises the practical question of whether it would be better to learn an estimate $\hat{b}$ of $b$ by minimizing the empirical risk
\begin{equation}
    \label{eqn:b_empirical}
    \hat{\mathcal{L}}_b(\hat{b}) = \frac{1}{N}\sum_{i=1}^N\left(\frac{1}{2}|\hat{b}(t_i, x_{t_i}^{i})|^2 - \hat{b}(t_i, x_{t_i}^{i})\cdot \left(\partial_t I(t_i,x_0^i, x_1^i) + \dot{\gamma}(t_i) z^{i}\right)\right),
\end{equation}
or to learn estimates of $\hat{v}$ and $\hat{s}$ by minimizing the empirical risks
\begin{equation}
    \label{eqn:v_empirical}
   \hat{\mathcal{L}}_v(\hat{v}) = \frac{1}{N}\sum_{i=1}^N\left(\frac{1}{2}|\hat{v}(t_i, x_{t_i}^{i})|^2 - \hat{v}(t_i, x_{t_i}^{i})\cdot \partial_t I(t_i,x_0^i, x_1^i)\right)
\end{equation}
and
\begin{equation}
    \label{eqn:s_empirical}
    \hat{\mathcal{L}}_s(\hat{s}) = \frac{1}{N}\sum_{i=1}^N\left(\frac{1}{2}|\hat{s}(t_i, x_{t_i}^{i})|^2 + \gamma(t_i)^{-1} \hat{s}(t_i, x_{t_i}^{i})\cdot z^{i}\right)
\end{equation}
\begin{algorithm}[t!]
    \caption{Learning $b$ with arbitrary $\rho_0$ and $\rho_1$.}
    \DontPrintSemicolon
    \SetKwRepeat{Repeat}{repeat}{until}
    \SetKwBlock{Init}{Initialize}{}
    \SetKwBlock{Update}{Update Parameters}{}
    \textbf{Input:} Batch size $N$, interpolant function $I(t,x_0, x_1)$, coupling $\nu(dx_0, dx_1)$ to sample $x_0, x_1$, noise function $\gamma(t)$, initial parameters $\theta_b$, gradient-based optimization algorithm to minimize $\mathcal L_b $ defined in \eqref{eq:obj:v}, number of gradient steps $N_g$.\\
    \textbf{Returns}: An estimate $\hat b$ of $b$.\\
    \setstretch{1.3}
    \For{$j=1, \hdots, N_g$}{
        Draw $N$ samples $(t_i, x_0^i, x_1^i, z^i) \sim \mathsf{Unif}([0, 1])\times\nu \times\mathsf{N}(0, 1)$, for $i=1,\hdots, N$.\\
        Construct samples $x_t^i = I(t_i,x_0^i, x_1^i) + \gamma(t_i)z^i$, for $i = 1, \hdots, N$.\\
        Take gradient step with respect to $\mathcal L_b\left(\theta_b, \{x_t^i\}_{i=1}^N \right)$. \\
    }
    \textbf{Return:} $\hat b$.
    \label{alg:learning:b}
\end{algorithm}
and construct the estimate $\hat{b}(t, x) = \hat{v}(t, x) - \gamma(t)\dot{\gamma}(t)\hat{s}(t, x)$. Above, $x_{t_i}^{i} = I(t_i, x_0^i, x_1^i) + \gamma(t_i) z^{i}$, and $N$~denotes the number of samples $t^i$, $x_0^i$, and $x_1^i$. If the practitioner is interested in a deterministic generative model (for example, to exploit adaptive integration or exact likelihood computation), learning the estimate $\hat{b}$ directly only requires learning a single model, and hence will typically lead to greater efficiency. This recommendation is captured in Algorithm \ref{alg:learning:b}. If both stochastic and deterministic generative models are of interest, it is necessary to learn two models for most choices of the interpolant; we discuss more suggestions for stochastic case below. 

\paragraph{Antithetic sampling and capping.} In practice, the losses for $b$~\eqref{eq:obj:v} and $s$~\eqref{eq:obj:s} can become high-variance near the endpoints $t=0$ and $t=1$ due to the presence of the singular term $1/\gamma(t)$ and the (potentially) singular term $\dot{\gamma}(t)$. This issue can be eliminated by using antithetic sampling, which we found necessary for stable training of objectives involving $\gamma^{-1}(t)$. 
To show why, we consider the loss~\eqref{eq:obj:s} for $s$, but an analogous calculation can be performed for the loss~\eqref{eq:obj:v} for $b$ or \eqref{eq:obj:eta:0} for $\eta_z$ (even though it is not necessary for this last quantity). 
We first observe that, by definition of $x_t$ and by Taylor expansion, as $t\to0$ or $t\to1$
\begin{equation}
\label{eqn:antithetic_div}
\begin{aligned}
    \frac{1}{\gamma(t)}z\cdot s(t, x_t) &= \frac{1}{\gamma(t)}z\cdot s(t, I(t,x_0, x_1) + \gamma(t) z),\\
    &= \frac{1}{\gamma(t)}z\cdot \left(s(t, I(t,x_0, x_1)) + \gamma(t)\nabla s(t, I(t,x_0, x_1))z + o(\gamma(t))\right),\\
    &= \frac{1}{\gamma(t)}z\cdot s(t, I(t,x_0, x_1)) + z\cdot \nabla s(t, I(t,x_0, x_1))z + o(1).
\end{aligned}
\end{equation}
Even though the conditional mean of the first term at the right-hand side is finite in the limit as $t\to0$ or $t\to1$, its variance diverges. By contrast, let $x_t^+ = I(t,x_0, x_1) + \gamma(t) z$ and $x_t^- = I(t,x_0, x_1) - \gamma(t) z$ with $x_0, x_1$, and $z$ fixed. Then, 
\begin{equation}
\label{eqn:antithetic_conv}
\begin{aligned}
    &\frac{1}{2\gamma(t)}\left(z\cdot s(t, x_t^+) - z\cdot s(t, x_t^-)\right) \\
    &= \frac{1}{2\gamma(t)}\left(z\cdot s(t, I(t,x_0, x_1) + \gamma(t) z)) - z\cdot s(t, I(t,x_0, x_1) - \gamma(t) z)\right),\\
    &= \frac{1}{2\gamma(t)}z\cdot \left(s(t, I(t,x_0, x_1)) + \gamma(t)\nabla s(t, I(t,x_0, x_1))z + o(\gamma(t))\right)\\
    &\qquad -\frac{1}{2\gamma(t)}z\cdot \left(s(t, I(t,x_0, x_1)) - \gamma(t)\nabla s(t, I(t,x_0, x_1))z + o(\gamma(t))\right),\\
    &= z\cdot \nabla s(t, I(t,x_0, x_1))z + o(1),
\end{aligned}
\end{equation}
so that both the conditional mean and variance are finite in the limit as $t\to0$ or $t\to1$ despite the singularity of $1/\gamma(t)$. 
In practice, this can be implemented by using $x_t^+$ and $x_t^-$ for every draw of $x_0, x_1$, and $z$ in the empirical discretization of the population loss.

\begin{algorithm}[t!]
    \caption{Learning $\eta_{z}$  with arbitrary $\rho_0$ and $\rho_1$.}
    \DontPrintSemicolon
    \SetKwRepeat{Repeat}{repeat}{until}
    \SetKwBlock{Init}{Initialize}{}
    \SetKwBlock{Update}{Update Parameters}{}
    \textbf{Input:} Batch size $N$, interpolant function $I(t,x_0, x_1)$, a coupling $\nu(dx_0, dx_1)$ to sample $x_0, x_1$, noise function $\gamma(t)$, initial parameters $\theta_{\eta_z}$, gradient-based optimization algorithm to minimize $\mathcal L_{\eta_z} $ defined in \eqref{eq:obj:eta:0}, number of gradient steps $N_g$.\\
    \textbf{Returns}: An estimate $\hat \eta_z$ of $\eta_z$.\\ 
    \setstretch{1.3}
    \For{$j = 1, \hdots, N_g$}{
        Draw $N$ samples $(t_i, x_0^i, x_1^i, z^i) \sim \mathsf{Unif}([0, 1])\times\nu \times\mathsf{N}(0, 1)$, for $i=1,\hdots, N$.\\
        Construct samples $x_t^i = I(t_i,x_0^i, x_1^i) + \gamma(t_i)z^i$, for $i = 1, \hdots, N$.\\
        Take gradient step with respect to $\mathcal L_{\eta_z}\left(\theta_{\eta_z}, \{x_t^i\}_{i=1}^N \right)$. \\
    }
    \textbf{Return:} $\hat \eta_z$.
    \label{alg:learning:eta:sde}
\end{algorithm}

\paragraph{Learning the score $s$ versus learning a denoiser $\eta_z$.}
When learning $s$, an alternative to antithetic sampling is to consider learning the denoiser $\eta_z$ defined in \eqref{eq:denoiser}, which is related to the score by a factor of $\gamma$. Note that the objective function for the denoiser in \eqref{eq:obj:eta:0}  is well behaved for all $t \in [0,1]$, and can be thought of as a generalization of the DDPM loss introduced in \cite{ho2020}. The empirical risk associated with this loss reads
\begin{equation}
    \label{eqn:etaz_empirical}
    \hat{\mathcal{L}}_{\eta_z}(\hat{\eta}_z) = \frac{1}{N}\sum_{i=1}^N\left(\frac{1}{2}|\hat{\eta}_z(t_i, x_{t_i}^{i})|^2 - \hat{\eta}_z(t_i, x_{t_i}^{i})\cdot z^{i}\right)
\end{equation}
A detailed procedure for learning the denoiser $\eta_z$, e.g. for its use in an SDE-based generative model is given in Algorithm \ref{alg:learning:eta:sde}.
For the case of one-sided spatially-linear interpolants, the procedure becomes particularly simple, which is highlighted in Algorithm~\ref{alg:learning:eta:os}.

 \begin{algorithm}[t!]
    \caption{Learning $\eta_{z}^\OS$ with Gaussian $\rho_0$.}
    \DontPrintSemicolon
    \SetKwRepeat{Repeat}{repeat}{until}
    \SetKwBlock{Init}{Initialize}{}
    \SetKwBlock{Update}{Update Parameters}{}
    \textbf{Input:} Batch size $N$, interpolant function $x^\OSLIN_t $, a coupling $\nu(dx_0, dx_1)$ to sample $z, x_1$, initial parameters $\theta_{\eta_z^\OS}$, gradient-based optimization algorithm to minimize $\mathcal L_{\eta_z^\OS}$ defined in \eqref{eq:obj:eta:0}, number of gradient steps $N_g$.\\
    \textbf{Returns}: An estimate $\hat \eta_z^\OS$ of $\eta_z^\OS$.\\ 
    \setstretch{1.3}
    \For{$j = 1, \hdots, N_g$}{
        Draw $N$ samples $(t_i, z^i, x_1^i) \sim \mathsf{Unif}([0, 1])\times\nu$, for $i=1,\hdots, N$.\\
        Construct samples $x_t^i = \alpha(t_i)z + \beta(t_i)x_1^i$, for $i = 1, \hdots, N$.\\
        Take gradient step w.r.t $\mathcal L_{\eta_z^\OS}\left(\theta_{\eta_z^\OS}, \{x_t^i\}_{i=1}^N \right)$. \\
    }
    \textbf{Return:} $\hat \eta_z^\OS$.
    \label{alg:learning:eta:os}
\end{algorithm}

\subsection{Sampling}
\label{sec:sampling}
    We now discuss several practical aspects of sampling generative models based on stochastic interpolants.
    These are intimately related to the choice of objects that are learned, as well as to the specific interpolant used to build a path between $\rho_0$ and $\rho_1$.
    A general algorithm for sampling models built on either ordinary or stochastic differential equations is presented in Algorithm~\ref{alg:sampling}.

\begin{algorithm}[t!]
    \caption{Sampling general stochastic interpolants.}
    \DontPrintSemicolon
    \SetKwBlock{Init}{Initialize}{}
    \SetKwRepeat{Repeat}{repeat}{until}
    \textbf{Input:} Number of samples $n$, timestep $\Delta t$, drift estimates $\hat{b}$ and $\hat{\eta}_z$, initial time $t_0$, final time $t_f$, noise function $\gamma(t)$, diffusion coefficient $\epsilon(t)$, SDE or ODE timestepper \texttt{TakeStep}.\\
    \textbf{Returns}: $\{\hat{x}_1^{(i)}\}_{i=1}^n$, a batch of model samples.\\ 
    \setstretch{1.0}
    \Init{
    Set time $t = t_0$.\\
    Draw initial conditions $\hat{x}^{(i)}_{t_0} \sim \rho_0$ for $i = 1, \hdots, n$.\\ 
    Construct $\hat{s}(t, x) = -\hat{\eta}_z(t, x) / \gamma(t)$.\\
    Construct $\hat b_{\fwd}(t, x) = \hat b(t, x) + \epsilon(t) \hat s(t, x)$. \tcp{Reduces to $\hat b$ for $\epsilon(t) = 0$ (ODE).}
    }
    \While{$t < t_f$}{
        Propagate $\hat{x}^{(i)}_{t+\Delta t} = \texttt{TakeStep}(t, \hat{x}^{(i)}_t, b_\fwd, \epsilon, \Delta t)$ for $i = 1, \hdots, n$. \tcp{ODE or SDE integrator.}
        Update $t = t + \Delta t$.
    }
    \textbf{Return}: $\{\hat{x}^{(i)}\}_{i=1}^n$.
    \label{alg:sampling}
\end{algorithm}

\paragraph{Using the denoiser $\eta_z$ instead of the score $s$.}

We remarked in Section~\ref{sec:learning} that learning the denoiser $\eta_z$ is more numerically stable than learning the score $s$ directly. 
We note that while the objective for $\eta_z$ is well-behaved for all $t \in [0, 1]$, the resulting drifts can become singular at $t=0$ and $t=1$ when using $s(t, x) = -\eta_z(t, x) / \gamma(t)$.
There are several ways to avoid this singularity in practice.
One method is to choose a time-varying $\epsilon(t)$ that vanishes in a small interval around the endpoints $t=0$ and $t=1$, which avoids this numerical instability.
An alternative option is to integrate the SDE up to a final time $t_f$ with $t_f < 1$, and then to perform a step of denoising using \eqref{eq:identity:1}. 
We use this approach in Section~\ref{sec:numerics} below when sampling the SDE.

\paragraph{A denoiser is all you need for spatially-linear one-sided interpolants.}
\begin{algorithm}[t!]
    \caption{Sampling spatially-linear one-sided interpolants with Gaussian $\rho_0$.}
    \DontPrintSemicolon
    \SetKwBlock{Init}{Initialize}{}
    \SetKwRepeat{Repeat}{repeat}{until}
    \textbf{Input:} Number of samples $n$, timestep $\Delta t$, denoiser estimate $\hat{\eta}_z$, initial time $t_0$, final time $t_f$, noise function $\gamma(t)$, diffusion coefficient $\epsilon(t)$, interpolant functions $\alpha(t)$ and $\beta(t)$, SDE or ODE timestepper \texttt{TakeStep}.\\
    \textbf{Returns}: $\{\hat{x}_1^{(i)}\}_{i=1}^n$, a batch of model samples.\\ 
    \setstretch{1.0}
    \Init{
    Set time $t = t_0$.\\
    Draw initial conditions $\hat{x}^{(i)}_{t_0} \sim \rho_0$ for $i = 1, \hdots, n$.\\ 
    Construct $\hat{s}(t, x) = -\hat{\eta}_z(t, x) / \alpha(t)$.\\
    Construct $\hat{b}(t, x) = \dot{\alpha}(t)\hat{\eta}_z^\OS(t, x) + \frac{\dot\beta(t)}{\beta(t)}\left(x - \alpha(t)\hat{\eta}_z^\OS(t, x)\right)$.\\
    Construct $\hat b_{\fwd}(t, x) = \hat b(t, x) + \epsilon(t) \hat s(t, x)$. \tcp{Reduces to $\hat b$ for $\epsilon(t) = 0$ (ODE).}
    }
    \While{$t < t_f$}{
        Propagate $\hat{x}^{(i)}_{t+\Delta t} = \texttt{TakeStep}(t, \hat{x}^{(i)}_t, b_\fwd, \epsilon, \Delta t)$ for $i = 1, \hdots, n$. \tcp{ODE or SDE integrator.}
        Update $t = t + \Delta t$.
    }
    \textbf{Return}: $\{\hat{x}^{(i)}\}_{i=1}^n$.
    \label{alg:sampling_os}
\end{algorithm}

As shown in~\eqref{eq:b:os:solved}, and as considered in Section~\ref{sec:denoiser}, the denoiser $\eta_z^{\OS}$ is sufficient to represent the velocity field $b$ appearing in the probability flow equation~\eqref{eq:ode:1}.

Using this definition for $b$ and the relationship $s(t, x) = -\eta_z(t, x) / \gamma(t)$, we state the following ordinary and stochastic differential equations for sampling
\begin{equation}
\begin{aligned}
    \text{ODE}:& \quad \dot X_t = \dot \alpha(t) \eta^\OS_z(t,X_t)  + \frac{\dot \beta(t)}{\beta(t)} \big (X_t- \alpha(t) \eta^\OS_z(t,X_t) \big ) \\
    \text{SDE}:& \quad dX^\fwd_t =  \big ( \dot \alpha(t) \eta^\OS_z(t,X^\fwd_t)  + \frac{\dot \beta(t)}{\beta(t)} \big (X^\fwd_t- \alpha(t) \eta^\OS_z(t,X^\fwd_t) \big ) - \frac{\epsilon(t)}{\alpha(t)} \eta^\OS_z(t,X^\fwd_t) \big ) dt \\
    &\qquad\qquad + \sqrt{2 \epsilon(t)} dW_t.
\end{aligned}
\end{equation}
Because $\beta(0) = 0$, the drift is numerically singular in both equations. 
However, $b(t=0, x)$ has a finite limit
\begin{equation}
    \label{eqn:b_nonsing}
    b(t=0, x) = \dot\alpha(0) x + \dot \beta(0) \mathbb E[x_1],
\end{equation}
as originally given in~\eqref{eq:b:lin:t1}.
Equation~\eqref{eqn:b_nonsing} can be estimated using available data, which means that when learning a one-sided interpolant, ODE and SDE-based generative models can be defined exactly on the interval $t \in [0,1]$ using only a score or a denoiser without singularity. 

The factor of $\alpha(t)^{-1}$ in the final term of the SDE could pose numerical problems at $t=1$, as $\alpha(1) = 0$. As discussed in the paragraph above, a choice of $\epsilon(t)$ which is such that $\epsilon(t)/\alpha(t) \rightarrow C$  for some constant $C$ as $t\rightarrow 1$ avoids any issue.

An algorithm for sampling with only the denoiser $\eta_z^\OS$ is given in Algorithm~\ref{alg:sampling_os}.

\section{Numerical results}
\label{sec:numerics}
\begin{figure}[t!]
    \centering
    \includegraphics[width=0.65\linewidth]{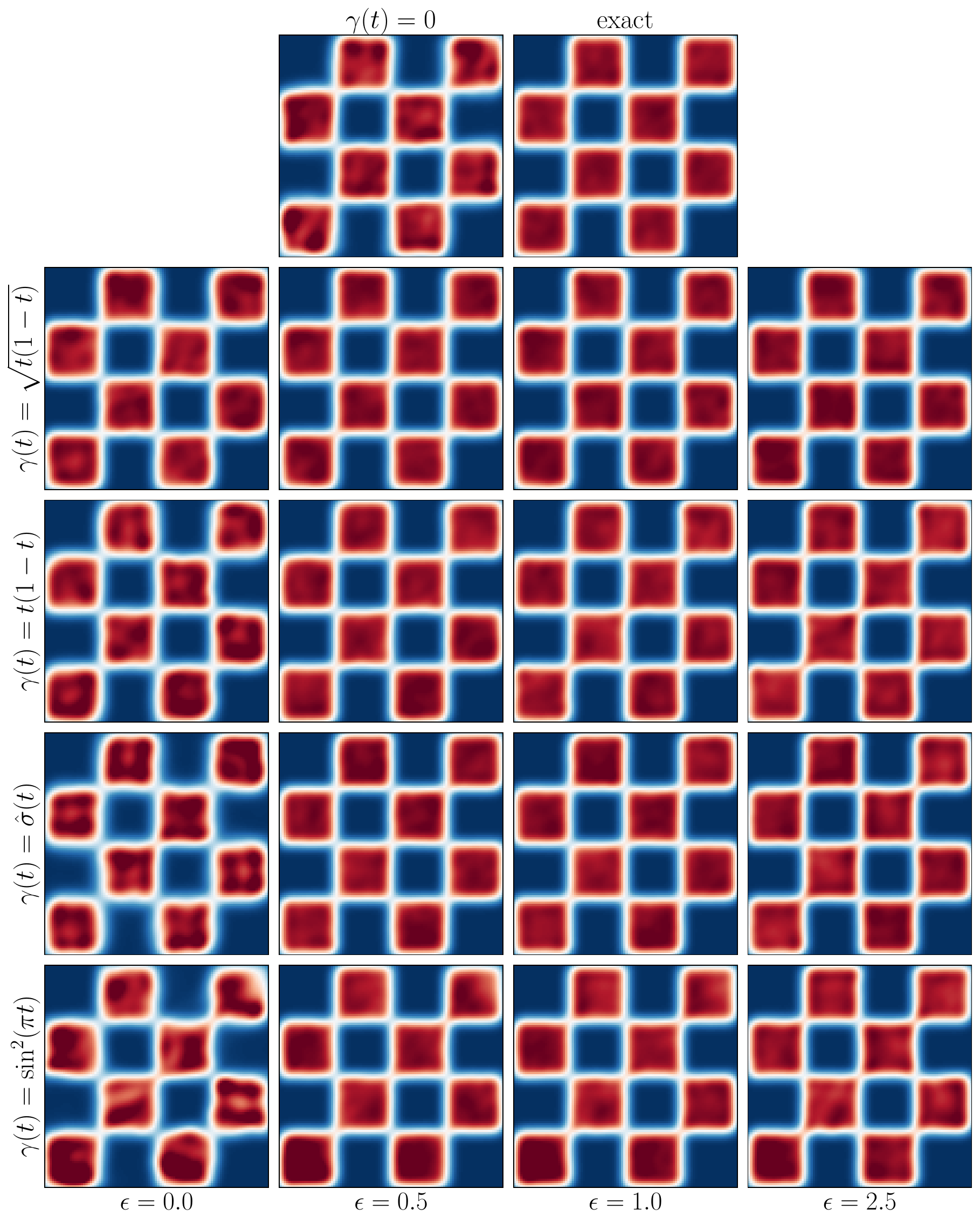}
    \caption{\textbf{The effects of $\gamma(t)$ and $\epsilon$ on sample quality: qualitative comparison.} Kernel density estimates of $\hat\rho(1)$ for models with different choices of $\gamma$ and $\epsilon$.  Sampling with $\epsilon = 0$ corresponds to using the probability flow with the learned drift~$\hat b = \hat v - \gamma\dot \gamma \hat s$, whereas sampling with $\epsilon >0 $ corresponds to using the SDE with  the learned drift~$\hat b$ and score~$\hat s$. We find empirically that SDE sampling is generically better than ODE sampling for this target density, though the gap is smallest for the probability flow specified with $\gamma(t) = \sqrt{t(1-t)}$, in agreement with the Remark~\ref{rem:endpoints} regarding the influence of $\gamma$ on $b$ at the endpoints. 
    The SDE performs well at any noise level, though numerically integrating it for higher $\epsilon$ requires a smaller step size.}
    \label{fig:ode-sde-big}
\end{figure}

So far, we have been focused on the impact of $\alpha$, $\beta$, and $\gamma$ in~\eqref{eq:lin:interp} on the density $\rho(t)$, which we illustrated analytically. 
In this section, we study examples where the drift coefficients must be learned over parametric function classes. 
In particular, we explore numerically the tradeoffs between generative models based on ODEs and SDEs, as well as the various design choices introduced in Sections~\ref{sec:generalization},~\ref{sec:gen}, and~\ref{sec:practical}. 
In Section~\ref{sec:sde:ode}, we consider simple two-dimensional distributions that can be visualized easily.
In Section~\ref{sec:num_gmm}, we consider high-dimensional Gaussian mixtures, where we can compare our learned models to analytical solutions.
Finally in Section~\ref{sec:num_images} we perform  some experiments in image generation.

\subsection{Deterministic versus stochastic models: 2D}
\label{sec:sde:ode}
\begin{figure}[t!]
    \centering
    \includegraphics[width=0.8\linewidth]{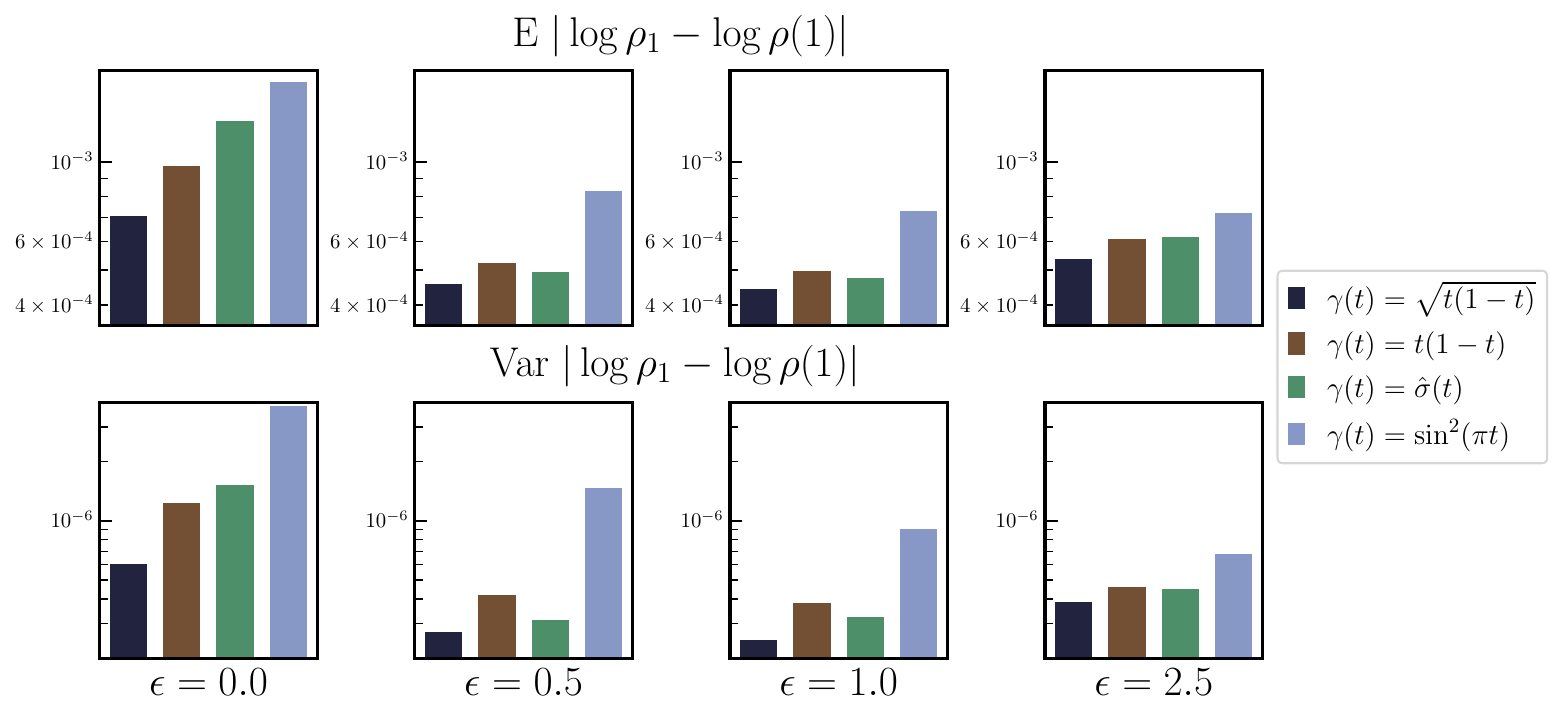}
    \caption{\textbf{The effects of $\gamma(t)$ and $\epsilon$ on sample quality: quantitative comparison.} 
    For each~$\gamma$ and each~$\epsilon$ specified in Figure~\ref{fig:ode-sde-big}, we compute the mean and variance of the absolute value of the difference of $\log \rho_1$ (exact) and $\log \hat \rho(1)$ (model). 
    The model specified with $\gamma(t) = \sqrt{t(1-t)}$ is the best performing probability flow ($\epsilon = 0$). 
    At large $\epsilon$, SDE sampling with the same learned drift~$\hat b$ and score~$\hat s$ performs better, complementing the observations in the previous figure.} 
    \label{fig:fdiv}
\end{figure}

As shown in Section~\ref{sec:cont:eq}, the evolution of $\rho(t)$ can be captured exactly by either the transport equation~\eqref{eq:transport} or by the forward and backward Fokker-Planck equations~\eqref{eq:fpe} and~\eqref{eq:fpe:tr}.
These perspectives lead to generative models that are either based on the deterministic dynamics~\eqref{eq:ode:1} or the forward and backward stochastic dynamics~\eqref{eq:sde:1} and~\eqref{eq:sde:R}, where the level of stochasticity can be tuned by varying the diffusion coefficient $\eps(t)$. 
We showed in Section~\ref{sec:likelihood_bounds} that setting a constant $\eps(t) = \eps > 0$ can offer better control on the likelihood when using an imperfect velocity~$b$ and an imperfect score~$s$.
Moreover, the optimal choice of $\eps$ is determined by the relative accuracy of the estimates $\hat b$ and $\hat s$. 
Having laid out the evolution of $\rho(t)$ for different choices of $\gamma$ in the previous section, we now show how different values of~$\epsilon$ can build these densities up from individual trajectories.
The stochasticity intrinsic to the sampling process increases with $\epsilon$, but by construction, the marginal density $\rho(t)$ for fixed $\alpha$, $\beta$ and $\gamma$ is independent of $\epsilon$.

\paragraph{The roles of $\gamma(t)$ and $\eps$ for 2D density estimation.} 
To explore the roles of $\gamma$ and $\epsilon$, we consider a target density $\rho_1$ whose mass concentrates on a two-dimensional checkerboard and a base density $\rho_0 = \mathsf{N}(0,\Id)$; here, the target was chosen to highlight the ability of the method to learn a challenging density with sharp boundaries. 
The same model architecture and training procedure was used to learn both $v$ and $s$ for several choices of $\gamma$ given in Table~\ref{tab:gammas}. 
The feed-forward network was defined with $4$ layers, each of size $512$, and with the ReLU~\cite{vinod2010} as an activation function.

After training, we draw 300,000 samples using either an ODE ($\epsilon=0$) or an SDE with $\epsilon = 0.5$, $\epsilon=1.0$, or $\epsilon = 2.5$. 
We compute kernel density estimates for each resulting density, which we compare to the exact density and to the original stochastic interpolant from~\cite{albergo2023building} (obtained by setting $\gamma = 0$). 
Results are given in Figure~\ref{fig:ode-sde-big} for each $\gamma$ and each $\epsilon$. 
Sampling with $\epsilon > 0 $ empirically performs better, though the gap is smallest when using the $\gamma$ specified in \eqref{eq:gam:Bt}. 
Moreover, even when $\epsilon = 0$, using the probability flow with $\gamma$ given by~\eqref{eq:gam:Bt} performs better than the original interpolant from \cite{albergo2023building}. 
Numerical comparisons of the mean and variance of the absolute value of the difference of $\log \rho_1$ (exact) from $\log \hat\rho(1)$ (model) for the various configurations are given in Figure~\ref{fig:fdiv}, which corroborate the above observations.

\subsection{Deterministic versus stochastic models: 128D Gaussian mixtures}
\label{sec:num_gmm}
We now study the performance of the stochastic interpolant method in the case where the target is a high-dimensional Gaussian mixture. Gaussian mixtures (GMs) are a convenient class of target distributions to study, because they can be made arbitrarily complex by increasing the number of modes, their separation, and the overall dimensionality. 
Moreover, by considering low-dimensional projections, we can compute quantitative error metrics such as the $\mathsf{KL}$-divergence between the target and the model as a function of the (constant) diffusion coefficient $\epsilon$.
This enables us to quantify the tradeoffs of ODE and SDE-based samplers.

\paragraph{Experimental details.} 
We consider the problem of mapping $\rho_0 = \mathsf{N}(0, \Id)$ to a Gaussian mixture with five modes in dimension $d=128$. 
The mean $m_i\in\RR^d$ of each mode is drawn i.i.d. $m_i \sim \mathsf{N}(0, \sigma^2 \Id)$ with $\sigma = 7.5$. 
To maximize performance at high~$\epsilon$, the timestep should be adapted to~$\epsilon$; here, we chose to use a fixed computational budget that performs well for moderate levels of $\epsilon$ to avoid computational effort that may become unreasonable in practice. 
\begin{wrapfigure}[16]{r}{0.35\textwidth}
\centering
\includegraphics[width=\linewidth]{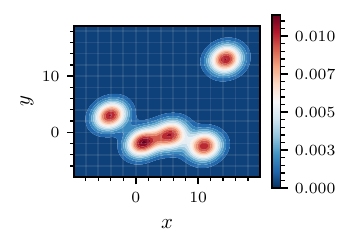}
\caption{\textbf{Gaussian mixtures: target projection.} Low-dimensional marginal of the target density $\rho_1$ for the Gaussian mixture experiment, visualized via KDE.}
\label{fig:gmm_target}
\end{wrapfigure}
Each covariance $C_i \in \RR^{d\times d}$ is also drawn randomly with $C_i = \tfrac{1}{d}W_i^\T W_i + \Id$ and $(W_i)_{kl} \sim \mathsf{N}(0, 1)$ for $k, l = 1, \hdots, d$; this choice ensures that each covariance is a positive definite perturbation of the identity with diagonal entries that are $O(1)$ with respect to dimension $d$. 
For a \textit{fixed} random draw of the means $\{m_i\}_{i=1}^5$ and covariances $\{C_i\}_{i=1}^5$, we study the four combinations of learning $b$ or $v$ and $s$ or $\eta_z$ to form a stochastic interpolant from $\rho_0$ to $\rho_1$. 
In each case, we consider a linear interpolant with $\alpha(t) = 1-t$, $\beta(t) = t$, and $\gamma(t) = \sqrt{t(1-t)}$. 
For visual reference, a projection of the target density $\rho_1$ onto the first two coordinates is depicted in Figure~\ref{fig:gmm_target} -- it contains significant multimodality, several modes that are difficult to distinguish, and one mode that is well-separated from the others, which requires nontrivial transport to resolve. 
In the following experiments, all samples were generated with the fourth-order Dormand-Prince adaptive ODE solver (\texttt{dopri5}) for $\epsilon=0$ and by using one thousand timesteps of the Heun SDE integrator introduced in~\cite{Karras2022edm} for $\epsilon \neq 0$. 
When learning $\eta_z$, to avoid singularity at $t=0$ and $t=1$ when dividing by $\gamma(t)$ in the formula $s(t, x) = -\eta(t, x) / \gamma(t)$, we set $t_0 = 10^{-4}$ and $t_f = 1 - t_0$ in Algorithm~\ref{alg:sampling}. For all other cases, we set $t_0 = 0$ and $t_f = 1$.

\begin{figure}[t!]
    \centering
    \includegraphics[width=\textwidth]{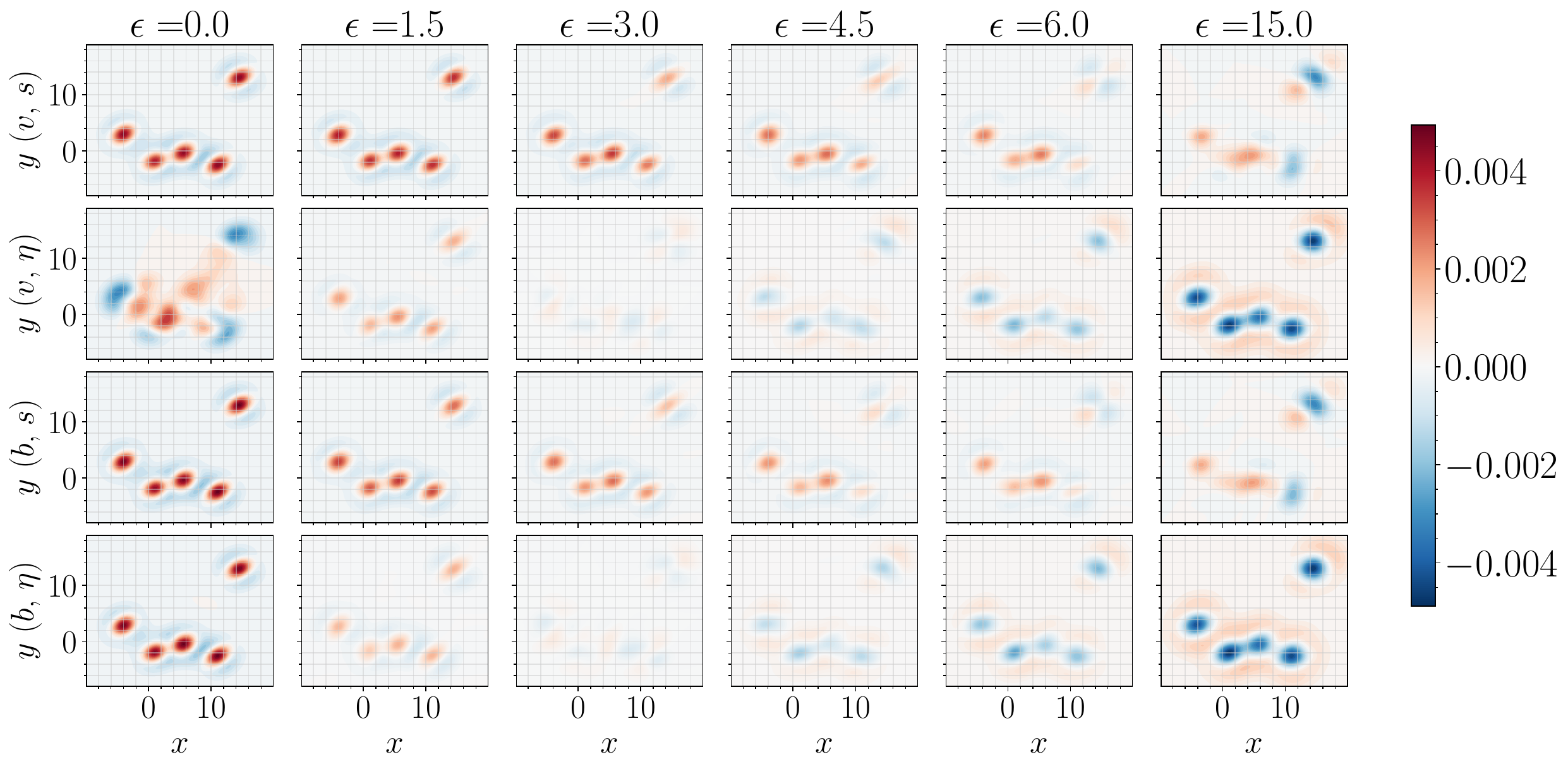}
    \caption{\textbf{Gaussian mixtures: density errors.} Errors $\hat{\rho}_1(x, y) - \rho_1(x, y)$ in the marginals over the first two coordinates for all four variations of learning $b$ or $v$ and $s$ or $\eta$, computed via kernel density estimation. For small $\epsilon$, the model densities tend to be overly-concentrated, and overestimate the density within the modes and underestimate the densities in the tails. As $\epsilon$ increases, the model becomes less concentrated and more accurately represents the target density. For $\epsilon$ too large, the model becomes overly spread out and under-estimates the density within the modes. \textit{Note: visualizations show two-dimensional slices of a $128$-dimensional density.}}
    \label{fig:gmm_density_errors}
\end{figure}
\paragraph{Quantitative metric} 
To quantify performance, we make use of an error metric given by a $\mathsf{KL}$-divergence between kernel density estimates (KDE) of low-dimensional marginals of $\rho_1$ and the model density $\hat{\rho}_1$; this error metric was chosen for computational tractability and interpretability. 
To compute it, we draw $50,000$ samples from $\rho_1$ and each $\hat{\rho}_1$. We obtain samples from the marginal density over the first two coordinates by projection, and then compute a Gaussian KDE with bandwidth parameter chosen by Scott's rule. We then draw a fresh set of $N_e = 50,000$ samples $\{x_i\}_{i=1}^{N_e}$ with each $x_i \sim \rho_1$ for evaluation.
To compute the $\mathsf{KL}$-divergence, we form a Monte-Carlo estimate with control variate 
\begin{equation}
 \KL{\rho_1}{\hat{\rho}_1} \approx \frac{1}{N_e}\sum_{i=1}^{N_e} \left(\log \rho_1(x_i) - \log \hat{\rho}_1(x_i) - \left(\frac{\hat{\rho}_1(x_i)}{\rho_1(x_i)} - 1\right)\right).   
\end{equation} 
We found use of the control variate $\hat{\rho}_1 / \rho_1 - 1$ helpful to reduce variance in the Monte-Carlo estimate; moreover, by concavity of the logarithm, use of the control variate ensures that the Monte-Carlo estimate cannot become negative.

\paragraph{Results.}
\begin{figure*}[t!]
    \centering
    \includegraphics[width=\textwidth]{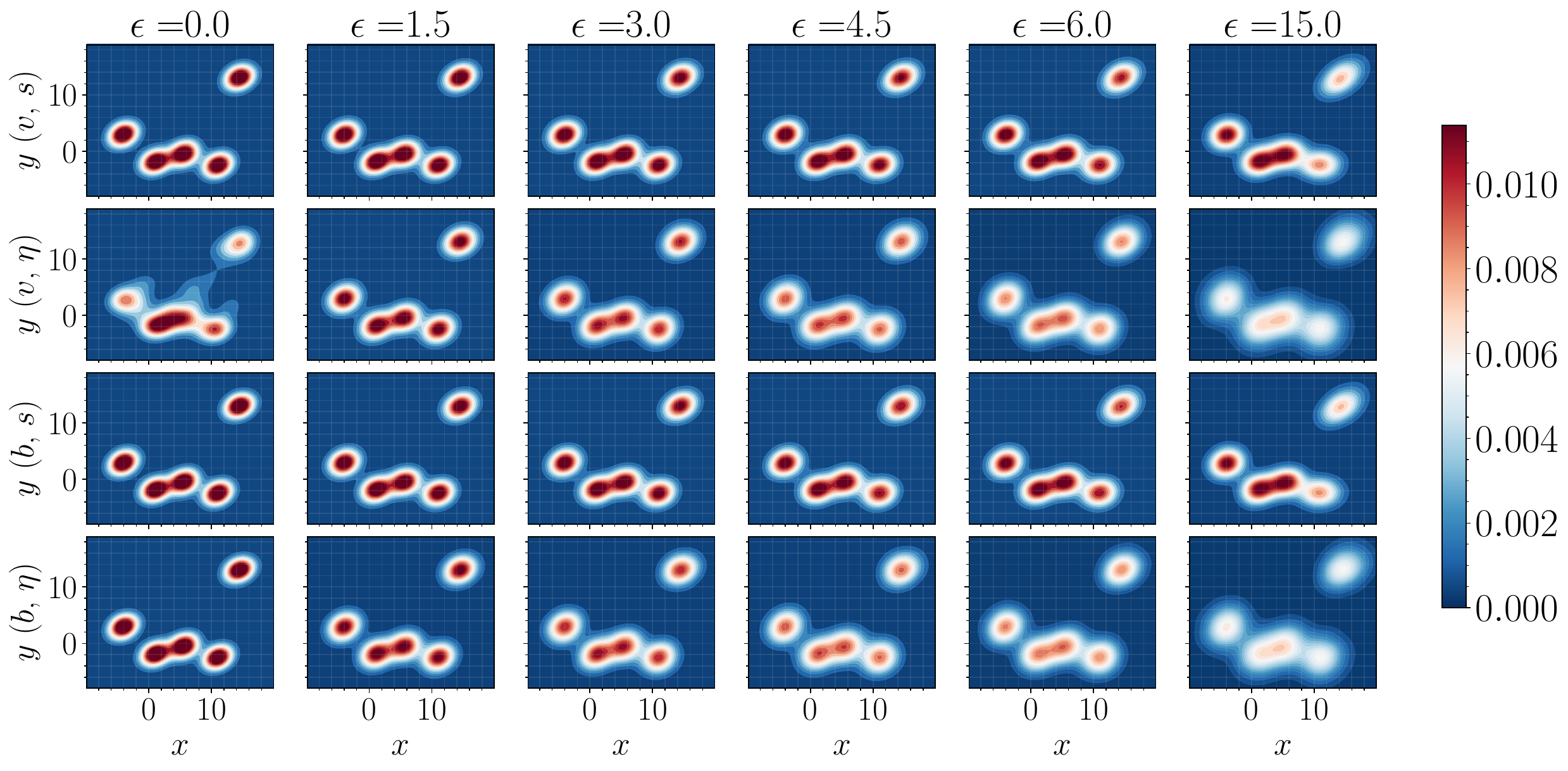}
    \caption{\textbf{Gaussian mixtures: densities.} A visualization of the marginal densities in the first two variables of the model density $\hat{\rho}_1$ computed for all four variations of learning $b$ or $v$ and $s$ or $\eta$, computed via kernel density estimation. \textit{Note: visualizations show two-dimensional slices of a $128$-dimensional density.}}
    \label{fig:gmm_densities}
\end{figure*}
\begin{figure*}
\centering
\includegraphics[width=0.5\linewidth]{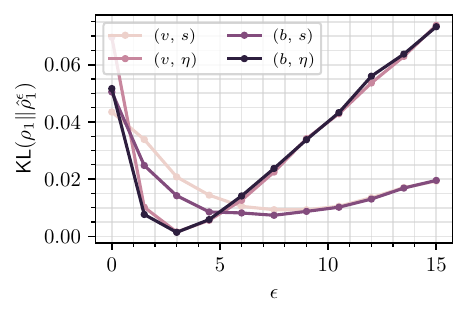}
\caption{\textbf{Gaussian mixtures: quantitative comparison.} $\KL{\rho_1}{\hat{\rho}_1^{\epsilon}}$ as a function of $\epsilon$ when learning each of the four possible sets of drift coefficients $(b,s)$, $(b, \eta)$, $(v,s)$, $(v, \eta)$. The best performance is achieved when learning $b$ and $\eta$, along with a proper choice of $\epsilon > 0$.}
\label{fig:gmm_kl}
\end{figure*}

Figures~\ref{fig:gmm_density_errors} and~\ref{fig:gmm_densities} display two-dimensional projections (computed via KDE) of the model density error $\hat{\rho}_1 - \rho_1$ and the model density $\hat{\rho}_1$ itself, respectively, for different instantiations of Algorithms~\ref{alg:learning:b} and~\ref{alg:learning:eta:sde} and different choices of $\epsilon$ in Algorithm~\ref{alg:sampling}. 
Taken together with Figure~\ref{fig:gmm_target}, these results demonstrate qualitatively that small values of $\epsilon$ tend to over-estimate the density within the modes and under-estimate the density in the tails. 
Conversely, when $\epsilon$ is taken too large, the model tends to under-estimate the modes and over-estimate the tails. Somewhere in between (and for differing levels of $\epsilon$), every model obtains its optimal performance. 
Figure~\ref{fig:gmm_kl} makes these observations quantitative, and displays the $\mathsf{KL}$-divergence from the target marginal to the model marginal $\KL{\rho_1}{\hat{\rho}_1^\epsilon}$ as a function of $\epsilon$, with each data point on the curve matching the models depicted in Figures~\ref{fig:gmm_density_errors} and~\ref{fig:gmm_densities}. 
We find that for each case, there is an optimal value of $\epsilon \neq 0$, in line with the qualitative picture put forth by Figures~\ref{fig:gmm_density_errors} and~\ref{fig:gmm_densities}. 
Moreover, we find that learning $b$ generically performs better than learning $v$, and that learning $\eta$ generically performs better than learning $s$ (except when $\epsilon$ is taken large enough that performance starts to degrade). 
With proper treatment of the singularity in the sampling algorithm when using the denoiser in the construction of $s(t, x) = -\eta(t, x)/\gamma(t)$ -- either by capping $t_0 \neq 0$ and $t_f \neq 1$ or by properly tuning $\epsilon(t)$ as discussed in Section \ref{sec:sampling} -- our results suggest that learning the denoiser is best practice.

\subsection{Image generation}
\label{sec:num_images}
In the following, we demonstrate that the proposed method scales straightforwardly to high-dimensional problems like image generation. 
To this end, we illustrate the use of our approach  on the $128\times128$ Oxford flowers dataset \cite{Nilsback06} by testing two different variations of the interpolant for image generation: the one-sided interpolant, using $\rho_0 = \mathsf{N}(0,\Id)$, as well as the mirror interpolant, where $\rho_0 = \rho_1$ both represent the data distribution.
The purpose of this section is to demonstrate that our theory is well-motivated, and that it provides a framework that is both scalable and flexible. 
In this regard, image generation is a convenient exercise, but is not the main focus of this work, and we will leave a more thorough study on other datasets such as ImageNet with standard benchmarks such as the Frechet Inception Distance (FID) for a future study. 

\paragraph{Generation from Gaussian $\rho_0$.} 
We train spatially-linear one-sided interpolants $x_t = (1 -t)z + t x_1$ and $x_t = \cos({\tfrac{\pi}{2}}t) z + \sin({\tfrac{\pi}{2} t}) x_1$ on the $128\times128$ Oxford flowers dataset, where we take $z \sim \mathsf{N}(0,\Id)$ and $x_1$ from the data distribution. 
Based on our results for Gaussian mixtures, we learn the drift $b(t,x)$, the score $s(t,x)$, and the denoiser $\eta_z(t,x)$ to benchmark our generative models based on ODEs or SDEs.
In all cases, we parameterize the networks representing $\hat \eta$, $\hat s$ and $\hat b$ using the U-Net architecture used in~\cite{ho2020}. 
Minimization of the objective functions given in Section \ref{sec:onesided} is performed using the Adam optimizer. 
Details of the architecture in both cases and all training hyperparameters are provided in Appendix \ref{app:exp:img}. 

\begin{figure}[t!]
    \centering
    \includegraphics[width=\linewidth]{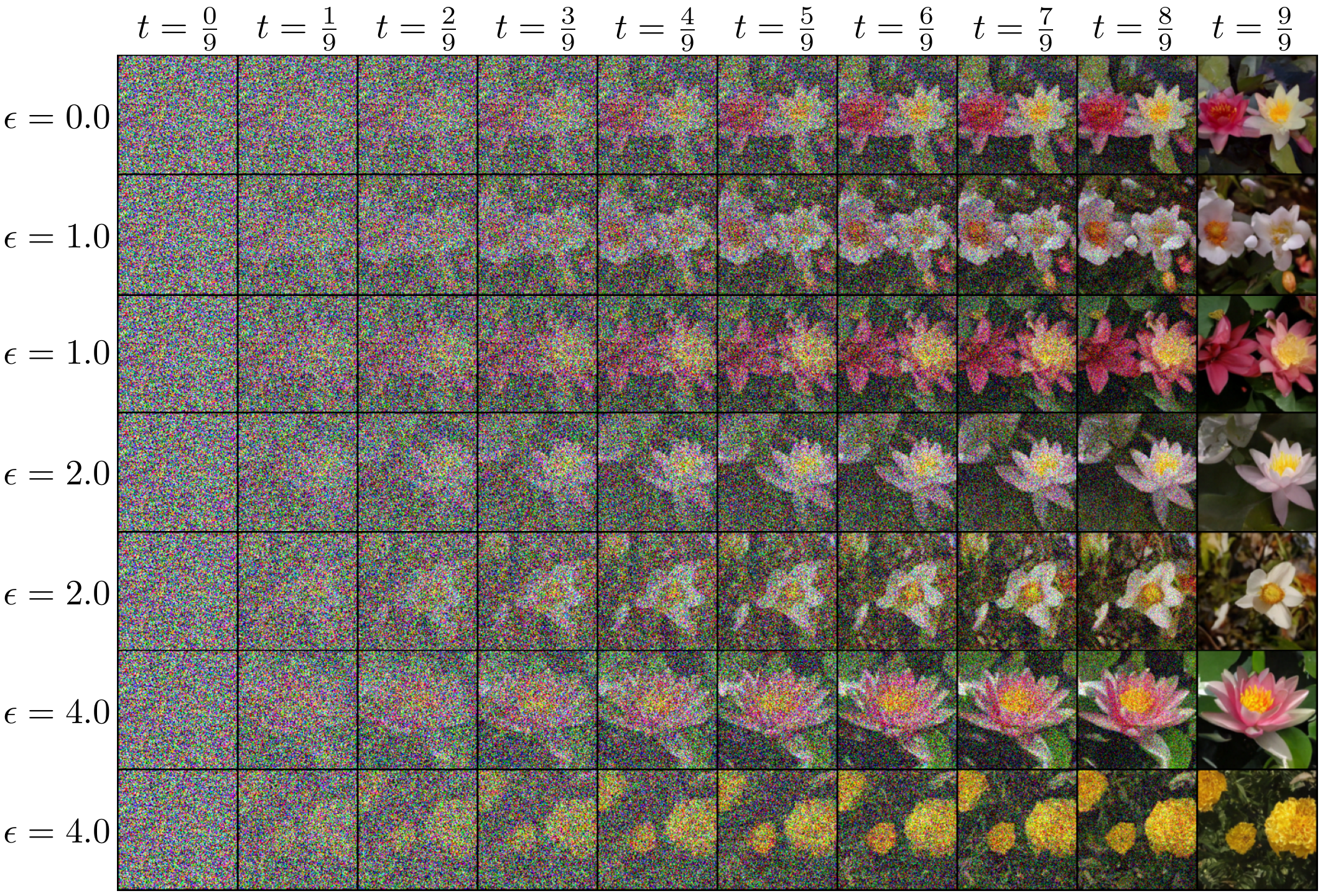}
    \caption{\textbf{Image generation: Oxford flowers.}
    Example generated flowers from the same initial condition $x_0$ using either the ODE with $\epsilon = 0$ and learned $b$ or the SDE for various increasing values of $\epsilon$ with learned $b$ and $s$.
    For $\epsilon = 0$, sampling is done using the \texttt{dopri5} solver and therefore the number of steps is adaptive. 
    Otherwise, $2000$, $2500$, and $4000$ steps were taken using the Heun solver for $\epsilon = 1.0$, $2.0$, and $4.0$ respectively. }
    \label{fig:flowers:sde:many}
\end{figure}

Like in the case of learning Gaussian mixtures, we use the fourth-order \texttt{dopri5} solver when sampling with the ODE and the Heun method for the SDE, as detailed in Algorithm \ref{alg:sampling}. 
When learning a denoiser $\eta_z$, we found it beneficial to complete the image generation with a final denoising step, in which we set $\eps=0$ and switch the integrator to the one given in~\eqref{eq:iterate}.

Exemplary images generated from the model using the ODE and the SDE with various value of the diffusion coefficient $\epsilon$ are shown in Figure~\ref{fig:flowers:sde:many}, starting from the same sample from $\rho_0$. 
The result illustrates that different images can be generated from the same sample when using the SDE, and their diversity increases as we increase the diffusion coefficient $\eps$.
To highlight that the model does not memorize the training set,  in Figure~\ref{fig:overfit} we compare an example generated image to its five nearest neighbors in the training set (measured in $\ell_1$ norm), which appear visually quite different.

\begin{figure}[t!]
    \centering
    \includegraphics[width=0.6\linewidth]{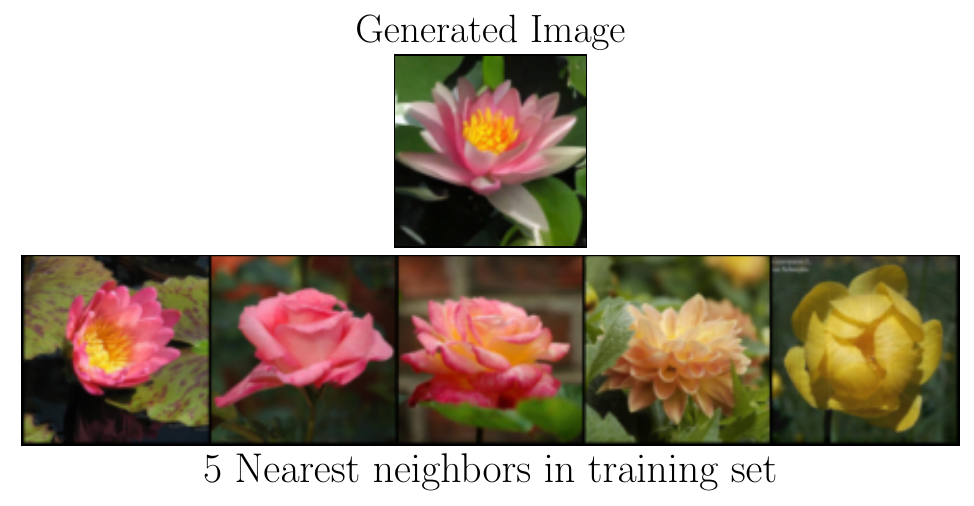}
    \caption{\textbf{No memorization.} Top: An example generated image from a trained linear one-sided interpolant. Bottom: Five nearest neighbors (in $\ell_1$ norm) to the generated image in the dataset. The nearest neighbors are visually distinct, highlighting that the interpolant does not overfit on the dataset.}
    \label{fig:overfit}
\end{figure}

\paragraph{Mirror interpolant.}
We consider the mirror interpolant $x_t = x_1 + \gamma(t) z$, for which~\eqref{eq:mirror:b:2}  shows that the drift $b$ is given in terms the denoiser $\eta_z$ by $b(t,x) = \dot \gamma(t) \eta_z(t)$; this means that it is sufficient to only learn an estimate $\hat{\eta}_z$ to construct a generative model.
Similar to the previous section, we demonstrate this on the Oxford flowers dataset, again making use of a U-Net parameterization for $\hat{\eta}_z(t, x)$.
Further experimental details can be found in Appendix~\ref{app:exp:img}.
In this setup the output image at time $t=1$ is the same as the input image if we use the ODE~\eqref{eq:ode:1}; with the SDE, however, we can generate new images from the same input. This is illustrated  in Figure~\ref{fig:mirror}, where we show how a sample image from the dataset $\rho_1$ is pushed forward through the SDE \eqref{eq:sde:1} with $\eps(t) = \eps = 10$. 
As can be seen the original image is resampled to a proximal flower not seen in the dataset.

\begin{figure}[t!]
    \centering
    \includegraphics[width=\linewidth]{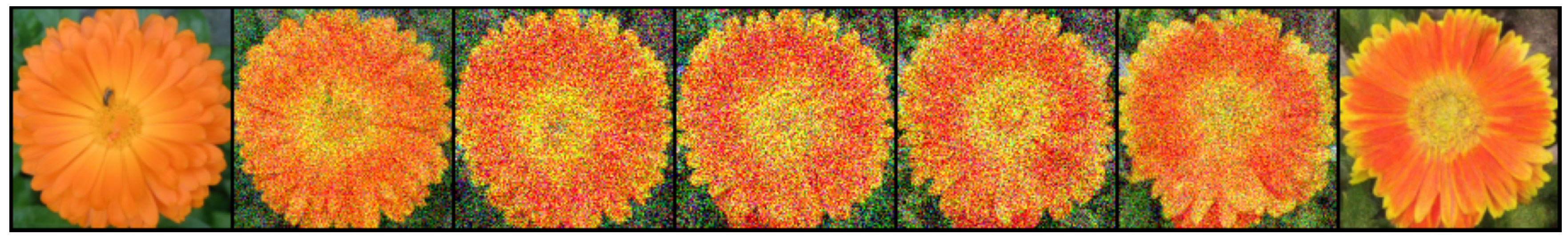}
    \caption{\textbf{Mirror interpolant.} 
    Example trajectory from a learned denoiser model $\hat b = \dot \gamma(t)\hat \eta_z$ for the mirror interpolant $x_t = x_1 + \gamma(t) z$ on the $128\times128$ Oxford flowers dataset with $\gamma(t) = \sqrt{t(1-t)}$. 
    The parametric form for $\hat \eta_z$ is the U-Net from \cite{ho2020}, with hyperparameter details given in Appendix \ref{app:exp}. 
    The choice of $\epsilon$ in the generative SDE given in \eqref{eq:sde:1} influences the extent to which the generated image differs from the original. Here, $\epsilon(t) = 10.0$.}
    \label{fig:mirror}
\end{figure}

\section{Conclusion}
\label{sec:conc}
The above exposition provides a full treatment of the stochastic interpolant method, as well as a careful consideration of its relation to existing literature. 
Our goal is to provide a general framework that can be used to devise generative models built upon dynamical transport of measure. 
To this end, we have detailed mathematical theory and efficient algorithms for constructing both deterministic and stochastic generative models that map between two densities exactly in finite time. 
Along the way, we have illustrated the various design parameters that can be used to shape this process, with connections, for example, to optimal transport and Schr\"odinger bridges.
While we detail specific instantiations, such as the mirror and one-sided interpolants, we highlight that there is a much broader space of possible designs that may be relevant for future applications.
Several candidate application domains include the solution of inverse problems such as image inpainting and super-resolution, spatiotemporal forecasting of dynamical systems, and scientific problems such as sampling of molecular configurations and machine learning-assisted Markov chain Monte-Carlo.

\section*{Acknowledgements}
\label{sec:ack}
We thank Joan Bruna, Jonathan Niles-Weed, Loucas Pillaud-Vivien, and C\'edric Gerbelot for helpful discussions regarding stability estimates of dynamical transport.  We are also grateful to Qiang Liu, Ricky Chen, and Yaron Lipman for feedback on previous and related work, and to Kyle Cranmer and Michael Lindsey for discussions on transport costs. We thank Mark Goldstein for insightful comments regarding practical considerations when training denoising-diffusion models. MSA is supported by the National Science
Foundation under the award PHY-2141336.  EVE is supported by the National Science Foundation under awards DMR-1420073, DMS-2012510, and DMS-2134216, by the Simons Collaboration on Wave Turbulence,
Grant No. 617006, and by a Vannevar Bush Faculty Fellowship.

\appendix
\section{Bridging two Gaussian mixture densities}
\label{app:Gauss:mixt}
In this appendix, we consider the case where $\rho_0$ and $\rho_1$ are both Gaussian mixture densities. We denote by
\begin{equation}
\label{eq:NmC}
\begin{aligned}
    {\sf N}(x|m,C) &= (2\pi)^{-d/2} [\det C]^{-1/2} \exp\left(-\tfrac12 (x-m)^\T C^{-1} (x-m)\right),\\
    & = (2\pi)^{-d} \int_{\RR^d} e^{ik\cdot (x-m) -\frac12 k^\T C k} dk,
\end{aligned}
\end{equation}
the Gaussian probability density with mean vector $m \in \RR^d$ and positive-definite symmetric covariance matrix $C=C^\T\in \RR^{d\times d}$. We assume that
\begin{equation}
    \label{eq:rho0:mixt}
    \rho_0(x) = \sum_{i=1}^{N_0} p^0_i {\sf N}(x|m^0_i,C^0_i), \qquad
    \rho_1(x) = \sum_{i=1}^{N_1} p^1_i {\sf N}(x|m^1_i,C^1_i)
\end{equation}
where $N_0,N_1\in \NN$,  $p^0_i>0$ with $\sum_{i=1}^{N_0} p_i^0 = 1$, $m_i^0\in \RR^d$, $C_i^0= (C_i^0)^\T \in \RR^{d\times d}$ positive-definite, and similarly for $p_i^1$, $m_i^1$, and $C_i^1$. 
We have:
\begin{restatable}{proposition}{gaussmixt}
\label{th:Gauss:mixt}
Consider the process $x_t$ defined in~\eqref{eq:stochinterp} using the probability densities in~\eqref{eq:rho0:mixt} and the interpolant in~\eqref{eq:lin:interp}, i.e.
\begin{equation}
    \tag{\ref{eq:lin:interp}}
    x^\LIN_t = \alpha(t) x_0+ \beta(t) x_1 + \gamma(t) z,
\end{equation} 
where $x_0,\sim\rho_0$, $x_1\sim \rho_1$, and $z \sim {\sf N}(0,\Id)$ with $x_0\perp x_1\perp z$, and $\alpha, \beta, \gamma^2\in C^2([0,1])$ satisfy the conditions in \eqref{eq:lin:interp:a:b}. Denote \begin{equation}
    \label{eq:mij:Cij}
    m_{ij}(t) = \alpha(t) m^0_i+\beta(t) m^1_j, \quad C_{ij}(t) = \alpha^2(t) C^0_i+\beta^2(t) C^1_j +\gamma^2(t) \text{\it Id}, 
\end{equation}
where $i=1,\ldots, N_0,$ $j=1,\ldots, N_1$.
Then the probability density $\rho$ of $x_t$ is the Gaussian mixture density
\begin{equation}
    \label{eq:rhot:Gaussmixt}
    \rho(t,x) = \sum_{i=1}^{N_0}\sum_{j=1}^{N_1}
    p^0_ip^1_j {\sf N}(x| m_{ij}(t),C_{ij}(t))
\end{equation}
and  the velocity $b$ and the score $s$ defined in~\eqref{eq:b:ode:def} and \eqref{eq:s:def} are 
\begin{equation}
    \label{eq:vt:Gausmixt}
    b(t,x) = \frac{\sum_{i=1}^{N_0}\sum_{j=1}^{N_1}
    p^0_ip^1_j \left(\dot m_{ij}(t) + \tfrac12\dot C_{ij}(t) C^{-1}_{ij}(t)(x-m^{ij}(t))  \right) {\sf N}(x| m_{ij}(t),C_{ij}(t))}{\sum_{i=1}^{N_0}\sum_{j=1}^{N_1}
    p^0_ip^1_j {\sf N}(x| m_{ij}(t),C_{ij}(t))},
\end{equation}
and
\begin{equation}
    \label{eq:st:Gausmixt}
    s(t,x) = -\frac{\sum_{i=1}^{N_0}\sum_{j=1}^{N_1}
    p^0_ip^1_j  C^{-1}_{ij}(t)(x-m_{ij}(t))  {\sf N}(x| m_{ij}(t),C_{ij}(t))}{\sum_{i=1}^{N_0}\sum_{j=1}^{N_1}
    p^0_i p^1_j {\sf N}(x| m_{ij}(t),C_{ij}(t))}.
\end{equation}
\end{restatable}
This proposition implies that $b$ and $s$ grow at most linearly in $x$, and are approximately linear in regions where the modes of $\rho(t,x)$ remain well-separated. In particular, if $\rho_0$ and $\rho_1$ are both Gaussian densities, $\rho_0=\mathsf{N}(m_0,C_0)$ and $\rho_1=\mathsf{N}(m_1,C_1)$, we have  
\begin{equation}
    \label{eq:vt:mode:Gaussmixt}
    b(t,x)=\dot m(t) + \tfrac12 \dot C(t) C^{-1}(t)(x-m(t)),
\end{equation}
and
\begin{equation}
    \label{eq:st:mode:Gaussmixt}
    s(t,x)= - C^{-1}(t)(x-m(t))  ,
\end{equation}
where
\begin{equation}
    \label{eq:mt:Ct}
    m(t) = \alpha(t) m_0+ \beta(t) m_1, \quad C(t) = \alpha^2(t) C_0+ \beta^2(t) C_1 +\gamma^2(t) \text{\it Id}.
\end{equation}
Note that the probability flow ODE~\eqref{eq:ode:1} associated with the velocity~\eqref{eq:vt:mode:Gaussmixt} is the linear ODE
\begin{equation}
    \label{eq:ode:gm}
    \frac{d}{dt} X_t = \dot m(t) + \tfrac12 \dot C(t) C^{-1}(t)(X_t-m(t)).
\end{equation}
This equation can only be solved analytically if $\dot C(t)$ and $C(t)$ commute (which is the case e.g. if $C_0 = \Id$), but it is easy to see that it always guarantees that 
\begin{equation}
    \label{eq:ode:gm:sol}
    \EE_0 X_t(x_0) =  m(t), \qquad \EE_0 \big[(X_t(x_0)-m(t)) (X_t(x_0)-m(t))^\T\big]  = C(t),
\end{equation}
where $X_t(x_0)$ denotes the solution to~\eqref{eq:ode:gm} for the initial condition $X_{t=0}(x_0) = x_0$ and $\EE_0$ denotes expectation over $x_0\sim \rho_0$. A similar statement is true if we solve \eqref{eq:ode:gm} with final conditions at $t=1$ drawn from $\rho_1$. Similarly, the forward SDE~\eqref{eq:sde:1} associated with the velocity~\eqref{eq:vt:mode:Gaussmixt}  and the score~\eqref{eq:st:mode:Gaussmixt} is the linear SDE
\begin{equation}
    \label{eq:sde:gm}
    d X^\fwd_t = \dot m(t)dt + \big(\tfrac12 \dot C(t) - \eps\big) C^{-1}(t)(X^\fwd_t-m(t)) dt + \sqrt{2\eps} dW_t.
\end{equation}
and its solutions are such that  
\begin{equation}
    \label{eq:sde:gm:sol}
    \EE_0 \EE^{x_0}_\fwd X_t^\fwd =  m(t), \qquad \EE_0 \EE^{x_0}_\fwd\big[(X^\fwd_t-m(t) )(X_t^\fwd-m(t))^\T\big]  = C(t),
\end{equation}
where $\EE_\fwd^{x_0}$ denotes expectation over the solution of \eqref{eq:sde:gm} conditional on the event $X^\fwd_{t=0}=x_0$ and $\EE_0$ denotes expectation over $x_0\sim \rho_0$. A similar statement also holds for the backward SDE~\eqref{eq:sde:R}.

\begin{proof} The characteristic function of $\rho(t,x)$ is given by
\begin{equation}
    \label{eq:rhot:Gaussmixt:k}
    g(t,k)  = \EE e^{ik\cdot x_t} = \sum_{i=1}^{N_0}\sum_{j=1}^{N_1}
    p^0_ip^1_j  e^{ik\cdot m_{ij}(t) - \tfrac12 k^\T C_{ij}(t) k } .
\end{equation}
whose inverse Fourier transform is~\eqref{eq:rhot:Gaussmixt}. This automatically implies~\eqref{eq:st:Gausmixt} since we know from~\ref{eq:s:def} that  $s= \nabla \log \rho$. To derive \eqref{eq:vt:mode:Gaussmixt} use the function $m$ defined below in~\eqref{eq:F12k}:
\begin{equation}
    \label{eq:m:Gaussmixt:k}
    m(t,k)  = \sum_{i=1}^{N_0}\sum_{j=1}^{N_1}
    p^0_ip^1_j  (\dot m_{ij}(t) +\tfrac12 i \dot C_{ij}(t) k)e^{ik\cdot m_{ij}(t) - \tfrac12 k^\T C_{ij}^\gamma(t) k } .
\end{equation}
From~\eqref{eq:F1k:c}, we know that the inverse Fourier transform of this function is $b\rho$, so that we obtain
\begin{equation}
    \label{eq:m:Gaussmixt:x}
    b(t,x) \rho(t,x) = \sum_{i=1}^{N_0}\sum_{j=1}^{N_1}
    p^0_ip^1_j  \left(\dot m_{ij}(t) + \tfrac12\dot C_{ij}(t) C_{ij}^{-1}(t)(x-m_{ij}(t))\right) {\sf N}(x| m_{ij}(t),C^\eps_{ij}(t)).
\end{equation}
This gives~\eqref{eq:vt:Gausmixt}.
\end{proof}

\section{Proofs}
In this appendix, we provide the details for proofs omitted from the main text. For ease of reading, a copy of the original theorem statement is provided with the proof.

\subsection{Proof of Theorems~\ref{prop:interpolate}, ~\ref{prop:interpolate_losses}, and~\ref{thm:score}, and Corollary~\ref{prop:interpolate_fpe}.}
\label{app:proof:interpolate}
\interpolation*
\begin{proof}
 Let $g(t,k) = \EE e^{ik \cdot x_t}$,  $k \in \RR^d $, be the characteristic function of $\rho(t,x)$. From the definition of $x_t$ in~\eqref{eq:stochinterp},
\begin{equation}
\label{eq:charact}
 g(t,k) = \EE e^{ik\cdot (I(t,x_0,x_1) + \gamma(t) z)}.
\end{equation}
 Using the independence between $(x_0,x_1)$ and $z$ , we have
\begin{equation}
    \label{eq:gt:k}
    g(t,k) = \EE \left(e^{ik\cdot I(t,x_0,x_1) }\right) \EE \left( e^{i\gamma(t)  k\cdot z}\right) \equiv g_0(t,k)  e^{-\tfrac12 \gamma^2(t) |k|^2 }
\end{equation}
where we defined
\begin{equation}
    \label{eq:G0}
    g_0(t,k) = \EE \left(e^{ik\cdot I(t,x_0,x_1) }\right)
\end{equation}
The function~$g_0(t,k) $ is the characteristic function of $I(t,x_0,x_1)$ with $(x_0,x_1)\sim\nu$. From~\eqref{eq:gt:k}, we have
\begin{equation}
    \label{eq:bound:g}
    |g(t,k)| = |g_0(t,k)|  e^{-\tfrac12 \gamma^2(t) |k|^2  } \le e^{-\tfrac12 \gamma^2(t) |k|^2  }
\end{equation} 
Since $\gamma(t)>0$ for all $t\in(0,1)$ by assumption, this shows that 
\begin{equation}
    \label{eq:int:g}
    \forall p \in \NN\ \ \text{and} \ \ t\in(0,1) \quad : \quad \int_{\RR^d} |k|^p |g(t,k)|dk<\infty,
\end{equation} 
implying that $\rho(t,\cdot)$ is in $C^p(\RR^d)$ for any $p\in \NN$ and all $t\in (0,1)$.  
From~\eqref{eq:gt:k}, we also have
\begin{equation}
    \label{eq:bound:dtg}
    \begin{aligned}
        |\partial_t g(t,k)|^2 & = \left|\EE[ (ik\cdot \partial_t I_t(x_0,x_1) - \gamma(t) \dot \gamma(t) |k|^2)  e^{ik\cdot I_t(x_0,x_1)}]\right|^2 e^{-\gamma^2(t) |k|^2 }  \\ 
        &\le 2\left( |k|^2 \EE\big[|\partial_t I_t(x_0,x_1)|^2] + |\gamma(t) \dot \gamma(t)|^2 |k|^4\right) e^{-\gamma^2(t) |k|^2 }\\
        &\le 2\left( |k|^2 M_1 + 4|\gamma(t) \dot \gamma(t)|^2 |k|^4\right) e^{-\gamma^2(t)  |k|^2 }
    \end{aligned}
\end{equation}
and 
\begin{equation}
    \label{eq:bound:dttg}
    \begin{aligned}
        |\partial^2_t g(t,k)|^2 & \le 4\left( |k|^2 \EE\big[|\partial^2_t I_t(x_0,x_1)|^2] +(|\dot \gamma(t)|^2 + \gamma(t) \ddot \gamma(t))^2 |k|^4\right) e^{-\gamma^2(t) |k|^2 }\\
        & \quad + 8\left( |k|^2 \EE\big[|\partial_t I_t(x_0,x_1)|^4] +  (\gamma(t) \dot \gamma(t))^4  |k|^8\right) e^{-\gamma^2(t)  |k|^2 }  \\ 
        &\le 4\left( |k|^2 M_2 +(|\dot \gamma(t)|^2 + \gamma(t) \ddot \gamma(t))^2 |k|^4\right) e^{-\gamma^2(t) |k|^2 }\\
        & \quad + 8\left( |k|^2 M_1 +  (\gamma(t) \dot \gamma(t))^4  |k|^8\right) e^{-\gamma^2(t)  |k|^2 } 
    \end{aligned}
\end{equation}
where in both cases we used~\eqref{eq:It:L2} in Assumption~\ref{as:rho:I} to get the last inequalities. These imply that
\begin{equation}
    \label{eq:int:dg}
    \forall p \in \NN\ \ \text{and} \ \ t\in(0,1) \quad : \quad \int_{\RR^d} |k|^p |\partial_t g(t,k)|dk<\infty; \quad \int_{\RR^d} |k|^p |\partial^2_t g(t,k)|dk<\infty
\end{equation} 
indicating that $\partial_t \rho(t,\cdot)$ and $\partial^2_t \rho(t,\cdot)$ are in $C^p(\RR^d)$ for any $p\in \NN$, i.e. $\rho\in C^1((0,1); C^p(\RR^d))$ as claimed. To show that $\rho$ is also positive, denote by $\mu_0(t,dx)$ the unique (by the Fourier inversion theorem) probability measure associated with $g_0(t,k)$, i.e. the measure such that
\begin{equation}
    \label{eq:meas}
    g_0(t,k) = \int_{\RR^d} e^{ik\cdot x} \mu_0(t,dx).
\end{equation}
From~\eqref{eq:gt:k} and the convolution theorem it follows that we can express $\rho$ as
\begin{equation}
    \label{eq:rhot:invF}
    \rho(t,x) = \int_{\RR^d } \frac{e^{-|x-y|^2/(2\gamma^2(t))}}{(2\pi \gamma^2(t))^{d/2}} \mu_0(t,dy),
\end{equation}
This shows that  $\rho>0$ for all $(t,x)\in (0,1)\times \RR^d$. Since $x_{t=0} = x_0$ and  $x_{t=1} = x_1$ by definition of the interpolant, we also have $\rho(0) = \rho_0$ and $\rho(1)=\rho_1$, which shows that $\rho$ is also positive and in $C^p(\RR^d)$  at $t=0,1$ by Assumption~\ref{as:rho:I}. Note that since $\rho \in C^1((0,1); C^p(\RR^d))$ and is positive, we also immediately deduce that $s = \nabla \log \rho = \nabla \rho /\rho \in C^1((0,1); (C^p(\RR^d))^d)$.

To show that $\rho$ satisfies the TE~\eqref{eq:transport}, we take the time derivative of~\eqref{eq:charact} to deduce that
\begin{equation}
    \label{eq:ito2}
    \partial_t g(t,k) = ik\cdot m(t,k)
\end{equation}
where $m: [0,1]\times \RR^d\to\CC^d$ is the vector-valued function defined as
\begin{equation}
    \label{eq:F12k}
     m(t,k) = \EE\left(( \partial_t I(t, x_0,x_1) +\dot \gamma(t) z)e^{ik\cdot x_t} \right).
\end{equation}
By definition of the conditional expectation, $m(t,k)$ can be expressed as
\begin{equation}
    \label{eq:F1k:c}
    \begin{aligned}
     m(t,k) & = \int_{\RR^d} \EE\left(( \partial_t I(t, x_0,x_1) +\dot \gamma(t) z) e^{ik\cdot x_t} |x_t=x \right) \rho(t,x) dx\\
     & = \int_{\RR^d} e^{ik\cdot x} \EE\left(( \partial_t I(t, x_0,x_1) +\dot \gamma(t) z) |x_t=x \right) \rho(t,x) dx\\
     & = \int_{\RR^d} e^{ik\cdot x} b(t,x) \rho(t,x) dx
     \end{aligned}
\end{equation}
where the last equality follows from the definition of $b$ in~\eqref{eq:b:ode:def}. Inserting~\eqref{eq:F1k:c} in~\eqref{eq:ito2}, we deduce that this equation can be written in real space as the TE~\eqref{eq:transport}.

Let us now investigate the regularity of $b$. To that end, we go back to $m$ and use the independence  between $x_0$, $x_1$, and $z$, as well as Gaussian integration by parts to deduce that
\begin{equation}
    \label{eq:F1:1}
    m(t,k) = \EE\left((\partial_t I(t,x_0,x_1) -i \gamma(t) \dot\gamma(t) k) e^{ik\cdot I(t,x_1,x_0)}\right) e^{-\tfrac12 \gamma^2(t) |k|^2 },
\end{equation}
As a result
\begin{equation}
    \label{eq:F1:int}
    \begin{aligned}
        |m(t,k)|^2 &= \left|\EE\left((\partial_t I(t,x_0,x_1) -i \gamma(t) \dot\gamma(t) k) e^{ik\cdot I(t,x_1,x_0)}\right)\right|^2 e^{-\gamma^2(t)  |k|^2 } \\
        & \le  2\left(\EE\big[|\partial_t I(t,x_0,x_1)|^2\big] + |\gamma(t) \dot \gamma(t)|^2 |k|^2\right) e^{-\gamma^2(t)  |k|^2 } \\
        & \le 2M_1 e^{-\gamma^2(t)  |k|^2 },
    \end{aligned}
\end{equation}
and
\begin{equation}
    \label{eq:F1:dint}
    \begin{aligned}
        |\partial_t m(t,k)|^2
        &\le 4\left(\EE\big[|\partial^2_t I(t,x_0,x_1)|^2 + (\gamma(t) \ddot \gamma(t) + \dot \gamma^2(t))^2 \right) e^{-\gamma^2(t) |k|^2 } \\
        & \quad + 8|k|^2 \left( \EE \big |\partial_t I(t,x_0,x_1)|^4\big] + (\gamma(t) \dot \gamma(t))^4 |k|^4 \right) e^{-\gamma^2(t) |k|^2 } \\
        & \le 4\left(M_1 + (\gamma(t) \ddot \gamma(t) + \dot \gamma^2(t))^2 \right) e^{-\gamma^2(t) |k|^2 } \\
        & \quad + 8|k|^2 \left( M_2 + (\gamma(t) \dot \gamma(t))^4 |k|^4 \right) e^{-\gamma^2(t) |k|^2 } ,
    \end{aligned}
\end{equation}
where in both cases the last inequalities follow from~\eqref{eq:It:L2}. Therefore 
\begin{equation}
    \label{eq:int:f}
    \forall p \in \NN \ \ \text{and} \ \ t\in(0,1) \quad : \quad \int_{\RR^d} |k|^p |m(t,k)|dk<\infty, \quad \int_{\RR^d} |k|^p |\partial_t m(t,k)|dk<\infty,
\end{equation} 
which implies that the inverse Fourier transform of $m$ is a function $j : [0,1]\times \RR^d \to \RR^d$ that satisfies $j(t, \cdot) \in (C^p(\RR^d))^d$ for any $p\in \NN$ and any $t\in(0, 1)$ and can be expressed as
\begin{equation}
    \label{eq:j1}
    j(t,x) = (2\pi)^{-d} \int_{\RR^d} e^{-ik\cdot x} m(t,k) dk = \EE\left(\partial_t I(t,x_0,x_1)+ \dot \gamma (t) z |x_t = x\right) \rho(t,x) \equiv b(t,x) \rho_t(x)
\end{equation}
where the last equality follows from the definition of $b$ in~\eqref{eq:b:ode:def}. We deduce that $b\in C^0([0,1];(C^p(\RR^d))^d)$ for any $p\in \NN$ since $j\in C^0([0,1];(C^p(\RR^d))^d)$ and $\rho\in C^1([0,1];C^p(\RR^d))$, and $\rho>0$.

Finally, let us establish~\eqref{eq:bt:bounded}. By~\eqref{eq:It:L2} we have
\begin{equation}
    \label{eq:v:L2}
    \begin{aligned}
        \int_{\RR^d} |b(t,x)|^2 \rho(t,x) dx & = \int_{\RR^d}|\EE\left(\partial_t I(t,x_0,x_1)+\dot \gamma(t) z|x_t = x\right)|^2 \rho(t,x)dx \\
        &\le 2 \int_{\RR^d} \EE\left(|\partial_t I(t,x_0,x_1)|^2+|\dot \gamma(t)|^2 |z|^2|x_t = x\right) \rho(t,x)dx\\
        &  \le 2 \EE\big[|\partial_t I(t,x_0,x_1)|^2 +|\dot \gamma(t)|^2|z|^2\big]\\
        & < 2M_1^{1/2} + 2 d |\dot \gamma(t)|^2,
    \end{aligned} 
\end{equation} 
so that this integral is bounded for all $t\in (0,1)$. To analyze its behavior at the end points, notice that the decomposition \eqref{eq:b:decomp} implies that 
\begin{equation}
\label{eq:b:lim:o=0:1}
\begin{aligned}
    b_0(x) &\equiv \lim_{t\to0} b(t,x) = \EE_1 [\partial_t I(0,x,x_1) ] - \lim_{t\to0} \dot \gamma(t) \gamma(t) s_0(x), \\
    b_1(x) & \equiv \lim_{t\to1} b(t,x) = \EE_0 [\partial_t I(0,x_0,x) ] - \lim_{t\to1} \dot \gamma(t) \gamma(t) s_1(x), 
\end{aligned}
\end{equation}
where $s_0= \nabla \log\rho_0$, $s_1= \nabla \log\rho_1$, $\EE_0$ and $\EE_1$ denote expectations over $x_0\sim\rho_0$ and $x_1\sim \rho_1$, respectively, and we used the property that $x_{t=0}=x_0$ and $x_{t=1} = x_1$. Since $\lim_{t\to0,1} \dot \gamma(t) \gamma(t)$ exists by our assumption that $\gamma^2\in C^1([0,1])$, $b_0$ and $b_1$ are well defined, and 
\begin{equation}
    \label{eq:b0:1:int}
    \int_{\RR^d} |b_0(x)|^2 \rho_0(x) dx < \infty, \qquad \int_{\RR^d} |b_1(x)|^2 \rho_1(x) dx < \infty.
\end{equation}
by Assumption~\ref{as:rho:I}. As a result, the integral in~\eqref{eq:v:L2} is continuous at $t=0$ and $t=1$, so it must be integrable on $[0,1]$, and~\eqref{eq:bt:bounded} holds. 
\end{proof}

\interpolatelosses*
\begin{proof}
    By definition of $\rho$, the objective $\mathcal{L}_b$ defined in~\eqref{eq:obj:v} can also be written as 
\begin{equation}
    \label{eq:EL:v}
    \begin{aligned}
    \mathcal{L}_b[\hat b] &= \int_0^1\int_{\RR^d} \left( \tfrac12|\hat b(t,x)|^2 - \EE\left((\partial_t I(t,x_0,x_1)+\dot\gamma(t) z|x_t=x\right)\cdot \hat b(t,x)\right) \rho(t,x) dxdt\\
    &= \int_0^1\int_{\RR^d} \left( \tfrac12 |\hat b(t,x)|^2 - b(t,x)\cdot \hat b(t,x) \right) \rho(t,x) dxdt
    \end{aligned}
\end{equation}
where we used the definition of $b$ in~\eqref{eq:b:ode:def}. This quadratic objective is bounded from below since
\begin{equation}
    \label{eq:EL:v:b}
    \begin{aligned}
    \mathcal{L}_b[\hat b] &= \tfrac12 \int_0^1\int_{\RR^d} \left|\hat b(t,x) - b(t,x) \right|^2 \rho(t,x) dx dt- \tfrac12  \int_0^1\int_{\RR^d} \left|b(t,x) \right|^2 \rho(t,x) dxdt\\
    & \ge - \tfrac12 \int_0^1 \int_{\RR^d} \left|b(t,x) \right|^2 \rho(t,x) dxdt >-\infty
    \end{aligned}
\end{equation}
where the last inequality follows from~\eqref{eq:bt:bounded}. Since $\rho_t$ is positive the  minimizer of~\eqref{eq:EL:v} is unique and given by $\hat b= b$. 

\end{proof}

\score*
\begin{proof}
Since $\rho \in C^1((0,1); C^p(\RR^d))$ and is positive by Theorem~\ref{prop:interpolate}, we already  know that $s = \nabla \log \rho = \nabla \rho /\rho \in C^1((0,1); (C^p(\RR^d))^d)$.
To establish~\eqref{eq:s:def}, note that, for $t\in(0,1)$ where $\gamma(t)>0$, we have
 \begin{equation}
 \label{eq:gbp}
 \EE\left(z  e^{i\gamma(t) k \cdot z }\right) = -\gamma^{-1}(t)(i\partial_k)  \EE e^{i\gamma(t) k \cdot z } = -\gamma^{-1}(t)(i\partial_k)  e^{-\tfrac12 \gamma^2(t) |k|^2} = i \gamma(t)k e^{-\tfrac12 \gamma^2(t) |k|^2}.
 \end{equation}
 As a result, using the independence between $x_0$, $x_1$, and $z$, we have
\begin{equation}
 \label{eq:gbp:2}
 \EE\left(z  e^{i k \cdot x_t }\right) = i \gamma(t)k g(t,k)
 \end{equation}
 where $g$ is the characteristic function of $x_t$ defined in~\eqref{eq:charact}. Using the properties of the conditional expectation, the left-hand side of this equation  can be written as
\begin{equation}
 \label{eq:gbp:3}
 \EE\left(z  e^{i k \cdot x_t }\right) = \int_{\RR^d} \EE\left(z  e^{i k \cdot x_t }|x_t=x\right) \rho(t,x) dx = \int_{\RR^d} \EE\left(z  |x_t=x\right) e^{i k x_t } \rho(t,x) dx 
 \end{equation}
 Since the left hand side of~\eqref{eq:gbp:2} is the Fourier transform of $-\gamma(t) \nabla \rho(t,x)$, we deduce that
 \begin{equation}
 \label{eq:gbp:4}
 \EE\left(z  |x_t=x\right)  \rho(t,x) = -\gamma(t) \nabla \rho(t,x) = -\gamma(t) s(t,x) \rho(t,x).
 \end{equation}
 Since $\rho(t,x)>0$, this implies~\eqref{eq:s:def} for $t\in(0,1)$ where $\gamma(t)>0$.

To establish~\eqref{eq:st:bounded}, notice that
\begin{equation}
    \label{eq:s:L2}
    \begin{aligned}
        \int_{\RR^d} |s(t,x)|^2 \rho(t,x) dx & = \int_{\RR^d}|\EE\left((\gamma^{-1}(t) z|x_t = x\right)|^2 \rho(t,x)dx \\
        &\le \int_{\RR^d} \gamma^{-2}(t) \EE\left(|z|^2|x_t = x\right) \rho(t,x)dx\\
        &  \le d\gamma^{-2}(t).
    \end{aligned} 
\end{equation} 
This means that this integral is bounded for all $t\in(0,1)$. Since the integral is also continuous at $t=0$ and $t=1$, with values given by~\eqref{eq:rho0:1:sc}, it must be integrable on $[0,1]$ and~\eqref{eq:st:bounded} holds.

The objective $\mathcal{L}_s$ defined in~\eqref{eq:obj:s}  
can also be written as 
\begin{equation}
    \label{eq:EL:s}
    \begin{aligned}
     \mathcal{L}_s[\hat s] &= \int_0^1\int_{\RR^d} \left( \tfrac12|\hat s(t,x)|^2 + \gamma^{-1}(t)\EE\left(z|x_t=x\right)\cdot \hat s(t,x)\right) \rho(t,x) dxdt\\
    &= \int_0^1\int_{\RR^d} \left( \tfrac12 |\hat s(t,x)|^2 - s(t,x)\cdot \hat s(t,x) \right) \rho(t,x) dxdt,
    \end{aligned}
\end{equation}
where we used the definition of $s$ in~\eqref{eq:s:def}. This quadratic objective is bounded from below since
\begin{equation}
    \label{eq:EL:s:b}
    \begin{aligned}
    \mathcal{L}_s[\hat s] &= \tfrac12 \int_0^1\int_{\RR^d} \left|\hat s(t,x) - s(t,x) \right|^2 \rho(t,x) dxdt - \tfrac12  \int_0^1\int_{\RR^d} \left|s(t,x) \right|^2 \rho(t,x) dxdt\\
    & \ge - \tfrac12  \int_0^1\int_{\RR^d} \left|s(t,x) \right|^2 \rho(t,x) dxdt >-\infty
    \end{aligned}
\end{equation}
where the last inequality follows from~\eqref{eq:st:bounded}. Since $\rho$ is positive the  minimizer of \eqref{eq:EL:s} is unique and given by $\hat s = s$. 
\end{proof}

\interpolationfpe*
\begin{proof}
The forward FPE~~\eqref{eq:fpe} and the backward FPE~\eqref{eq:fpe:tr} are direct consequences of the TE~\eqref{eq:transport} and~\eqref{eq:s:def}, since the equality
\begin{equation}
    \label{eq:score}
\eps(t)\Delta \rho =  \eps(t)\nabla \cdot (\rho \nabla \log \rho ) = \eps(t)\nabla \cdot (s \rho ) 
\end{equation}
 can be used to convert between these equations.

\end{proof}

\subsection{Proof of Lemma~\ref{lem:reversed}}
\label{app:proof:generative}

\reversed*
\begin{proof}
    The SDE~\eqref{eq:sde:1} and the ODE~\eqref{eq:ode:1} are the evolution equations for the processes whose densities solve~\eqref{eq:fpe} and~\eqref{eq:transport}, respectively. The equation that requires some explanation is the backwards SDE~\eqref{eq:sde:R}, which can be solved backwards in time from $t=1$ to $t=0$. As discussed in the main text, by definition, its solution is $X^\rev_{t}=Z^\fwd_{1-t}$ where $Z^\fwd_t$ solves the forward SDE
\begin{equation}
    \label{eq:sde:generic:rev:Y}
    dZ^\fwd_t = -b_\rev(1-t,Z^\fwd_t)dt + \seps d W_t.
\end{equation}
To see how to write the backward It\^o formula~\eqref{eq:ito:formula}  note that given any  $f\in C^1([0,1];C_0^2(\RR^d))$, we have 
\begin{equation}
    \label{eq:Ito:rev:Y}
    \begin{aligned}
    df(1-t,Z^\fwd_t) &= -\partial_t f(1-t,Z^\fwd_t)dt+\nabla f(1-t,Z^\fwd_t) \cdot dZ^\fwd_t + \eps \Delta f(1-t,Z^\fwd_t) dt\\
    &= -\partial_t f(1-t,Z^\fwd_t)dt+\left( - b_\rev(1-t,Z^\fwd_t) \cdot \nabla f(1-t,Z^\fwd_t)  + \eps \Delta f(1-t,Z^\fwd_t)\right) dt\\
    & \quad+ \seps \nabla f(1-t,Z^\fwd_t) \cdot dW_t
     \end{aligned}
\end{equation}
In integral form, this equation can be written as
\begin{equation}
    \label{eq:int:Ito:rev:Y}
    \begin{aligned}
    f(1,Z^\fwd_{1}) &= f(1-t,Z^\fwd_{1-t}) \\
    &\quad- \int_{1-t}^1 \left( \partial_t f(1-s,Z^\fwd_s)+ b_\rev(1-s,Z^\fwd_s) \cdot \nabla f(1-s,Z^\fwd_s)  - \eps \Delta f(1-s,Z^\fwd_s)\right) ds \\
    & \quad - \seps \int_{1-t}^1 \nabla f(1-s,Z^\fwd_s) \cdot dW_s. 
     \end{aligned}
\end{equation}
Using $X^\rev_t = Z^\fwd_{1-t}$ and $W^\rev_t = - W_{1-t}$ and changing integration variable from $s$ to $1-s$, this is 
\begin{equation}
    \label{eq:int:Ito:rev:X}
    \begin{aligned}
    f(1,X^\rev_0) & = f(1-t,X^\rev_t) - \int_t^1 \left(\partial_t f(s,X^\rev_s) + b_\rev(s,X^\rev_s) \cdot \nabla f(s,X^\rev_s)  - \eps \Delta f(s,X^\rev_s)\right) ds \\
    & \quad - \seps \int_t^1 \nabla f(s,X^\rev_s) \cdot dW^\rev_s, 
     \end{aligned}
\end{equation}
In differential form, this is equivalent to saying that 
\begin{equation}
    \label{eq:Ito:rev}
    \begin{aligned}
    df(t,X^\rev_t) &= \partial_t f(t,X^\rev_t)dt + \nabla f(t,X^\rev_t) \cdot dX^\rev_t - \eps \Delta f(X^\rev_t) dt\\
    &= \left( \partial_t f(t,X^\rev_t) +b_\rev(t,X^\rev_t) \cdot \nabla f(t,X^\rev_t)  - \eps \Delta f(t,X^\rev_t)\right) dt + \seps \nabla f(t,X^\rev_t) \cdot dW_t
     \end{aligned}
\end{equation}
which is the backward It\^o formula~\eqref{eq:ito:formula}. Similarly, by the It\^o isometries we have: for any $g\in C^0([0,1];(C_0^0(\RR_d))^d)$ and $t\in[0,1]$
\begin{equation}
\label{eq:ito:iso:Y}
\EE^x \int_{1-t}^1 g(s,Z^\fwd_s) \cdot dW_s = 0, \qquad \EE^x \left|\int_{1-t}^1 g(s,Z^\fwd_s) \cdot dW_s\right|^2 =   \int_{1-t}^1 \EE^x|g (s,Z^\fwd_s) |^2ds.
\end{equation}
Written in terms of $X^\rev_t$, these are~\eqref{eq:ito:iso}.
\end{proof}

\paragraph{Derivation of \eqref{eq:ob:w:alt}.} For any $t\in[0,1]$ the score is the minimizer of
\begin{equation}
    \begin{aligned}
        &\int_{\RR^d} |\hat s(t,x) - \nabla \log \rho(t,x)|^2 \rho(t,x) dx\\
        &= \int_{\RR^d} \left(|\hat s(t,x)|^2 - 2\hat s(t,x) \cdot\nabla\log \rho(t,x)+ | \nabla \log \rho(t,x)|^2\right) \rho(t,x) dx \\
        &= \int_{\RR^d} \left(|\hat s(t,x)|^2 + 2\nabla \cdot \hat s(t,x) + | \nabla \log \rho(t,x)|^2\right) \rho(t,x) dx,
    \end{aligned}
\end{equation}
where we used the identity $\hat s \cdot \nabla \log \rho \,\rho = \hat s \cdot \nabla \rho$ and integration by parts to obtain the second equality. The last term involving $|\nabla \log \rho|^2$ is a constant in $\hat s$ that can be neglected for optimization. Expressing the remaining terms as an expectation over $x_t$ and integrating the result in time gives~\eqref{eq:ob:w:alt}.

\subsection{Proofs of Lemmas~\ref{lemma:kl_transport} and~\ref{lemma:kl_fpe}, and Theorem~\ref{thm:kl:bound}.}
\label{app:proof:kl}
\kltransport*
\begin{proof}
Using~\eqref{eq:2:te}, we compute analytically
\begin{align*}
    \frac{d}{dt}\KL{\rho(t)}{\hat\rho(t)} &= \frac{d}{dt}\int_{\RR^d} \log\left(\frac{\rho}{\hat \rho}\right)\rho dx\\
    &= \int_{\RR^d} \hat\rho\left(\frac{\partial_t \rho}{\hat\rho} - \frac{\rho}{(\hat\rho)^2}\partial_t\hat\rho\right)dx + \int\log\left(\frac{\rho}{\hat\rho}\right)\partial_t\rho dx\\
    &= -\int_{\RR^d} \left(\frac{\rho}{\hat\rho}\right)\partial_t\hat\rho dx + \int_{\RR^d}\log\left(\frac{\rho}{\hat\rho}\right)\partial_t\rho dx\\
    &= \int_{\RR^d} \left(\frac{\rho}{\hat\rho}\right)\nabla\cdot\left(\hat b\hat\rho\right)dx - \int\log\left(\frac{\rho}{\hat\rho}\right)\nabla\cdot\left(b\rho\right)dx\\
    &= -\int_{\RR^d} \nabla\left(\frac{\rho}{\hat\rho}\right)\cdot \hat b\hat\rho dx + \int\left(\nabla\log\rho - \nabla\log\hat\rho\right)\cdot b \rho dx\\
    &= -\int_{\RR^d} \left(\frac{\nabla\rho}{\hat\rho} - \frac{\rho\nabla\hat\rho}{\hat\rho^2}\right)\cdot \hat b\hat\rho dx 
    + \int_{\RR^d}\left(\nabla\log\rho - \nabla\log\hat\rho\right)\cdot b \rho dx\\
    &= \int_{\RR^d} \left(\nabla\log\hat\rho - \nabla\log\rho\right)\cdot \hat b \rho dx + \int_{\RR^d}\left(\nabla\log\rho - \nabla\log\hat\rho\right)\cdot b \rho dx\\
    &= \int_{\RR^d} \left(\nabla\log\hat\rho - \nabla\log\rho\right)\cdot (\hat b - b) \rho dx.
\end{align*}
where we omitted the argument $(t,x)$ of all functions for simplicity of notation. Integrating both sides from $0$ to $1$ completes the proof.
\end{proof}

\klfpe*
\begin{proof}
Similar to the proof of Lemma~\ref{lemma:kl_transport}, we can use the FPE in~\eqref{eq:2:fpe} to compute $\frac{d}{dt}\KL{\rho(t)}{\hat\rho(t)}$, which leads to the main result. Instead, we take a simpler approach, leveraging the result in Lemma~\ref{lemma:kl_transport}. We re-write the Fokker-Planck equations in~\eqref{eq:2:fpe} as the (score-dependent) transport equations
\begin{equation}
    \label{eqn:sde_transports}
    \begin{aligned}
        \partial_t \rho &= -\nabla\cdot((b_\fwd - \eps\nabla\log\rho)\rho),\\ 
        \partial_t \hat\rho &= -\nabla\cdot((\hat b_\fwd - \eps\nabla\log\hat\rho)\hat\rho).
    \end{aligned}
\end{equation}
Applying Lemma~\ref{lemma:kl_transport} directly, we find that
\begin{equation}
\begin{aligned}
    \KL{\rho(1)}{\hat \rho(1)} = \int_0^1 \int_{\RR^d} \left[\left(\nabla\log\hat\rho  - \nabla\log\rho \right)\cdot \left(\left[\hat b_\fwd - \eps\nabla\log\hat\rho \right] - \left[b_\fwd - \eps\nabla\log\rho \right]\right)\right] \rho dx dt.
\end{aligned}
\end{equation}
Expanding the above,
\begin{equation}
\begin{aligned}
    \KL{\rho(1) }{\hat\rho(1) } &= \int_0^1 \int_{\RR^d}\left[\left(\nabla\log\hat\rho  - \nabla\log\rho \right)\cdot (\hat b_\fwd - b_\fwd)\right]\rho dx dt\\
    &\qquad  - \eps\int_0^1 \int_{\RR^d} \left[\norm{\nabla\log\rho  - \nabla\log\hat\rho }^2\right]\rho dx dt,
\end{aligned}
\end{equation}
which proves the first part of the lemma. Now, by Young's inequality, it holds for any fixed $\eta > 0$ that
\begin{align}
    \KL{\rho_1}{\hat\rho(1)} &\leq \frac{1}{2\eta}\int_0^1 \int_{\RR^d}\norm{\hat{b}_\fwd - b_\fwd}^2\rho_tdx dt\nonumber\\
    &\qquad  + \left(\tfrac{1}{2}\eta - \eps\right)\int_0^1 \int_{\RR^d} \norm{\nabla\log\rho_t - \nabla\log\hat{\rho}_t}^2\rho_tdxdt.
\end{align}
Hence, for $\eta =2\eps$, 
\begin{align}
    \label{eqn:kl_step}
    \KL{\rho_1}{\hat\rho(1)} &\leq \frac{1}{4\eps} \int_0^1 \int_{\RR^d}\norm{\hat{b}_\fwd - b_\fwd}^2\rho dxdt,
\end{align}
which proves the second part of the lemma.
\end{proof}

\likelihoodbound*
\begin{proof}
Observe that by Proposition~\ref{prop:generative}, the target density $\rho_1 =\rho(1, \cdot)$ is the density of the process $X_t$ that evolves according to SDE~\eqref{eq:sde:1}. By Lemma~\ref{lemma:kl_fpe} and an additional application of Young's inequality, we then have that 
\begin{align}
    \KL{\rho_1}{\hat\rho(1)} &\leq \frac{1}{2\eps}\int_0^1\int_{\RR^d}\left(|\hat{b} - b|^2 +\epsilon^2 \norm{s - \hat{s}}^2\right)\rho dxdt.
\end{align}
Using the definition of $\mathcal{L}_b[\hat{b}]$ and $\mathcal{L}_s[\hat{s}] $ in~\eqref{eq:obj:v}, and~\eqref{eq:obj:s}, we conclude that
\begin{equation}
    \KL{\rho_1}{\hat{\rho}(1)} \leq \frac{1}{2\eps}\left(\mathcal{L}_b[\hat{b}] - \min_{\hat{b}}\mathcal{L}_b[\hat{b}]\right)  + \frac{\eps}{2}\left(\mathcal{L}_s[\hat{s}] - \min_{\hat{s}}\mathcal{L}_s[\hat{s}]\right),
\end{equation}
which is~\eqref{eq:bound:kl}. Using that $b = v - \gamma\dot{\gamma} s$ and $\hat{b} = \hat{v} - \gamma\dot{\gamma}\hat{s}$, we can write~\eqref{eqn:kl_step} in this case as 
\begin{align}
    \KL{\rho_1}{\hat\rho(1)} &\leq \frac{1}{4\eps} \int_0^1 \int_{\RR^d}\norm{\hat{v} - v - (\gamma(t)\dot\gamma(t) - \epsilon)(\hat{s} - s)}^2\rho dxdt,\\
    &\leq \frac{1}{2\eps} \int_0^1 \int_{\RR^d}\left(\norm{\hat{v} - v}^2 + (\gamma(t)\dot\gamma(t) - \epsilon)^2\norm{\hat{s} - s}^2\right)\rho dxdt,\\
    &\leq \frac{1}{2\eps}\left(\mathcal{L}_{v}[\hat{v}] - \min_{\hat{v}}\mathcal{L}_{v}[\hat{v}]\right) + \frac{\max_{t\in [0 ,1]}(\gamma(t)\dot\gamma(t) - \epsilon)^2}{2\eps}\left(\mathcal{L}_{s}[\hat{s}] - \min_{\hat{v}}\mathcal{L}_{s}[\hat{s}]\right),
\end{align}
which is the final part of the theorem.
\end{proof}

\subsection{Proofs of Lemma~\ref{lem:tesol} and Theorem~\ref{prop:sde:rho}}
\label{app:prrof:de}

\TEs*
\begin{proof}
    If $\hat \rho$ solves the TE~\eqref{eq:TE:hat} and $ X_{s,t}$ solves the ODE~\eqref{eq:ode:st}, we have
\begin{equation}
    \label{eq:te:sol:1}
    \begin{aligned}
        \frac{d}{dt} \hat\rho(t,X_{s,t}(x)) & = \partial_t  \hat\rho(t, X_{s,t}(x))  + b_\ODE(t,X_{s,t}(x))  \cdot \nabla  \hat\rho(t,X_{s,t}(x))\\
        &= - \nabla \cdot b_\ODE(t, X_{s,t}(x)) \hat\rho(t,X_{s,t}(x))
    \end{aligned}
\end{equation}
This equation implies that
\begin{equation}
    \label{eq:te:sol:2}
    \begin{aligned}
        \frac{d}{dt} \left( \exp\left( \int_{s}^t\nabla \cdot b_\ODE(\tau,X_{s,\tau}(x)) d\tau \right) \hat\rho(t,X_{s,t}(x)) \right) = 0.
    \end{aligned}
\end{equation}
Integrating~\eqref{eq:te:sol:2} over $[0,t]$ and setting $s=t$ in the result gives
\begin{equation}
    \label{eq:te:sol:3:f}
    \begin{aligned}
          \hat\rho(t,x)  = \exp\left( -\int_0^{t}\nabla \cdot b_\ODE(\tau,X_{t,\tau}(x)) d\tau \right) \hat\rho(0,X_{t,0}(x))
    \end{aligned}
\end{equation}
If we use the initial condition $\hat \rho(0) = \rho_0$, this gives~\eqref{eq:TE:hat:s:f}.
Similarly, Integrating ~\eqref{eq:te:sol:2}  on $[t,1]$ and setting $s=t $ in the result gives
\begin{equation}
    \label{eq:te:sol:3:r}
    \begin{aligned}
         \hat\rho(t,x)  = \exp\left( \int_{t}^1\nabla \cdot b_\ODE(\tau,X_{t,\tau}(x)) d\tau \right) \hat\rho(1,X_{t,1}(x))
    \end{aligned}
\end{equation}
If we use the final condition $\hat \rho(1) = \rho_1$, this gives~\eqref{eq:TE:hat:s:b}.

\end{proof}

\FK*
\begin{proof}
Evaluating $d\hat \rho_\fwd(t, Y^\rev_t)$ via the backward It\^o formula~\eqref{eq:ito:formula} we obtain
\begin{equation}
    \label{eq:ito:rho:1}
    \begin{aligned}
    d\hat\rho_\fwd(t, Y_t^\rev) &= \partial_t \hat \rho_\fwd(t, Y_t^\rev) dt + \nabla \hat \rho_\fwd(t, Y_t^\rev) \cdot dY^\rev_t - \eps \Delta \hat \rho_\fwd(t,Y_t^\rev) dt\\
    &= \partial_t \hat \rho_\fwd(t, Y_t^\rev) dt + \nabla \hat \rho_\fwd(t, Y_t^\rev) \cdot \hat b_\fwd(t, Y^\rev_t)dt + \seps \nabla \hat \rho_\fwd(t, Y_t^\rev) \cdot dW^\rev_t - \eps \Delta \hat \rho_\fwd(t, Y_t^\rev) dt\\
    & = -\nabla \cdot \hat b_\fwd(t, Y^\rev_t)  \hat \rho_\fwd(t, Y^\rev_t) dt + \seps \nabla \hat\rho_\fwd(t, Y_t^\rev) \cdot dW^\rev_t.
    \end{aligned}
\end{equation}
where we used~\eqref{eq:sde:y:R} in the second step and~\eqref{eq:fpe:f:hat} in the last one. This equation can be written as a total differential in the form
\begin{equation}
    \label{eq:ito:rho:2}
    \begin{aligned}
    & d \left(\exp\left(- \int_t^1\nabla \cdot \hat b_\fwd(\tau, Y^\rev_\tau) d\tau\right)  \hat \rho_\fwd(t, Y_t^\rev) \right) \\
    & =  \seps \exp\left(-\int_t^1\nabla \cdot \hat b_\fwd(\tau, Y^\rev_\tau) d\tau\right) \nabla \hat \rho_\fwd(t, Y_t^\rev) \cdot dW^\rev_t,
    \end{aligned}
\end{equation}
which after integration over $t\in[0,1]$ becomes
\begin{equation}
    \label{eq:ito:rho:3}
    \begin{aligned}
    & \hat \rho_\fwd(1, Y_{t=1}^\rev) - \exp\left(- \int_0^1\nabla \cdot \hat b_\fwd(t, Y^\rev_t)  dt\right)  \rho_0(Y^\rev_{t=0}) \\
    & =  \seps \int_0^1 \exp\left(-\int_t^1\nabla \cdot \hat b_\fwd(\tau, Y^\rev_\tau) d\tau\right) \nabla \hat \rho(t, Y_t^\rev) \cdot dW^\rev_t.
    \end{aligned}
\end{equation}
where we used $\hat \rho(0)=\rho_0$.
Taking an expectation conditioned on the event $Y_{t=1}^\rev=x$ and using that the term on the right-hand side has mean zero, we find that
\begin{equation}
    \label{eq:ito:rho:4}
    \hat \rho_\fwd(1, Y_{t=1}^\rev) - \EE^x_\rev\exp\left(- \int_0^1\nabla \cdot \hat b_\fwd(t, Y^\rev_t)  dt\right)  \rho_0(Y^\rev_{t=0}) = 0.
\end{equation}
This gives~\eqref{eq:fk}.

Similarly, evaluating $d\hat \rho_\rev(t, Y^\fwd_t)$ via It\^o's formula, we obtain
\begin{equation}
    \label{eq:ito:rho:r:1}
    \begin{aligned}
    d \hat \rho_\rev(t, Y^\fwd_t) &= \partial_t \hat \rho_\rev(t, Y^\fwd_t) dt + \nabla \hat \rho_\rev(t, Y^\fwd_t) \cdot dY^\fwd_t + \epsilon \Delta \hat \rho_\rev(t, Y^\fwd_t) dt\\
    &= \partial_t \hat \rho_\rev(t, Y^\fwd_t) dt + \nabla \hat \rho_\rev(t, Y^\fwd_t) \cdot \hat b_\rev(t, Y^\fwd_t)dt + \seps \nabla \hat \rho_\rev(t, Y^\fwd_t) \cdot dW_t + \eps \Delta \hat \rho_\rev(t, Y^\fwd_t) dt\\
    & = -\nabla \cdot \hat b_\rev(t, Y^\fwd_t)   \hat \rho_\rev(t,Y^\fwd_t) dt + \seps \nabla \hat \rho_\rev(t,Y^\fwd_t) \cdot dW_t.
    \end{aligned}
\end{equation}
where we used~\eqref{eq:sde:y:1} in the second step and~\eqref{eq:fpe:r:hat} in the last one.  This equation can be written as a total differential in the form
\begin{equation}
    \label{eq:ito:rho:r:2}
    \begin{aligned}
    & d\left(\exp\left( \int_0^t\nabla \cdot \hat b_\rev(\tau, Y^\fwd_\tau)  d\tau\right)  \hat \rho_\rev(t,Y^\fwd_t) \right) \\
    & =  \seps \exp\left( \int_0^t\nabla \cdot \hat b_\rev(\tau, Y^\fwd_\tau)  d\tau\right) \nabla \hat \rho_\rev(t, Y^\fwd_t) \cdot dW_t.
    \end{aligned}
\end{equation}
Integrating the above on $t\in[0,1]$, we find that
\begin{equation}
    \label{eq:ito:rho:r:3}
    \begin{aligned}
    & \exp\left( \int_0^1\nabla \cdot \hat b_\rev(t, Y^\fwd_t)  dt\right)\rho_1( Y^\fwd_1) -   \hat \rho_\rev(0, Y^\fwd_{t=0})\\
    &= \seps \int_0^1 \exp\left( \int_0^t\nabla \cdot \hat b_\rev(\tau, Y^\fwd_\tau)  d\tau\right) \nabla \hat \rho_\rev(t,Y^\fwd_t) \cdot dW_t.
    \end{aligned}
\end{equation}
where we used $\hat \rho(1) = \rho_1$.
Taking an expectation conditioned on the event $Y^\fwd_{t=0} = x$ and applying the It\^o isometry, we deduce that
\begin{equation}
    \label{eq:ito:rho:r:4}
    \EE_\fwd^x \left(\exp\left( \int_0^1\nabla \cdot \hat b_\rev(t, Y^\fwd_t) dt\right)\rho_1(Y^\fwd_{t=1})\right) - \hat \rho_\rev(0, y) =0.
\end{equation}
This gives~\eqref{eq:fk:rev}.

\end{proof}

\subsection{Proof of Theorem~\ref{thm:diff}}
\label{app:diff}

\diffgen* 
\begin{proof}
Let us first consider what happens on the interval $t\in [0,\delta]$ where $I(t,x_0,x_1) = x_0$ by assumption and the stochastic interpolant~\eqref{eq:bb} with $x_0$ fixed  and $a(t)=a>0$ reduces to
\begin{equation}
\label{eq:bb:delta}
x_t = x_0 + \sqrt{2at(1-t)} z \quad \text{with} \quad x_0 \text{fixed}, \  z \sim {\sf N}(0,\Id).
\end{equation} 
The law of this process is simply
\begin{equation}
\label{eq:bb:delta:2}
x_t \sim {\sf N}(x_0,2at(1-t)\Id),
\end{equation} 
which means that its density satisfies for all $t>0$ the FPE
\begin{equation}
\label{eq:bb:delta:3}
\partial_t \rho + 2a t \nabla \cdot \left( s(t,x) \rho\right ) = a \Delta \rho,
\end{equation} 
where $s(t,x) = \nabla \log \rho(t,x)$ is the score, which is explicitly given by
\begin{equation}
\label{eq:score:delta}
s(t,x) = -\frac{x-x_0}{2at(1-t)}\qquad  \text{for} \ \ t\in (0,\delta].
\end{equation} 
This means that the drift term in the FPE~\eqref{eq:bb:delta:3} can be written as
\begin{equation}
\label{eq:score:delta:2}
2a t s(t,x) = -\frac{x-x_0}{1-t}\qquad \text{for} \ \ t\in (0,\delta].
\end{equation} 
and this equality also holds in the limit as $t\to0$. It is also easy to check that 
\begin{equation}
\label{eq:score:delta:3}
-\frac{x-x_0}{2at(1-t)} = -\frac{\sqrt{2at}}{\sqrt{(1-t)}} \EE(z|x_t=x)  \qquad \text{for} \ \ t\in [0,\delta]
\end{equation} 
which is consistent with the drift in~\eqref{eq:u:def:x0} on $t\in [0,\delta]$ since $\partial_t I(t,x_0,x_1) = 0$ on this interval. Since the right-hand side of \eqref{eq:score:delta:2} is nonsingular at $t=0$  it also means that the SDE~\eqref{eq:diff:sde} is well-defined for $t\in[0,\delta]$ and the law of its solutions coincide with that of the process $x_t$ defined~\eqref{eq:bb:delta}, $X_t^\diff \sim {\sf N}(x_0,2at(1-t)\Id)$ for $t\in[0,\delta]$. 

Considering next what happens on $t\in(\delta,1]$, since $\gamma(t) = \sqrt{2at(1-t)}$ is only zero at the endpoint $t=1$ where we assume that the density $\rho_1$ satisfies Assumption~\ref{as:rho:I} we can mimick all the arguments in the proof of of Theorem~\ref{prop:interpolate} and Corollaries~\ref{prop:interpolate_fpe} and \ref{prop:generative} to terminate the proof.
\end{proof}

Note that this proof shows that is enough to have $\partial_t I(t=0,x_0,x_1) = 0$, since this implies that $I(t,x_0,x_1) = x_0 + O(t^2)$. It also shows that, while it is key to use an SDE on $t\in[0,\delta]$ so that the generative process can spread the mass away from $x_0$, the diffusive noise is no longer necessary and we could switch back to a probability flow ODE on $t\in (\delta,1]$ (using a time-dependent $a(t)$ to that effect with $a(t)=0$ for $t\in (\delta,1]$).

\subsection{Proof of Lemma~\ref{lem:interp} and Theorem~\ref{prop:sb}.}
\label{app:sb}

\interppdf*
\begin{proof}
By definition of the map~$T$ in~\eqref{eq:interpol}, if $x_0\sim \rho_0$ and $x_1\sim\rho_1$, then $T^{-1}(0,x_0)\sim {\sf N}(0,\Id)$ and $T^{-1}(1,x_1)\sim {\sf N}(0,\Id)$. As a result, since $x_0$, $x_1$, and $z$ are independent, and $z\sim {\sf N}(0,\Id)$,  we have
\begin{equation}
    \label{eq:gauss:sum}
    \alpha(t) T^{-1}(0,x_0) + \beta(t) T^{-1}(1,x_1) +\gamma(t) z \sim {\sf N}(0,\alpha^2(t)\Id+\beta^2(t)\Id+\gamma^2(t)\Id) = {\sf N}(0,\Id)
\end{equation}
where the second equality follows from the condition $\alpha^2(t)+\beta^2(t)+\gamma^2(t)=1$. Therfore, using again the definition of the map~$T$
\begin{equation}
    \label{eq:rho:map}
    T(t,\alpha(t) T^{-1}(0,x_0) + \beta(t) T^{-1}(1,x_1) +\gamma(t) z) \sim \rho(t),
\end{equation}
and we are done.
\end{proof}

\schrob*
\begin{proof}
Denote by $\hat\rho(t,x)$ the density of $\hat x_t$ and define the current $\hat\jmath : [0,1]\times \RR^d \to \RR^d$ as
\begin{equation}
    \label{eq:def:rho:j}
        \hat \jmath(t,x) = \EE \big(\partial_t \hat I_t +\gamma(t) z| x=\hat x_t \big) \hat \rho(t,x).
\end{equation}
The max-min problem~\eqref{eq:max:min} can then be formulated as the constrained optimization problem:
\begin{equation}
    \label{eq:max:min:rho:j}
    \begin{aligned}
        \max_{\hat \rho,\hat \jmath} \min_{\hat u} &\int_0^1 \int_{\RR^d}\left( \tfrac12|\hat u(t,x)|^2 \hat \rho(t,x)-  \hat u(t,x)\cdot \hat \jmath(t,x) \right) dx dt\\
        \text{subject to:} \quad & \partial_t \hat \rho + \nabla \cdot \hat \jmath= \eps  \Delta \hat\rho, \quad \hat\rho(t=0) = \rho_0, \quad \hat \rho(t=1)= \rho_1
    \end{aligned}
\end{equation}
To solve this problem we can use the extended objective
\begin{equation}
    \label{eq:max:min:rho:j:2}
    \begin{aligned}
        \max_{\hat \rho,\hat \jmath} \min_{\hat u} &\Bigg(\int_0^1 \int_{\RR^d}\left( \tfrac12|\hat u(t,x)|^2 \hat \rho(t,x)-  \hat u(t,x)\cdot \hat \jmath(t,x) \right) dx dt\\
        & -\int_0^1 \int_{\RR^d}\lambda(t,x) \left(\partial_t \hat \rho(t,x) + \nabla \cdot \hat \jmath(t,x) - \eps  \Delta \hat\rho(t,x)\right) dx dt \\
        & +\int_{\RR^d} \eta_0(x)\left(\hat\rho(0,x)-\rho_0(x)\right) dx-\int_{\RR^d} \eta_1(x)\left(\hat\rho(1,x)-\rho_1(x)\right) dx \Bigg)
    \end{aligned}
\end{equation}
where $\lambda(t,x)$, $\eta_0(x)$, and $\eta_1(x)$ are Lagrange multipliers used to enforce the constraints. The unique minimizer $(\rho,j,\lambda)$ of this optimization problem solves the Euler-Lagrange equations:
\begin{equation}
    \label{eq:max:min:rho:j:el}
    \begin{aligned}
        & \partial_t \rho + \nabla \cdot j = \eps  \Delta \rho, \quad \rho(t=0) = \rho_0\quad \rho(t=1) = \rho_1\\
        & \partial_t \lambda + \tfrac12 |u|^2 = -\eps \Delta \lambda,\\
        & j = u \rho \\
        & u = \nabla \lambda
    \end{aligned}
\end{equation}
We can use the last two equations to write the first two as \eqref{eq:max:min:rho:j:el:2}, with $u=\nabla \lambda$. Since under Assumption~\ref{as:sb:2} there is an interpolant that realizes the density $\rho(t)$ that solves~\eqref{eq:max:min:rho:j:el:2}, we conclude that an optimizer $(I,u)$ of the the max-min problem~\eqref{eq:max:min} exists. For any optimizer, $I$ will be such that $\rho(t)$ is the density of $x_t= I(t,x_0,x_1)+\gamma(t) z$, and $u$ will satisfy $u=\nabla \lambda$.
\end{proof}

\subsection{Proof of Theorem~\ref{thm:denoise:iter}}
\label{app:denoise}

\dn*

\begin{proof}
    Use 
    \begin{equation}
    \label{eq:betaa:expand}
    \beta(t_{j+1}) \beta^{-1}(t_j) = 1 + \dot\beta(t_j) \beta^{-1}(t_i) (t_{j+1}-t_j) + O\big((t_{j+1}-t_j)^2\big)
\end{equation}
and
\begin{equation}
    \label{eq:alpha:expand}
    \alpha(t_{j+1}) - \alpha(t_j)\beta(t_{j+1}) \beta^{-1}(t_j) = \big( \dot\alpha(t_j) - \alpha(t_j) \dot\beta(t_j) \beta^{-1}(t_i)) (t_{j+1}-t_j) + O\big((t_{j+1}-t_j)^2\big)
\end{equation}
to deduce that \eqref{eq:iterate} implies
\begin{equation}
    \label{eq:iterate:expand}
    \begin{aligned}
    X^\DEN_{{j+1}} & = X^\DEN_{j} + \dot \beta(t_{j}) \beta^{-1}(t_j) X^\DEN_{j} (t_{j+1}-t_j)\\
    & + \big( \dot\alpha(t_j) - \alpha(t_j) \dot\beta(t_j) \beta^{-1}(t_i))   \eta^\OS_z(t_{j},X^\DEN_{j})(t_{j+1}-t_j)  + O\big((t_{j+1}-t_j)^2\big).
    \end{aligned}
\end{equation}
or equivalently
\begin{equation}
    \label{eq:iterate:expand:2}
    \begin{aligned}
    \frac{X^\DEN_{{j+1}} -X^\DEN_{j}}{t_{j+1}-t_j} = \dot \beta(t_{j}) \beta^{-1}(t_j) X^\DEN_{j} + \big( \dot\alpha(t_j) - \alpha(t_j) \dot\beta(t_j) \beta^{-1}(t_i))   \eta^\OS_z(t_{j},X^\DEN_{j})  + O(t_{j+1}-t_j).
    \end{aligned}
\end{equation}
Taking the limit as $N,j\to \infty$ with $j/N\to t\in[0,1]$, we recover~\eqref{eq:iterate:lim} and deduce that $X^\DEN_{j} \to X_t$.
\end{proof}

Note that the proof shows that the result of Theorem~\ref{thm:denoise:iter} also holds if we use a nonuniform grid of times $t_j$, $j\in\{1,\ldots,N\}$.

\subsection{Proof of Theorem~\ref{thm:cond}}
\label{app:rect}

\rec*

\begin{proof}
The first part of the statement can be established by following the same steps as in the proof of Theorem~\ref{prop:interpolate} and Corollary~\ref{prop:interpolate_fpe}. For the proof of the second part, use first \eqref{eq:new:os} written as $x^\REC_t = M(t,z)$ in \eqref{eq:b:explicit} to deduce that
\begin{equation}
    \label{eq:b:explicit:int}
    b^\REC(t,x^\REC_t) = \dot \alpha(t) N(t,M(t,z)) + \dot\beta(t) X_{t=1}(N(t,M(t,z))) = \dot \alpha(t) z + \dot\beta(t) X_{t=1}(z) = \dot x^\REC_t
\end{equation}
This shows that~\eqref{eq:b:explicit} is the unique minimizer of~\eqref{eq:obj:brec}. Next, use \eqref{eq:prod:flow:sol} written as $X^\REC_t(x) = M(t,x)$ in \eqref{eq:b:explicit} to deduce that
\begin{equation}
    \label{eq:b:explicit:2}
    b^\REC(t,X^\REC_t(x)) = \dot \alpha(t) N(t,M(t,x)) + \dot\beta(t) X_{t=1}(N(t,M(t,x))) = \dot \alpha(t) x + \dot\beta(t) X_{t=1}(x)
\end{equation}
This implies that~\eqref{eq:prod:flow:sol} solves \eqref{eq:prob:flow:rec}, and since the solution of this ODE is unique, we are done.
\end{proof}

\begin{table}[!t]
\centering
\begin{tabular}{lccc}
\toprule &  ImageNet 32$\times$32 & Flowers & Flowers (Mirror)   \\
\midrule Dimension & 32$\times$32 & 128$\times$128 &  128$\times$128\\

\# Training point &  1,281,167  & 315,123 & 315,123 \\
\midrule Batch Size & 512 & 64 & 64 \\
Training Steps  & 8$\times$10$^5$  & 3.5 $\times$ 10$^5$  & 8 $\times$ 10$^5$  \\
Hidden dim  & 256 & 128 & 128 \\
Attention Resolution  & 64 & 64 & 64 \\
Learning Rate (LR) & $0.0002$ & $0.0002$ & $0.0002$\\
LR decay (1k epochs) & 0.995 & 0.995 & 0.985  \\
U-Net dim mult & [1,2,2,2] & [1,1,2,3,4]  & [1,1,2,3,4]  \\
Learned $t$ sinusoidal embedding  & Yes & Yes & Yes \\
$t_{0}, t_{f}$ for $t\sim \text{Unif}[t_0, t_f]$, learning $\eta$ & [0.0, 1.0] & n/a &  n/a \\
$t_{0}, t_{f}$ for $t\sim \text{Unif}[t_0, t_f]$, learning $s$ & n/a & [.0002, .9998] & [.0002, .9998] \\
$t_{0}, t_{f}$ when sampling with ODE & [0.0, 1.0] & [0.0001, 0.9999] & [.0001, .9999] \\
$t_{0}, t_{f}$ when sampling with SDE, $\eta$ & [0.0, 0.97] + denoising & n/a & n/a \\
$t_{0}, t_{f}$ when sampling with SDE, $s$ & n/a & [.0001, .9999] & [.0001, .9999] \\
$\gamma(t)$ in $x_t$ & $\sqrt{(t(1-t)}$ & $\sqrt{(t(1-t)}$ & $\sqrt{(10t(1-t)}$ \\
EMA decay rate & 0.9999 & 0.9999 & 0.9999 \\
EMA start iteration & 10000 & 10000 & 10000 \\
$\#$ GPUs & 2 & 4 & 2\\
\bottomrule
\end{tabular}
\caption{Hyperparameters and architecture for image datasets.}
\label{tab:archs:img}
\end{table}

\section{Experimental Specifications}
\label{app:exp}

Details for the experiments in Section~\ref{sec:sde:ode} are provided here. Feed forward neural networks of depth $4$ and width $512$ are used for each model of the velocities $b$, $v$, and $s$. Training was done for $7000$ iterations on batches comprised of $25$ draws from the base, $400$ draws from the target, and $100$ time slices. At each iteration, we used a variance reduction technique based on antithetic sampling, in which two samples $\pm z$ are used for each evaluation of the loss. The objectives given in \eqref{eq:obj:v} and \eqref{eq:obj:s} were optimized using the Adam optimizer. The learning rate was set to $.002$ and was dropped by a factor of $2$ every $1500$ iterations of training. To integrate the ODE/SDE when drawing samples, we used the Heun-based integrator as suggested in \cite{Karras2022edm}.

\subsection{Image Experiments}
\label{app:exp:img}

\paragraph{Network Architectures.} For all image generation experiments, the U-Net architecture originally proposed in \cite{ho2020} is used. The specification of architecture hyperparameters as well as training hyperparameters are given in Table \ref{tab:archs:img}. The same architecture is used regardless of whether learning $b, v, s,$ or $\eta$. 

When using the SDE and learning $\eta_z$, we found that integrating to a time slightly before $t_f=1.0$ and using the denoising formula~\eqref{eq:iterate} beyond this point provided the best results, as described in the main text.

\bibliography{refs}

\end{document}